\documentclass{article}
\usepackage{jmlr2e}
\firstpageno{1}

\usepackage{amsmath,amssymb,amsfonts}
\newtheorem{thm}{Theorem}
\newtheorem{assump}{Assumption}

\usepackage{algorithm}
\usepackage{algorithmic}
\usepackage{xcolor}
\DeclareMathOperator*{\argminA}{arg\,min}    
\usepackage{tabularx}

\let\cite\citep

\begin{document}
	%\maketitle
	
	\title{On the Global Convergence of Natural Actor-Critic with Two-layer Neural Network Parametrization}
	
	\author{\name Mudit Gaur\email mgaur@purdue.eduu \\
		\name Amrit Singh Bedi\email amritbd@umd.edu \\
			\name Di Wang\email di.wang@kaust.edu.sa \\
		\name Vaneet Aggarwal\email{vaneet@purdue.edu}\\
	%	\addr Purdue University, West Lafayette IN 47907, USA
	} 
\editor{}

\maketitle

\begin{abstract} \label{abstract}
Actor-critic algorithms have shown remarkable success in solving state-of-the-art decision-making problems. However, despite their empirical effectiveness, their theoretical underpinnings remain relatively unexplored, especially with neural network parametrization. In this paper, we delve into the study of a natural actor-critic algorithm that utilizes neural networks to represent the critic. Our aim is to establish sample complexity guarantees for this algorithm, achieving a deeper understanding of its performance characteristics. To achieve that, we propose a Natural Actor-Critic algorithm with 2-Layer critic parametrization (NAC2L). Our approach involves estimating the $Q$-function in each iteration through a convex optimization problem. We establish that our proposed approach attains a sample complexity of $\tilde{\mathcal{O}}\left(\frac{1}{\epsilon^{4}(1-\gamma)^{4}}\right)$. In contrast, the existing sample complexity results in the literature only hold for a tabular or linear MDP. Our result, on the other hand, holds for countable state spaces and does not require a linear or low-rank structure on the MDP.

%\am{One important point is missing from the abstract which I am not exactly sure what yo write. Mudit, please try to answer that here, I will include it in the abstract text later. }

%\am{Question: Why studying the convergence analysis is important?}

%\am{Why the achieved bound of $\tilde{\mathcal{O}}\left(\frac{1}{\epsilon^{4}(1-\gamma)^{4}}\right)$ is interesting and why one should care about it?}

%Once these are answered, I will modify the abstract. 
\end{abstract}

\section{Introduction} \label{introduction}
%\textcolor{red}
%{Amrit: My take/suggestion would be this:
%No need to introduce what RL is (so first paragraph in intro is redundant), Starting my mentioning that actor critic algorithms are kind of standard in practice in RL literature.
%Then motivation it like: for policy gradient algorithms, it has been shown in the literature that the global optimality is possible, which could be of crucial nature for various practical applications. We need to motivate a little bit why global optimality is important.
%Then we mention that, even though these results exist for policy gradient, there is no result in the literature for the actor-critic style approach, which is essentially derived from PG only. So we ask this question: "Is it possible to show the global optimality of actor-critic algorithms?"
%Now, you mention the challenges in doing that. Like, why the current analysis are not able to do that, first is, to do the actor-critic analysis at the first place, actor analysis requires to bound the critic error in terms of critic iterate difference, which can only be possible fo convex objective, which demands for the linear approximation assumption for the critic. Then mention why that assumption is not valid in practice. So, then what to do?
%This is where your proposed approach comes to rescue and helps to prove the global optimality without relying on linear approximation. Explain the intuition of the idea here.}

\textbf{Motivation.} The use of neural networks in Actor-critic (AC) algorithms is widespread in various machine learning applications, such as games \citep{DBLP:journals/corr/abs-1708-04782,Haliem2021LearningMG}, robotics \citep{morgan2021model}, autonomous driving \citep{9351818}, ride-sharing \citep{zheng2022stochastic}, networking \citep{geng2020multi}, and recommender systems \citep{li2020video}. AC algorithms sequentially update the estimate of the actor (policy) and the critic (value function) based on the data collected at each iteration, as described in \citep{konda1999actor}. An empirical and theoretical improvement of the AC, known as Natural Actor-Critic (NAC), was proposed in \citep{peters2008natural}. NAC replaced the stochastic gradient step of the actor with the natural gradient descent step described in \citep{kakade2001natural} based on the theory of \citep{rattray1998natural}. Despite its widespread use in state-of-the-art reinforcement learning (RL) implementations, there are no finite-time sample complexity results available in the literature for a setup where neural networks (NNs) are used to represent the critic. Such results provide estimates of the required data to achieve a specified accuracy in a loss function, typically in terms of the difference between the average reward obtained in following the policy obtained through the algorithm being analyzed and the optimal policy.

\textbf{Challenges.} The primary challenge in obtaining sample complexity for a NAC is that the critic loss function with NN parametrization becomes non-convex, making finite-time sample guarantees impossible. Although recent work in \citep{fu2020single} obtains upper bounds on the error in estimating the optimal action-value function in NN critic settings, these upper bounds are asymptotic and depend on the NN parameters. Similarly, other asymptotic sample complexity results, such as those in \citep{kakade2001natural,konda1999actor,bhatnagar2010actor}, require i.i.d. sampling of the state action pairs at each iteration of the algorithm. While sample complexity results for linear MDP structures have been derived in \citep{wu2020finite,xu2020improving} (see Related work), the structure allows for a closed-form critic update. The relaxation of the linear MDP assumption requires a critic parametrization to learn the Q-function, limiting the ability to find globally optimal policies due to non-global optimality of critic parameters.

\textbf{Approach.} As the optimal critic network update is not available for general MDPs, a parametrization of the critic network is used. To address the challenge of optimizing the critic parameters, we assume a two-layer neural network structure for the critic parametrization. Recent studies have shown that the parameter optimization for two-layer ReLU neural networks can be transformed into an equivalent convex program, which is computationally tractable and solvable exactly \cite{pmlr-v119-pilanci20a}. Convex formulations for convolutions and deeper models have also been investigated \cite{sahiner2020vector,sahiner2020convex}. In this work, we use these approaches to estimate the parameterized Q-function in the critic. In the policy gradient analysis, the error caused by the lack of knowledge of the exact gradient term is assumed to be upper bounded by a constant (see related work for details), where the sample complexity of this term is not accounted due to no explicit critic estimation. This paper considers the sample complexity of the critic function to derive the overall sample complexity of the natural actor-critic algorithm. Resolving the optimality of the critic parameter update, this paper considers the following question:

{\centering{\emph{Is it possible to obtain global convergence sample complexity results of Natural Actor-Critic algorithm with two-layer neural network parametrization of critic?}}}

\textbf{Contributions.} 
%
%We answer the above question affirmatively in this work. We propose an actor-critic based algorithm, wherein the critic and the policy actor are parameterized by a neural network. The actor is represented by a neural network using a smooth activation function while the critic is represented by a two-layer network with an ReLU activation function. As is the case in \citep{agarwal2020optimality}, we use the smoothness property of the actor neural network to obtain an upper bound on the error of estimation of the optimal value function in terms of the error incurred in obtaining the compatible function approximation term \citep{Sutton1998}. This error is in turn decomposed into different components each of which are upper bounded separately. The error incurred due to the estimation of the critic function is converted to convex optimization using the results in \citep{wang2021hidden,pmlr-v119-pilanci20a}. This allows us to obtain a finite time sample complexity bound. We improve upon previously obtained sample complexity results for Natural Actor Critic algorithms as we remove the need for a linear function approximation for the critic. Thus we obtain the first finite time sample complexity result of Natural Actor Critic without a linear function approximation of the critic.
%
%
In this work, we provide an affirmation to the above-mentioned question. We summarize our contributions as follows. 

\begin{itemize}%[leftmargin=+.1in]

\item We propose a novel actor-critic algorithm, employing a general parametrization for the actor and two-layer ReLU neural network parametrization for the critic. 

\item Building upon the insights presented in \citep{agarwal2020optimality}, we leverage the inherent smoothness property of the actor parametrization to derive an upper bound on the estimation error of the optimal value function. This upper bound is expressed in terms of the error incurred in attaining the compatible function approximation term, as elucidated in \citep{DBLP:conf/nips/SuttonMSM99}. Our analysis dissects this error into distinct components, each of which is individually bounded from above.

\item To address the estimation error arising from the critic function, we leverage the interesting findings of \citep{wang2021hidden,pmlr-v119-pilanci20a}, converting the problem of finding critic function parameters into a convex optimization problem. This allows us to establish a finite-time sample complexity bound, a significant improvement over previous results obtained for NAC algorithms. Notably, our approach eliminates the reliance on a linear function approximation for the critic, presenting the first finite-time sample complexity result for NAC without such an approximation.
\end{itemize}

\section{Related Works} \label{Related Works}

% \subsection{Natural Actor Critic as Extension of Natural Policy Gradient} \label{Actor Critic}
\textbf{Natural Actor-Critic (NAC).} Actor critic methods, first conceptualised in \cite{sutton1988learning}, aim to combine the benefits of the policy gradient methods and $Q$-learning based methods. The policy gradient step in these methods is replaced by a Natural Policy Gradient proposed in \citep{kakade2001natural} to obtain the so-called Natural Actor Critic in \citep{peters2005natural}. Sample complexity results for Actor Critic were first obtained for MDP with finite states and actions in \citep{williams1990mathematical}, and more recently in \citep{lan2023policy,zhang2020provably}. Finite time convergence for natural actor critic using a linear MDP assumption has been obtained in \citep{chen2022finite,khodadadian2021finite,xu2020non} with the best known sample complexity of $\tilde{\mathcal{O}}\left(\frac{1}{{\epsilon}^{2}(1-\gamma)^{4}}\right)$  \citep{xu2020improving}. Finite time sample complexity results are however, not available for Natural Actor Critic setups for general MDP where neural networks are used to represent the critic.  \citep{fu2020single} obtained asymptotic upper bounds on the estimation error for a natural actor critic algorithm where a neural network is used to represent the critic. The key related works here are summarized in Table \ref{tbl_related}. 

 \textbf{Sample Complexity of PG Algorithms.} Policy Gradient methods (PG) were first introduced in \citep{DBLP:conf/nips/SuttonMSM99}, with Natural Policy Gradients in \citep{kakade2001natural}. Sample complexity were obtained in tabular setups in  \citep{even2009online} and then improved in \citep{agarwal2020optimality}. Using linear dynamics assumption convergence guarantees were obtained in  \citep{fazel2018global}, \citep{dean2020sample}, \citep{bhandari2019global} . \citep{zhang2020global} obtains sample complexity for convergence to second order stationary point policies. In order to improve the performance of Policy Gradient methods, techniques such as REINFORCE \cite{williams1992simple},  function approximation \citep{sutton1999policy}, importance weights \citep{papini2018stochastic}, adaptive batch sizes \citep{papini2017adaptive}, Adaptive Trust Region Policy Optimization \citep{shani2020adaptive} have been used. 
We note that the Natural Policy Gradient approach has been shown to achieve a sample complexity of $\tilde{\mathcal{O}}\left(\frac{1}{\epsilon^{3}(1-\gamma)^{6}}\right)$ \citep{liu2020improved}. However, they assume that the error incurred due to the lack of knowledge of the exact gradient term is assumed to be upper bounded by a constant (Assumption 4.4, \citep{liu2020improved}). Additionally, each estimate of the critic requires on average a sample of size $\left(\frac{1}{1-\gamma}\right)$ and the obtained estimate is only known to be an unbiased estimate of the critic, with no error bounds provided. Explicit dependence of the constant of (Assumption 4.4, \citep{liu2020improved}) in terms of number of samples is not considered in their sample complexity results. For the Natural Actor Critic analysis in this paper, we provide explicit dependence of the error in estimating the gradient update and incorporate the samples required for the estimation in the sample complexity analysis. We describe this difference in detail in Appendix \ref{Comp_npg}.

\begin{table}[t]
\caption{\small This table summarizes the sample complexities of different Natural Actor-Critic Algorithms. We note that we present the first finite time sample complexity results for general MDPs where  both actor and critic are  parametrized by a neural network.}
\label{tbl_related}{%
\begin{tabular}{|c|c|c|c|c|c|}
\hline
References & \multicolumn{1}{c|}{\begin{tabular}[c]{@{}c@{}}Actor \\ parametrization\end{tabular}} & \multicolumn{1}{c|}{\begin{tabular}[c]{@{}c@{}}Critic \\ parametrization\end{tabular}}    & \multicolumn{1}{c|}{\begin{tabular}[c]{@{}c@{}}MDP \\ Assumption\end{tabular}}  & \multicolumn{1}{c|}{\begin{tabular}[c]{@{}c@{}}Sample \\ Complexity\end{tabular}} \\ \hline
\citep{williams1990mathematical}           &        Tabular                                                                               &         Tabular     & Finite                                                                               &                Asymptotic                                                                   \\ \hline
 \citep{borkar1997actor}          &                  Tabular                                                                     &                 Tabular                                                                            &   Finite &                   Asymptotic                                                             \\ \hline
     \citep{xu2020non}     &                     General                                                                &                      None (Closed Form)                                                                     &    Linear &                    $\tilde{\mathcal{O}}(\epsilon^{-4}(1-\gamma)^{-9})$                                                           \\ \hline
    \citep{khodadadian2021finite}     &                      General                                                                &                      None (Closed Form)                                                                     &    Linear &                    $\tilde{\mathcal{O}}(\epsilon^{-3}(1-\gamma)^{-8})$                                                           \\ \hline
 \citep{xu2020improving}          &       General                                                                                &                        None (Closed Form)          & Linear                                                         &              $\tilde{\mathcal{O}}(\epsilon^{-2}(1-\gamma)^{-4})$                                                                        \\ \hline
\citep{fu2020single}  &                    Neural Network                                                                   &    Neural Network & None                                                                  & Asymptotic                                                                       \\ \hline
This work  &                   General                                                                   &     Neural Network (2-layer) & None                                                               & $\tilde{\mathcal{O}}(\epsilon^{-4}(1-\gamma)^{-4})$                                                                       \\ \hline
\end{tabular}
}
\end{table}

\section{Problem Setup} \label{Problem Setup}

We consider a discounted Markov Decision Process (MDP) given by the tuple ${\mathcal{M}}:=(\mathcal{S}, \mathcal{A}, P, R, \gamma)$, where $\mathcal{S}$ is a bounded measurable state space, $\mathcal{A}$ is the finite set of actions. ${P}:\mathcal{S}\times\mathcal{A} \rightarrow \mathcal{P}(\mathcal{S})$  is the probability transition kernel\footnote{
For a measurable set $\mathcal{X}$, let $\mathcal{P}(\mathcal{X})$ denote the set of all probability measures over $\mathcal{X}$.
}, 
${R}: \mathcal{S}\times\mathcal{A} \rightarrow \mathcal{P}([0,R_{\max}])$ is the reward kernel on the state action space with $R_{\max}$ being the absolute value of the maximum reward, and $0<\gamma<1$ is the discount factor. A policy $\pi:\mathcal{S} \rightarrow \mathcal{P}(\mathcal{A})$ maps a state to a probability distribution over the action space. 
 Here, we denote by $\mathcal{P}(\mathcal{S}), \mathcal{P}(\mathcal{A}), \mathcal{P}([a,b])$, the set of all probability distributions over the state space, the action space, and a closed interval $[a,b]$, respectively. With the above notation, we define the action value function for a given policy $\pi$ as 
 \begin{align}
     Q^{\pi}(s,a) = \mathbb{E}\left[\sum_{t=0}^{\infty}\gamma^{t}r(s_{t},a_{t})|s_{0}=s,a_{0}=a\right], \label{ps_1}
 \end{align}
 where  $r(s_{t},a_{t}) \sim R(\cdot|s_{t},a_{t})$, $a_{t} \sim \pi(\cdot|s_{t})$ and $s_{t+1} \sim P(\cdot|s_{t},a_{t})$ for $t=\{0,\cdots,\infty\}$. For a discounted MDP, we  define the optimal
action value functions as $Q^{*}(s,a) =  \sup_{\pi}Q^{\pi}(s,a),  \hspace{0.5cm} \forall (s,a) \in  \mathcal{S}\times\mathcal{A} \label{ps_2}$. A policy that achieves the optimal  action value functions is known as the  \textit{optimal policy} and is denoted as $\pi^{*}$. It holds that $\pi^{*}$ is the greedy policy with respect to $Q^{*}$ \citep{10.5555/1512940}. Hence finding  $Q^{*}$ is sufficient to obtain the optimal policy. In a similar manner, we can define the value function as $V^{\pi}(s) = \mathbb{E}\left[\sum_{t=0}^{\infty}\gamma^{t}r'(s_{t},a_{t})|s_{0}=s\right], \label{ps_8}$ and from the definition in \eqref{ps_1}, it holds that  $V^{\pi}(s) = \mathbb{E}_{a\sim\pi}\left[Q^{\pi}(s,a) \right]$. Hence, we can define the optimal value function as $V^{*}(s) =  \sup_{\pi}V^{\pi}(s),  \hspace{0.5cm} \forall s \in  \mathcal{S}\label{ps_9_1}$. We define $\rho_{\pi}(s)$ as the stationary state distribution induced by the policy $\pi$ starting at state $s$ and $\zeta_{\pi}(s,a)$ is the corresponding stationary state action distribution defined as $\zeta_{\pi}(s,a)=\rho_{\pi}(s).\pi(a|s)$.
We overload notation to define $V^{\pi}(\rho) = \mathbb{E}_{s_{0} \sim \rho}[V^{\pi}(s_{0})]$, where $\rho$ is an initial state distribution.
We can define the visitation distribution as   $ d^{\pi}(s_{0}) = (1-\gamma)\sum_{t=0}^{\infty} {\gamma}^{t}Pr^{\pi}(s_{t}=s|s_{o})$. Here $Pr^{\pi}(s_{t}=s|s_{o})$ denotes the probability the state at time $t$ is $s$ given a starting state of $s_{o}$. Hence, we can write $d^{\pi}_{\rho}(s) = \mathbb{E}_{s_{o} \sim \rho}[d^{\pi}(s_{0})]$.

We now describe our Neural Actor-Critic (NAC) algorithm.  In a natural policy gradient algorithm \citep{NIPS2001_4b86abe4}, the policy is parameterized as $\{\pi_{\lambda} , \lambda \in \Lambda\}$. We have $K$ total iterations of the algorithm. At iteration $k$, the policy parameters are updated using a gradient descent step of the form 
\begin{eqnarray}
    \lambda_{k+1} = \lambda_{k} + {\eta}F^{\dagger}_{\rho}(\lambda){\nabla}_{\lambda}V^{\pi_{\lambda}}(\rho), \label{alg_eq_2}
\end{eqnarray}
From the policy gradient theorem \citep{DBLP:conf/nips/SuttonMSM99} we have 
\begin{eqnarray}
   {\nabla}_{\lambda_{k}}V^{\pi_{\lambda_{k}}}(\rho) &=& \mathbb{E}_{s,a}(\nabla{\log(\pi_{\lambda_{k}})(a|s)}Q^{\pi_{\lambda_{k}}}(s,a)) \label{alg_eq_3_1}, \\
    F^{\dagger}_{\rho}(\lambda_{k}) &=& \mathbb{E}_{s,a} \left[{\nabla}\log{\pi_{\lambda_{k}}(a|s)}\left({\nabla_{t}}\log{\pi_{\lambda_{k}}(a|s)}\right)^{T}\right], \label{alg_eq_3}
\end{eqnarray}
where $s \sim d^{\pi_{\lambda_{k}}}_{\rho}, a \sim \pi_{\lambda_{k}}(.|s)$. From \cite{DBLP:conf/nips/SuttonMSM99}, the principle of compatible function approximation implies that we have  
\begin{eqnarray}
    F^{\dagger}_{\rho}(\lambda_{k}){\nabla}_{\lambda_{k}}V^{\pi_{\lambda_{k}}}(\rho) &=& \frac{1}{1-\gamma}w_{k}^{*} \\
    w_{k}^{*} &=& \argminA_{w} \mathbb{E}_{s,a} (A^{\pi_{\lambda_{k}}}(s,a) - w{{\nabla}_{\lambda_{k}}\log(\pi_{\lambda_{k}}(a|s))})^{2}, \label{alg_eq_5}
\end{eqnarray}
and $s \sim d^{\pi_{\lambda_{k}}}_{\rho}, a \sim \pi_{\lambda_{k}}(.|s)$
Here $(A^{\pi_{\lambda_{k}}}(s,a) = Q^{\pi_{\lambda_{k}}}(s,a) - V^{\pi_{\lambda_{k}}}(s))$. For natural policy gradient algorithms such as in \citep{agarwal2020optimality} and \citep{liu2020improved} an estimate of $Q^{\pi_{\lambda_{k}}}$ (and from that an estimate of $A^{\pi_{\lambda_{k}}}(s,a)$) is obtained through a sampling procedure that requires on average $\left(\frac{1}{1-\gamma}\right)$ for each sample of $Q^{\pi_{\lambda_{k}}}$.  For the natural actor critic setup we maintain a parameterised estimate of the $Q$-function is updated at each step and is used to approximate $Q^{\pi_{\lambda_{k}}}$. 

\citep{xu2020improving} and  \citep{wu2020finite} assume that the $Q$ function can be represented as $Q^{\lambda_{k}}(s,a) = w^{t}(\phi(s,a))$ where $\phi(s,a)$ is a known feature vector and $w^{t}$ is a parameter that is updated in what is known as the critic step. The policy gradient step is known as the actor update step.  In case a neural network is used to represent the $Q$ function, at each iteration $k$ of the algorithm, an estimate of its parameters are obtained by solving an optimization of the form 
\begin{eqnarray}
    \argminA_{\theta \in \Theta} \mathbb{E}_{s,a}({Q^{\pi_{\lambda_{k}}}-Q_{\theta}})^{2}, \label{alg_eq_6}
\end{eqnarray}
Due to the non-convex activation functions of a neural network the optimization in Equation \eqref{alg_eq_6} is non-convex and hence finite sample guarantees for actor critic algorithms with neural network representation of $Q$ functions is not possible. Thus in order to perform a finite time analysis, a linear MDP structure is assumed which is very restrictive and not practical for large state space problems. For a Natural Policy Gradient setup, a Monte Carlo sample is obtained, which makes the process require additional samples as causing high variance in critic estimation.

%\am{Before ending this section, discuss and derive the natural actor critic algorithm updates here. You may even need to discuss about the policy parameterizations. You don't need to spend a lot of time in derive the details. Use some standard reference and discuss properly the NAC updates, and mention them here}

%\am{\textbf{Limitations of existing NAC Approaches}}\am{Then before moving on the next section, we need to summarize the limitations of existing NAC approaches. THIS IS VERY VERY IMPORTANT. This part would again emphasize the importance of our work, and highlight why do we even need to care about our proposed ideas. Divide it into a list of limitations and highlight them. I provided a general list, but feel free to edit and correct and expand upon them.}

%\am{- Limitation 1:  Existing NAC approaches can only prove local convergence ?? }

%\am{Limitation 2: The critic analysis and assumptions are unrealistic. Explain here how the linear function approximation or other assumptions for the critic analysis are not very general, and could be violated easily in practice}

%\am{Limitation 3: Sample complexity bounds are loose in terms of $\frac{1}{(1-\gamma)}$ factor. Explain here how even for the special case of policy gradient setting, the existing bounds are loose in terms of the power of $\frac{1}{(1-\gamma)}$} \am{\textbf{ANOTHER VERY IMPORTANT CONCERN HERE IS THAT WHAT IS THE INTUITIVE EXPLANATION OF OBTAIN SUCH BIG IMPROVEMENT, THAT IS NOT VERY CLEAR. WE NEED AN ANSWER TO THAT, AND THEN EMPHASIZE ON THAT EVEN IN THE INTRODUCTION.}}

\section{Proposed Approach} \label{Proposed Algorithm}
\subsection{Convex Reformulation of 2 layer Neural Network} \label{convex_reform}
 
We represent the $Q$ function (critic) using a 2 layer ReLU neural network. A 2-layer ReLU Neural Network with input $x \in \mathbb{R}^{d}$ is defined as $f(x) = \sum_{i=1}^{m}\sigma'(x^{T}u_{i})\alpha_{i} \label{ReLU_0}$, where $m \geq 1$ is the number of neurons in the neural network, the parameter space is $\Theta_{m}=\mathbb{R}^{d\times m} \times \mathbb{R}^{m}$ and  $\theta =(U,\alpha)$ is an element of the parameter space, where $u_{i}$ is the $i^{th}$ column of $U$, and $\alpha_{i}$ is the $i^{th}$ coefficient of $\alpha$. The function $\sigma': \mathbb{R} \rightarrow \mathbb{R}_{\ge 0}$ is the ReLU or restricted linear function defined as     $\sigma'(x) \triangleq\max(x,0)$. In order to obtain parameter $\theta$ for a given set of data $X \in \mathbb{R}^{n \times d}$  and the corresponding response values $y \in \mathbb{R}^{n \times 1}$, we desire the parameter that minimizes the squared loss, given by 
\begin{align}
\mathcal{L}(\theta) =\argminA_{\theta} \Bigg\|\sum_{i=1}^{m}\sigma(Xu_{i})\alpha_{i}- y\Bigg\|_{2}^{2}. \label{ReLU_1}
\end{align}
In \eqref{ReLU_1}, we have the term $\sigma(Xu_{i})$ which is a vector $\{\sigma'((x_{j})^{T}u_{i})\}_{j \in \{1,\cdots, n\}}$, where $x_{j}$ is the $j^{th}$ row of $X$. It is the ReLU function applied to each element of the vector $Xu_{i}$. We note that the optimization in Equation \eqref{ReLU_1} is non-convex in $\theta$ due to the presence of the ReLU activation function. In \cite{wang2021hidden}, it is shown that this optimization problem has  an equivalent convex form, provided that the number of neurons $m$ goes above a certain threshold value. This convex problem is obtained by replacing the ReLU functions in the optimization problem with equivalent diagonal operators.
The convex problem is given as
\begin{align}
     \mathcal{L}^{'}_{\beta}(p) :=  \argminA_{p} \Bigg\|\sum_{D_{i} \in D_{X}}D_{i}(Xp_{i}) - y\Bigg\|^{2}_{2},\label{ReLU_1_0}
\end{align}
where $p \in \mathbb{R}^{d \times |D_{X}|}$. 
$D_{X}$ is the set of diagonal matrices $D_{i}$ which depend on the data-set $X$. Except for cases of $X$ being low rank, it is not computationally feasible to obtain the set $D_{X}$. We instead use $\tilde{D} \in D_{X}$ to solve the convex problem in \eqref{ReLU_1_0} 
where $p$ now would lie in $p \in \mathbb{R}^{d \times |\tilde{D}|}$. 
The relevant details of the formulation and the definition of the diagonal matrices $D_{i}$  are provided in Appendix \ref{cones_apdx}.
For a set of parameters $\theta =(u,\alpha) \in \Theta$, we denote neural network represented by these parameters as 

\begin{eqnarray}
    Q_{\theta}(s,a)=\sum_{i=1}^{m}\sigma'((s,a)^{T}u_{i})\alpha_{i}. \label{ReLU_1_2}
\end{eqnarray}

\subsection{Proposed Natural Actor Critic Algorithm with 2-Layer Critic Parametrization (NAC2L)}
\begin{algorithm}[t]
	\caption{Natural Actor Critic  with 2-Layer Critic Parametrization (NAC2L)}
	\label{algo_1}
	\textbf{Input:} $\mathcal{S},$ $ \mathcal{A}, \gamma, $ Time Horizon K $ \in \mathcal{Z}$ , Updates per time step J $ \in \mathcal{Z}$  ,starting state action sampling 
    distribution $\nu$, Number of convex optimization steps $T_{k,j}, k \in \{1,\cdots,K\}, j \in \{1,\cdots,J\}$, Actor SGD learning rate $\eta$
    \begin{algorithmic}[1]
    \STATE \textbf{Initialize:} $\tilde{Q}(s,a)=0 \hspace{0.1cm}\forall (s,a) \in \mathcal{S}\times\mathcal{A}$, \\
    ${Q}_{0}(s,a)=0 \hspace{0.1cm}\forall (s,a) \in \mathcal{S}\times\mathcal{A}$ \\
    $\lambda_{0}=\{0\}^{d}$
		\FOR{$k\in\{1,\cdots,K\}$} 
		{  
            \FOR{$j\in\{1,\cdots,J\}$} 
		  {  
            \STATE $X_{k}= \varnothing $
		    \STATE Take $n_{k,j}$ state action pairs sampled from $\nu$ as the starting state action distribution and then following policy $\pi_{\lambda_{k}}$.
                \label{a1_l1}\\  
                \STATE Set $y_i= r_{i} + \gamma\max_{a' \in \mathcal{A}}{Q}_{k,j-1}(s_{i+1},a')$, where $i \in \{1,\cdots,n\}$ \label{a1_l2}\\ 
                \STATE Set $X_{j},Y_{j}$ as  the matrix of the sampled state action pairs and vector of estimated $Q$ values respectively \label{a1_l3}\\
                \STATE $X_{k} = X_{k} \cup X_{j}$ \\
		      \STATE \textbf{Call  Algorithm } \ref{algo_2} with input ($X=X_{j}$, $y=Y_{j}$, $T=T_{j}$) and  return parameter $\theta$ \label{a1_l4}\\
                \STATE $Q_{k,j} = Q_{\theta}$
             }
             \ENDFOR\\
            \FOR{$i \in \{1,\cdots, (J.n_{k,j})\} $}
            {   
               \STATE $A_{k,J}(s_{i},a_{i}) = Q_{k,J}(s_{i},a_{i}) -  \sum_{a \in \mathcal{A}} \pi_{\lambda_{k}}(a|s_{i})Q_{k,J}(s_{i},a) $
               \STATE $w_{i+1} =\Big(w_{i} - 2{\beta_{i}}\Big(w_{i}{\cdot} {\nabla}_{\lambda}\log{\pi_{\lambda_{k}}}(a_{i}|s_{i})- A_{k,J}(s_{i},a_{i})\Big){\nabla}_{\lambda}\log{\pi_{\lambda_{k}}}(a_{i}|s_{i}) \Big)$
            }\ENDFOR
            \STATE  $w_{k} = w_{J.n_{k,j}}$
            \STATE Update $\lambda_{k+1} =  \lambda_{k}  + {\eta}\left(\frac{1}{1-\gamma}\right) w_{k}$
		}
		\ENDFOR\\
	Output: $\pi_{\lambda_{K+1}}$
	\end{algorithmic}
\end{algorithm}
\begin{algorithm}
	\caption{Neural Network Parameter Estimation}     %{\bf unclear - what is trained for your algo?}
	\label{algo_2}
	\begin{algorithmic}[1]
	\STATE{\textbf{Input:}} data $(X,y,T)$ \\
	\STATE{\textbf{Sample:}} $\tilde{D}={diag(1(Xg_{i}>0))} : g_{i} \sim \mathcal{N}(0,I), i \in [|\tilde{D}|]  $ \label{a2_l1}\\
	\STATE \textbf{Initialize} $y^1=0,  u^1=0$ \\
        \textbf{Initialize} $g(u) = \| \sum_{D_{i} \in \tilde{D}}D_{i}Xu_{i}-y \|^{2}_{2}$ 
        \FOR{$i\in\{0,\cdots,T\}$} \label{a2_l2}
		{
            \STATE $u^{i+1}= y_{i} - \eta_{i}\nabla{g(y_{i})} $\label{a2_l3}\\
            \STATE $y^{i+1} = \argminA_{y:|y|_{1} \le \frac{R_{max}}{1-\gamma}} \|u_{i+1}-y\|^{2}_{2}$ \label{a2_l4}\\
		}
		\ENDFOR\\
        \STATE Set $u^{T+1}=u^{*}$
	\STATE{Solve Cone Decomposition:}\\ 
	$\bar{v}, \bar{w} \in {u_{i}^{*}=v_{i}-w_{i}, i \in [d]\} }$ such that $v_{i},w_{i} \in \mathcal{K}_{i}$ and at-least one $v_{i},w_{i}$ is zero. \label{a2_l5}\\
	\STATE Construct $(\theta=\{u_{i},\alpha_{i}\})$  using the transformation
         \begin{eqnarray} 
            \psi(v_{i},w_{i}) &=& \left\{\begin{array}{lr}({v}_{i},1), & \text{if } {w}_{i}=0 \label{alg_2_trans}\\   
            ({w}_{i},-1), & \text{if }  
             {v}_{i} = 0\\
            (0,0), & \text{if } {v}_{i} = {w}_{i} = 0  \end{array} \right.
        \end{eqnarray}
        for all  ${i \in \{ 1,\cdots,m\}}$  
         \label{a2_l6}\\
        \STATE Return $\theta$ \label{a2_l7}\\
   \end{algorithmic}
\end{algorithm}
We summarize the proposed approach in Algorithm \ref{algo_1}. Algorithm \ref{algo_1} has an outer for loop with two inner for loops. At a fixed iteration $k$ of the outer for loop and iteration $j$ of the first inner for loop, we obtain a sequence of state action pairs and the corresponding state and reward by following the estimate of the policy at the start of the iteration. In order to perform the critic update, the state action pairs and the corresponding target $Q$ values are stored in matrix form and passed to Algorithm \ref{algo_2}, as the input and output values respectively to solve the following optimization problem.
\begin{eqnarray}
    \argminA_{\theta \in \Theta}\frac{1}{n_{k,j}} \sum_{i=1}^{n_{k,j}}\Bigg( Q_{\theta}(s_{i},a_{i}) - r(s_{i},a_{i}) - {\gamma}\mathbb{E}_{a^{'} \sim \pi_{\lambda_{k}}}Q_{k,j-1}(s_{i+1},a^{'})\Bigg)^{2}, \label{main_res_7}
\end{eqnarray}
where $Q_{k,j-1}$ is the estimate of the $Q$ function at the $k^{th}$ iteration of the outer for loop and the $(j-1)^{th}$ iteration of the  first inner for loop of Algorithm \ref{algo_1}. $Q_{\theta}$ is a neural network defined as in \eqref{ReLU_1_2} and $n_{k,j}$ is the number of state action pairs sampled at the $k^{th}$ iteration of the outer for loop and the $j^{th}$ iteration of the  first inner for loop of Algorithm \ref{algo_1}. This is done at each iteration of the first inner for loop to perform what is known as a Fitted Q-iteration step to obtain the estimate of the critic.

Algorithm \ref{algo_2} first  samples a set of diagonal matrices denoted by $\tilde{D}$ in line \ref{a2_l1} of Algorithm \ref{algo_2}. The elements of $\tilde{D}$ act as the diagonal matrix replacement of the ReLU function. Algorithm \ref{algo_2} then solves an optimization of the form given in Equation \eqref{main_res_7} by converting it to an optimization of the form \eqref{conv_form_1}. This convex optimization is solved in Algorithm \ref{algo_2} using the projected gradient descent algorithm. After obtaining the optima for this convex program, denoted by $u^{*}=\{u^{*}_{i}\}_{i \in \{1,\cdots,|\tilde{D}|\}}$, in line \ref{a2_l6}, we transform them into an estimate of the solutions for the optimization given in \eqref{main_res_7}, which are then passed back to Algorithm \ref{algo_1}. The procedure is described in detail along with the relevant definitions in Appendix \ref{cones_apdx}. 

The estimate of $w_{k}^{*}$ is obtained in the second inner for loop of Algorithm \eqref{algo_1} where a gradient descent is performed for the loss function of the form given in  Equation \eqref{alg_eq_5} using the state action pairs sampled in the first inner for loop. Note that we do not have access to the true $Q$ function that is required for the critic update. Thus we use the estimate of the $Q$ function obtained at the end of the first inner for loop. After obtaining our estimate of the minimizer of Equation \eqref{alg_eq_5}, we update the policy parameter using the stochastic gradient update step. Here the state action pairs used are the same we sampled in the first inner for loop.

%\am{Remak (Comparison to NAC Algorithm): Add a remark here explaining the difference to NAC algorithm and mention in one line that the only difference is in the critic parametrization. But it helps in terms of obtaining sharper guarantees for the Natural Actor Critif Approaches as explain in next section. }

\section{Global Convergence Analysis of NAC2L Algorithm} \label{Main Result}
\subsection{Assumptions} \label{Assumptions}
In this subsection, we formally describe the assumptions that will be used in the results.
\begin{assump}  \label{assump_1} 
For any $\lambda_{1}, \lambda_{2} \in \Lambda$ and $(s,a) \in (\mathcal{S}\times\mathcal{A})$ we have 
\begin{eqnarray}
    \|{\nabla}log(\pi_{\lambda_{1}})(a|s) - {\nabla}log(\pi_{\lambda_{2}})(a|s)\|_{2} \le \beta\|\lambda_{1} -\lambda_{2}\|_{2}  
\end{eqnarray}
where $\beta > 0$.
\end{assump}
Such assumptions have been utilised in prior policy Gradient based works such as \citep{agarwal2020optimality,liu2020improved}. This assumption is satisfied for the softmax policy parameterization 
\begin{eqnarray}
    \pi_{\lambda}(a|s) = \frac{\exp(f_{\lambda}(s,a))}{\sum_{a^{'}\in \mathcal{A}}f_{\lambda}(s,a^{'})}
\end{eqnarray}
where $f_{\lambda}(s,a)$ is a neural network with a smooth activation function \citep{agarwal2020optimality}.

\begin{assump}  \label{assump_2} 
Let $\theta^{*} \triangleq \arg\min_{\theta \in \Theta} \mathcal{L}(\theta)$, where $\mathcal{L}(\theta)$ is defined in  \eqref{ReLU_1} and we denote $Q_{\theta^{*}}(\cdot)$ as $Q_{\theta}(\cdot)$ as defined in \eqref{ReLU_1_2} for $\theta=\theta^{*}$. 
Also, let $\theta_{\tilde{D}}^{*} \triangleq \arg\min_{\theta \in \Theta} \mathcal{L}_{|\tilde{D}|}(\theta)$, where $\mathcal{L}_{\tilde{D}}(\theta)$ is the loss function $\mathcal{L}(\theta)$ with the set of diagonal matrices $D$ replaced by $\tilde{D} \in D$. Further, we denote $Q_{\theta_{|\tilde{D}|}^{*}}(\cdot)$ as $Q_{\theta}(\cdot)$ as defined in \eqref{ReLU_1_2} for $\theta=\theta_{|\tilde{D}|}^{*}$. Then we assume
\begin{equation}
    \mathbb{E}_{s,a}(|Q_{\theta^{*}} - Q_{\theta_{|\tilde{D}|}^{*}}|)_{\nu} \leq \epsilon_{|\tilde{D}|},
\end{equation}
for any $\nu \in \mathcal{P}(\mathcal{S}\times\mathcal{A})$.
\end{assump}
 Thus, $\epsilon_{|\tilde{D}|}$ is a measure of the error incurred due to taking a sample of diagonal matrices $\tilde{D}$ and not the full set $D_{X}$. In practice, setting $|\tilde{D}|$ to be the same order of magnitude as $d$  (dimension of the data) gives us a sufficient number of diagonal matrices to get a reformulation of the non convex optimization problem which performs comparably or better than existing gradient descent algorithms, therefore $\epsilon_{|\tilde{D}|}$ is only included for theoretical completeness and will be negligible in practice. This has been practically demonstrated in \cite{pmlr-v162-mishkin22a,pmlr-v162-bartan22a,pmlr-v162-sahiner22a}. Refer to Appendix \ref{cones_apdx} for details of $D_{X}$, $\tilde{D}$ and $\mathcal{L}_{|\tilde{D}|}(\theta)$.

\begin{assump}  \label{assump_3} 
We assume that for all functions $Q:\mathcal{S}\times\mathcal{A} \rightarrow \left[0,\left(\frac{R_{\max}}{1-\gamma}\right)\right]$,  there exists a function $ Q_{\theta} $ where $ \theta \in \Theta $ such that 
\begin{eqnarray}
\mathbb{E}_{s,a}{(Q_{\theta}-Q)}^{2}_{\nu} \le \epsilon_{approx}, 
\end{eqnarray}
for any $\nu \in \mathcal{P}(\mathcal{S}\times\mathcal{A})$. 
\end{assump}
$\epsilon_{approx}$ reflects the error that is incurred due to the inherent lack of expressiveness of the neural network function class. In the analysis of \cite{pmlr-v120-yang20a}, this error is assumed to be zero. Explicit upper bounds of $\epsilon_{bias}$ is given in terms of neural network  parameters in works such as \citep{yarotsky2017error}.

\begin{assump}  \label{assump_3_1} 
We assume that for all functions $Q:\mathcal{S}\times\mathcal{A} \rightarrow \left[0,\left(\frac{R_{\max}}{1-\gamma}\right)\right]$ and $\lambda \in \Lambda$,  there exists $w^{*} \in \mathbb{R}^{|\mathcal{A}|}$  such that 
\begin{eqnarray}
\mathbb{E}_{s,a}(|Q - w^{*}{\nabla}log(\pi_{{\lambda}}(a|s))|)_{\nu} \le \epsilon_{bias}, 
\end{eqnarray}
for any $\nu \in \mathcal{P}(\mathcal{S}\times\mathcal{A})$. 
\end{assump}
Assumption \ref{assump_3_1} is similar to the \emph{Transfer error} assumption in works such as \citep{agarwal2020flambe,liu2020improved}. The key difference is that in the referenced works the assumption is based on the true $Q$ function, while the estimate of $w^{*}$ is obtained by using a  noisy Monte Carlo estimate. For our case we use a known parameterised $Q$ function to obtain our estimate of $w^{*}$.

\begin{assump}  \label{assump_4} 
For any $\theta \in \Theta$, denote by $\pi_{\theta}$ as the policy corresponding to the parameter $\theta$ and $\mu_{\theta}$ as the corresponding stationary state action distribution of the induced Markov chain. We assume that there exists a positive integer $m$ such that 
\begin{align}
    d_{TV}\left(\mathbb{P}((s_{\tau},a_{\tau}) \in \cdot |(s_{0},a_{0})=(s,a)),\mu_{\theta}(\cdot)\right) \leq m{\rho}^{\tau}, \forall s \in \mathcal{S}
\end{align}
\end{assump}
This assumption implies that the Markov chain is geometrically mixing. Such assumption is widely used both in analysis of stochastic gradient descent literature such as \cite{9769873,sun2018markov}, as well as finite time analysis of RL algorithms  such as  \cite{wu2020finite,NEURIPS2020_2e1b24a6}.
\begin{assump}  \label{assump_6} 
Let $\nu_{1}$ be a probability measure on $\mathcal{S}\times\mathcal{A}$ which is absolutely continuous with respect to the Lebesgue measure. Let $\{\pi_{t}\}$ be a sequence of policies and suppose that the state action pair has an initial distribution of $\nu_{1}$. Then we assume that for all $\nu_{1}, \nu_{2} \in \mathcal{P}(\mathcal{S}\times\mathcal{A})$ there exists a constant $\phi_{\nu_{1},\nu_{2}} \le \infty$ such that
\begin{eqnarray} \sup_{\pi_{1},\pi_{2},\cdots,\pi_{m}}\Bigg|\Bigg| \frac{d(P^{\pi_{1}}P^{\pi_{2}}\cdots{P}^{\pi_{m}}\nu_{2})}{d\nu_{1}}\Bigg|\Bigg|_{\infty} &\le& \phi_{\nu_{1},\nu_{2} }
    \label{assump_6_1}
\end{eqnarray}
for all $m \in \{0,\cdots,\infty\}$, where  $ \frac{d(P^{\pi_{1}}P^{\pi_{2}}\cdots{P}^{\pi_{m}}\nu_{2})}{d\nu_{1}}$ denotes the Radon Nikodym derivative of the state action distribution $P^{\pi_{1}}P^{\pi_{2}}\cdots{P}^{\pi_{m}}\nu_{2}$ with respect to the distribution $\nu_{1}$.
\end{assump}

 This assumption puts an upper bound on the difference between the state action distribution $\nu_{1}$ and the state action distribution induced by sampling a state action pair from the distribution $\mu_{2}$ followed by any possible policy for the next $m$ steps for any finite value of $m$. Similar assumptions have been made in \cite{pmlr-v120-yang20a,JMLR:v17:10-364,farahmand2010error}. 

\subsection{Main Result} \label{Theorem Statement}
\begin{thm} \label{thm}
Suppose Assumptions \ref{assump_1}-\ref{assump_6} hold. Let Algorithm \ref{algo_1} run for $K$ iterations $J$ be the number of iterations of the first inner loop of Algorithm \ref{algo_1}. Let $n_{k,j}$ denote the number of  state-action pairs sampled and $T_{k,j}$ the number of iterations of Algorithm \ref{algo_2} at iteration $k$ of the outer for loop and iteration $j$ of the first inner for loop of Algorithm \ref{algo_1}. Let $\alpha_{i}$ be the projected gradient descent size at iteration $i$ of Algorithm \ref{algo_2} and $|\tilde{D}|$ the number of diagonal matrices sampled in Algorithm \ref{algo_2}. Let $\beta_{i}$ be the step size in the projected gradient descent at iteration $i$ of the second inner for loop of Algorithm \ref{algo_1}. Let $\nu \in \mathcal{P}(\mathcal{S}\times\mathcal{A})$ be the starting state action distribution at the each iteration $k$ of Algorithm \ref{algo_1}. If we have, 
$\alpha_{i} = \frac{||u^{*}_{k,j}||_{2}}{L_{k,j}\sqrt{{i}+1}}$, $\eta = \frac{1}{\sqrt{K}}$  and $\beta_{i} = \frac{\mu_{k}}{i+1}$, then we obtain
\begin{align}
\min_{k \le K}(V^{*}(\nu)-V^{\pi_{\lambda_{K}}}(\nu)) \leq&    {\mathcal{O}}\left(\frac{1}{\sqrt{K}(1-\gamma)}\right) \!+\!  \frac{1}{K(1-\gamma)}\sum_{k=1}^{K}\sum_{j=0}^{J-1} \mathcal{O}\left(\frac{\log\log(n_{k,j})}{\sqrt{n_{k,j}}}\right) \!+ \! \ \nonumber
\\
& + \frac{1}{K(1-\gamma)}\sum_{k=1}^{K}\sum_{j=0}^{J-1}{\mathcal{O}}\left(\frac{1}{\sqrt{T_{k,j}}}\right) + \frac{1}{K(1-\gamma)}\sum_{k=1}^{K}{\mathcal{O}}(\gamma^{J}) \nonumber\\
& + \frac{1}{1-\gamma}\left(\epsilon_{bias} + (\sqrt{\epsilon_{approx}}) + {\epsilon_{|\tilde{D}|}}\right)
\end{align}

where $||u^{*}_{k,j}||_{2}, L_{k,j},\mu_{k},\epsilon_{bias}, \epsilon_{approx}, {\epsilon_{|\tilde{D}|}}$ are constants.
\end{thm}
Hence, for $K  =  \mathcal{O}(\epsilon^{-2}(1-\gamma)^{-2}) $, $J = \mathcal{O}\left(\log\left(\frac{1}{\epsilon}\right)\right) $, $n_{k,j} = \tilde{\mathcal{O}}\left(\epsilon^{-2}(1-\gamma)^{-2}\right)$, 

$T_{k,j} = \mathcal{O}(\epsilon^{-2}(1-\gamma)^{-2})$ we have%
\begin{eqnarray}
\min_{k \le K}(V^{*}(\nu)-V^{\pi_{\lambda_{K}}}(\nu)) \leq \epsilon +  \frac{1}{1-\gamma}\left(\epsilon_{bias} + (\sqrt{\epsilon_{approx}}) + {\epsilon_{|\tilde{D}|}}\right),    
\end{eqnarray}
which implies a sample complexity of $ \sum_{k=1}^{K} \sum_{j=1}^{J} (n_{k,j}) = \tilde{\mathcal{O}}\left({\epsilon^{-4}(1-\gamma)^{-4}}\right)$.
\section{Proof Sketch of Theorem \ref{thm}} 
%
%In the following, we give an overview of the proof of the main result in Theorem \ref{thm}. {\bf Is this overview or complete proof? If overview, where is the complete proof?}
%
The detailed proof of Theorem \ref{thm} is given in Appendix \ref{thm proof}. The difference between our estimated value function denoted by $V^{\pi_{\lambda_{k}}}$ and the optimal value function denoted by $V^{*}$ (where $\pi_{\lambda_{k}}$ is the policy obtained at the step $k$ of algorithm \ref{algo_1}) is first expressed as a function of the compatible function approximation error, which is then split into different components which are analysed separately. The proof is thus split into two stages. In the first stage, we demonstrate how the difference in value functions is upper bounded as a function of the errors incurred till the final step $K$. The second part is to upper bound the different error components.

{\bf Upper Bounding Error in Separate Error Components: } We use the smoothness property assumed in Assumption \ref{assump_1} to obtain a bound on the expectation of the difference between our estimated value function and the optimal value function.
\begin{align}
  \min_{k \in \{1,\cdots,K\}} V^{*}(\nu) - V^{\pi_{\lambda_{K}}}(\nu) \leq \frac{\log(|\mathcal{A}|)}{K{\eta}(1-\gamma)}+ \frac{\eta{\beta}W^{2}}{2(1-\gamma)} + \frac{1}{K}\sum_{k=1}^{K}\frac{err_{k}}{1-\gamma}, \label{main_res_1}
 \end{align}
where 
\begin{eqnarray}
err_{k} = \mathbb{E}_{s,a}(A^{\pi_{\lambda_{k}}} - w^{k}{\nabla}log(\pi_{\lambda_{k}}(a|s))),\label{error_k}  
\end{eqnarray}
 and $s \sim d^{\pi^{*}}_{\nu}, a \sim \pi^{*}(.|s)$ and $W$ is a constant such that $||w^{k}||_{2} \le W$  $\forall k$, where $k$ denotes the iteration of the outer for loop of Algorithm \ref{algo_1}.
We split the term in \eqref{error_k} into the errors incurred due to the actor and critic step as follows
\begin{eqnarray}
err_{k} &=& \mathbb{E}_{s,a}(A^{\pi_{\lambda_{k}}} - w^{k}{\nabla}log(\pi_{\lambda_{k}}(a|s))) \\
&=& \underbrace{\mathbb{E}_{s,a}(A^{\pi_{\lambda_{k}}} - A_{{k,J}})}_{I} +  \underbrace{\mathbb{E}_{s,a}(A_{{k,J}} - w^{k}{\nabla}log(\pi_{\lambda_{k}}(a|s)))}_{II}  \label{main_res_2}.
\end{eqnarray}
Note that $I$ is the difference between the true $A^{\pi_{\lambda_{k}}}$ function corresponding to the policy $\pi_{\lambda_{k}}$ and $A_{k,J}$ is our estimate. This estimation is carried out in the first inner for loop of Algorithm \ref{algo_1}. Thus $I$ is the error incurred in the critic step. $II$ is the error incurred in the estimation of the actor update. This is incurred in the stochastic gradient descent steps in the second inner for loop of Algorithm \ref{algo_1}.

Also note that the expectation is with respect to the discounted state action distribution of the state action pairs induced by the optimal policy $\pi^{*}$. The state action samples that we obtain are obtained from the policy $\pi_{\lambda_{k}}$. Thus using assumption \ref{assump_6} we convert the expectation in Equation \eqref{main_res_2} to an expectation with respect to the stationary state action distribution induced by the policy $\pi_{\lambda_{k}}$.

{\bf Upper Bounding Error in Critic Step: }  For each iteration $k$ of the Algorithm \ref{algo_1}. We show that minimzing $I$ is equivalent to solving the following problem
\begin{eqnarray}
    \argminA_{\theta \in \Theta} \mathbb{E}_{s,a}({Q^{\pi_{\lambda_{k}}}-Q_{{\theta}}})^{2} \label{main_res_4},
\end{eqnarray}
We will rely upon the analysis laid out in \cite{farahmand2010error} and instead of the iteration of the value functions, we will apply a similar analysis to the action value function to obtain an upper bound for the error incurred in solving the problem in Equation \eqref{main_res_4} using the Fitted Q-Iteration.  We recreate the result for the value function from Lemmas 2 of \cite{munos2003error}  for the action value function $Q$ to obtain  
%\begin{eqnarray}
%    Q^{\pi_{\lambda_{K}}}-Q_{k,J} &\leq& \sum_{j=1}^{J-1} \gamma^{J-k-1} (P^{\pi^{*}})^{J-j-1}\epsilon_{k,j} , \label{main_res_8} + \gamma^{J} (P^{\pi^{*}})^{J}(Q^{\pi_{\lambda_{K}}}-Q_{0}) 
% \end{eqnarray}
%From this we obtain
%
\begin{eqnarray}
    \mathbb{E}_{s,a}(Q^{\pi_{\lambda_{k}}}-Q_{k,J}) &\leq&  \sum_{j=1}^{J-1} \gamma^{J-j-1}(P^{\pi_{\lambda_{k}}})^{J-j-1}{\mathbb{E}}|\epsilon_{k,j}| + \gamma^{J}\left(\frac{R_{max}}{1-\gamma}\right) , \label{main_res_9}
%     &\leq& \sum_{j=0}^{J-1} \gamma^{J-k-1}(P^{\pi^{\lambda_{k}}})^{J-j-1}{\mathbb{E}}\epsilon_{k,j} + \gamma^{J} \mathbb{E}_{a^{'}}(P^{\pi^{*}})^{J}(Q^{\pi_{\lambda_{K}}}-Q_{0}) , \\
\end{eqnarray}
where $\epsilon_{k,j} =  T^{\pi_{\lambda_{k}}}Q_{k,j-1} - Q_{k,j}$ is the Bellman error incurred at iteration $j$ where $T^{\pi_{\lambda_{k}}}Q_{k,j-1},P^{\pi_{\lambda_{k}}}$ are defined as in Equations \eqref{ps_3}, \eqref{ps_7} respectively. $Q_{k,J}$ denotes the $Q$ estimate at the final iteration $J$ of the first inner for loop and iteration $k$ of the outer for loop of Algorithm \ref{algo_1}

The first term on the right hand side is called as the algorithmic error, which depends on how good our approximation of the Bellman error is. The second term on the right hand side is called as the statistical error, which is  the error incurred due to the random nature of the system and depends only on the parameters of the MDP as well as the number of iterations of the FQI algorithm. Intuitively, this error depends on how much data is collected at each iteration, how efficient our solution to the optimization step is to the true solution, and how well our function class can approximate $ T^{\pi_{\lambda_{k}}}Q_{k,j-1}$. Building upon this intuition, we split $\epsilon_{k,j}$ into four different components as follows.
\begin{eqnarray}
    \epsilon_{k,j} &=& T^{\pi_{\lambda_{k}}}Q_{k,j-1} - Q_{k,j} \nonumber\\
                 &=& \underbrace{T^{\pi_{\lambda_{k}}}Q_{k,j-1}-Q^{1}_{k,j}}_{\epsilon^{1}_{k,j}} + \underbrace{Q^{1}_{k,j} -Q^{2}_{k,j}}_{\epsilon^{2}_{k,j}} + \underbrace{Q^{2}_{k,j} -Q^{3}_{k,j}}_{\epsilon^{3}_{k,j}} +\underbrace{Q^{3}_{k,j} - Q_{k,j}}_{\epsilon^{4}_{k,j}} \nonumber\\
                 &=& {\epsilon^{1}_{k,j}} + {\epsilon^{2}_{k,j}} + {\epsilon^{3}_{k,j}} + {\epsilon^{4}_{k,j}} ,\label{last}
\end{eqnarray}
We use the Lemmas \ref{lem_1}, \ref{lem_2}, \ref{lem_3}, and \ref{lem_4}  to bound the error terms in Equation  \eqref{last}. 

{\bf Upper Bounding Error in Actor Step:} Note that we require the minimization of the term $\mathbb{E}_{s,a}(A_{k,J} - w^{k}{\nabla}log(\pi_{\lambda_{k}}(a|s)))$. Here the expectation is with respect to stationary state action distribution corresponding to $\pi_{\lambda_{k}}$. But we do not have samples of states action pairs from the stationary distribution with respect to the policy $\pi_{\lambda_{k}}$, we only have samples from the Markov chain induced by the policy $\pi_{\lambda_{k}}$. We thus refer to the theory in \cite{9769873} and Assumption \ref{assump_4} to upper  bound the error incurred. 
\section{Conclusions} \label{Conclusion and Future Work}

In this paper, we study a Natural Actor Critic algorithm  with a neural network used to represent both the actor and the critic and find the sample complexity guarantees for the algorithm. Using the conversion of the optimization of a 2 layer ReLU Neural Network to a convex problem for estimating the critic, we show that our approach  achieves a sample complexity of $\tilde{\mathcal{O}}(\epsilon^{-4}(1-\gamma)^{-4})$. This demonstrates the first approach for achieving sample complexity beyond linear MDP assumptions for the critic.

{\bf Limitations: } Relaxing the different stated assumptions is an interesting direction for the future. Further, the results assume 2-layer neural network parametrization for the critic. One can likely use the framework described in \citep{belilovsky2019greedy} to extend the results to a multi layer setup.% The theoretical advancements in this paper can have a cascading effect on various problems of reinforcement learning, leading to further breakthroughs and innovations.
%This study raises multiple future problems. First is whether we can extend the results the for critic function represented by a multi layer neural network. Secondly, is it possible to impr, efficient analysis for the error incurred when a sample of cones are chosen rather than the complete set of cones and how to efficiently choose this subset will help with a complete analysis by which Assumption \ref{assump_2} can be relaxed. 
\bibliography{mybib}

\begin{thebibliography}{59}
\providecommand{\natexlab}[1]{#1}
\providecommand{\url}[1]{\texttt{#1}}
\expandafter\ifx\csname urlstyle\endcsname\relax
  \providecommand{\doi}[1]{doi: #1}\else
  \providecommand{\doi}{doi: \begingroup \urlstyle{rm}\Url}\fi

\bibitem[Agarwal et~al.(2020{\natexlab{a}})Agarwal, Kakade, Krishnamurthy, and
  Sun]{agarwal2020flambe}
Alekh Agarwal, Sham Kakade, Akshay Krishnamurthy, and Wen Sun.
\newblock Flambe: Structural complexity and representation learning of low rank
  mdps, 2020{\natexlab{a}}.

\bibitem[Agarwal et~al.(2020{\natexlab{b}})Agarwal, Kakade, Lee, and
  Mahajan]{agarwal2020optimality}
Alekh Agarwal, Sham~M Kakade, Jason~D Lee, and Gaurav Mahajan.
\newblock Optimality and approximation with policy gradient methods in markov
  decision processes.
\newblock In \emph{Conference on Learning Theory}, pages 64--66. PMLR,
  2020{\natexlab{b}}.

\bibitem[Bartan and Pilanci(2022)]{pmlr-v162-bartan22a}
Burak Bartan and Mert Pilanci.
\newblock Neural {F}isher discriminant analysis: Optimal neural network
  embeddings in polynomial time.
\newblock In \emph{Proceedings of the 39th International Conference on Machine
  Learning}, volume 162 of \emph{Proceedings of Machine Learning Research},
  pages 1647--1663. PMLR, 17--23 Jul 2022.
\newblock URL \url{https://proceedings.mlr.press/v162/bartan22a.html}.

\bibitem[Belilovsky et~al.(2019)Belilovsky, Eickenberg, and
  Oyallon]{belilovsky2019greedy}
Eugene Belilovsky, Michael Eickenberg, and Edouard Oyallon.
\newblock Greedy layerwise learning can scale to imagenet.
\newblock In \emph{International conference on machine learning}, pages
  583--593. PMLR, 2019.

\bibitem[Bertail and Portier(2019)]{article}
Patrice Bertail and François Portier.
\newblock Rademacher complexity for markov chains: Applications to kernel
  smoothing and metropolis–hastings.
\newblock \emph{Bernoulli}, 25:\penalty0 3912--3938, 11 2019.
\newblock \doi{10.3150/19-BEJ1115}.

\bibitem[Bertsekas and Shreve(2007)]{10.5555/1512940}
Dimitri~P. Bertsekas and Steven~E. Shreve.
\newblock \emph{Stochastic Optimal Control: The Discrete-Time Case}.
\newblock Athena Scientific, 2007.
\newblock ISBN 1886529035.

\bibitem[Bhandari and Russo(2019)]{bhandari2019global}
Jalaj Bhandari and Daniel Russo.
\newblock Global optimality guarantees for policy gradient methods.
\newblock \emph{arXiv preprint arXiv:1906.01786}, 2019.

\bibitem[Bhatnagar(2010)]{bhatnagar2010actor}
Shalabh Bhatnagar.
\newblock An actor--critic algorithm with function approximation for discounted
  cost constrained markov decision processes.
\newblock \emph{Systems \& Control Letters}, 59\penalty0 (12):\penalty0
  760--766, 2010.

\bibitem[Bonjour et~al.(2022)Bonjour, Haliem, Alsalem, Thomas, Li, Aggarwal,
  Kejriwal, and Bhargava]{Haliem2021LearningMG}
Trevor Bonjour, Marina Haliem, Aala Alsalem, Shilpa Thomas, Hongyu Li, Vaneet
  Aggarwal, Mayank Kejriwal, and Bharat Bhargava.
\newblock Decision making in monopoly using a hybrid deep reinforcement
  learning approach.
\newblock \emph{IEEE Transactions on Emerging Topics in Computational
  Intelligence}, 2022.

\bibitem[Borkar and Konda(1997)]{borkar1997actor}
Vivek~S Borkar and Vijaymohan~R Konda.
\newblock The actor-critic algorithm as multi-time-scale stochastic
  approximation.
\newblock \emph{Sadhana}, 22:\penalty0 525--543, 1997.

\bibitem[Chen and Zhao(2022)]{chen2022finite}
Xuyang Chen and Lin Zhao.
\newblock Finite-time analysis of single-timescale actor-critic.
\newblock \emph{arXiv preprint arXiv:2210.09921}, 2022.

\bibitem[Dean et~al.(2020)Dean, Mania, Matni, Recht, and Tu]{dean2020sample}
Sarah Dean, Horia Mania, Nikolai Matni, Benjamin Recht, and Stephen Tu.
\newblock On the sample complexity of the linear quadratic regulator.
\newblock \emph{Foundations of Computational Mathematics}, 20\penalty0
  (4):\penalty0 633--679, 2020.

\bibitem[Doan(2022)]{9769873}
Thinh~T. Doan.
\newblock Finite-time analysis of markov gradient descent.
\newblock \emph{IEEE Transactions on Automatic Control}, pages 1--1, 2022.
\newblock \doi{10.1109/TAC.2022.3172593}.

\bibitem[Even-Dar et~al.(2009)Even-Dar, Kakade, and Mansour]{even2009online}
Eyal Even-Dar, Sham~M Kakade, and Yishay Mansour.
\newblock Online markov decision processes.
\newblock \emph{Mathematics of Operations Research}, 34\penalty0 (3):\penalty0
  726--736, 2009.

\bibitem[Fan et~al.(2020)Fan, Wang, Xie, and Yang]{pmlr-v120-yang20a}
Jianqing Fan, Zhaoran Wang, Yuchen Xie, and Zhuoran Yang.
\newblock A theoretical analysis of deep q-learning.
\newblock In \emph{Proceedings of the 2nd Conference on Learning for Dynamics
  and Control}, volume 120 of \emph{Proceedings of Machine Learning Research},
  pages 486--489. PMLR, 10--11 Jun 2020.
\newblock URL \url{https://proceedings.mlr.press/v120/yang20a.html}.

\bibitem[Farahmand et~al.(2010)Farahmand, Szepesv{\'a}ri, and
  Munos]{farahmand2010error}
Amir-massoud Farahmand, Csaba Szepesv{\'a}ri, and R{\'e}mi Munos.
\newblock Error propagation for approximate policy and value iteration.
\newblock \emph{Advances in Neural Information Processing Systems}, 23, 2010.

\bibitem[Fazel et~al.(2018)Fazel, Ge, Kakade, and Mesbahi]{fazel2018global}
Maryam Fazel, Rong Ge, Sham Kakade, and Mehran Mesbahi.
\newblock Global convergence of policy gradient methods for the linear
  quadratic regulator.
\newblock In \emph{International conference on machine learning}, pages
  1467--1476. PMLR, 2018.

\bibitem[Fu et~al.(2020)Fu, Yang, and Wang]{fu2020single}
Zuyue Fu, Zhuoran Yang, and Zhaoran Wang.
\newblock Single-timescale actor-critic provably finds globally optimal policy.
\newblock \emph{arXiv preprint arXiv:2008.00483}, 2020.

\bibitem[Geng et~al.(2020)Geng, Lan, Aggarwal, Yang, and Xu]{geng2020multi}
Nan Geng, Tian Lan, Vaneet Aggarwal, Yuan Yang, and Mingwei Xu.
\newblock A multi-agent reinforcement learning perspective on distributed
  traffic engineering.
\newblock In \emph{2020 IEEE 28th International Conference on Network Protocols
  (ICNP)}, pages 1--11. IEEE, 2020.

\bibitem[Kakade and Langford(2002)]{kakade2002approximately}
Sham Kakade and John Langford.
\newblock Approximately optimal approximate reinforcement learning.
\newblock In \emph{Proceedings of the Nineteenth International Conference on
  Machine Learning}, pages 267--274, 2002.

\bibitem[Kakade(2001{\natexlab{a}})]{NIPS2001_4b86abe4}
Sham~M Kakade.
\newblock A natural policy gradient.
\newblock In \emph{Advances in Neural Information Processing Systems},
  volume~14. MIT Press, 2001{\natexlab{a}}.

\bibitem[Kakade(2001{\natexlab{b}})]{kakade2001natural}
Sham~M Kakade.
\newblock A natural policy gradient.
\newblock \emph{Advances in neural information processing systems}, 14,
  2001{\natexlab{b}}.

\bibitem[Khodadadian et~al.(2021)Khodadadian, Chen, and
  Maguluri]{khodadadian2021finite}
Sajad Khodadadian, Zaiwei Chen, and Siva~Theja Maguluri.
\newblock Finite-sample analysis of off-policy natural actor-critic algorithm.
\newblock In \emph{International Conference on Machine Learning}, pages
  5420--5431. PMLR, 2021.

\bibitem[Kiran et~al.(2022)Kiran, Sobh, Talpaert, Mannion, Sallab, Yogamani,
  and Pérez]{9351818}
B~Ravi Kiran, Ibrahim Sobh, Victor Talpaert, Patrick Mannion, Ahmad A.~Al
  Sallab, Senthil Yogamani, and Patrick Pérez.
\newblock Deep reinforcement learning for autonomous driving: A survey.
\newblock \emph{IEEE Transactions on Intelligent Transportation Systems},
  23\penalty0 (6):\penalty0 4909--4926, 2022.
\newblock \doi{10.1109/TITS.2021.3054625}.

\bibitem[Konda and Tsitsiklis(1999)]{konda1999actor}
Vijay Konda and John Tsitsiklis.
\newblock Actor-critic algorithms.
\newblock \emph{Advances in neural information processing systems}, 12, 1999.

\bibitem[Lan(2023)]{lan2023policy}
Guanghui Lan.
\newblock Policy mirror descent for reinforcement learning: Linear convergence,
  new sampling complexity, and generalized problem classes.
\newblock \emph{Mathematical programming}, 198\penalty0 (1):\penalty0
  1059--1106, 2023.

\bibitem[Lazaric et~al.(2016)Lazaric, Ghavamzadeh, and Munos]{JMLR:v17:10-364}
Alessandro Lazaric, Mohammad Ghavamzadeh, and R{\'e}mi Munos.
\newblock Analysis of classification-based policy iteration algorithms.
\newblock \emph{Journal of Machine Learning Research}, 17\penalty0
  (19):\penalty0 1--30, 2016.
\newblock URL \url{http://jmlr.org/papers/v17/10-364.html}.

\bibitem[Li et~al.(2020)Li, Li, Wang, and Li]{li2020video}
Dingcheng Li, Xu~Li, Jun Wang, and Ping Li.
\newblock Video recommendation with multi-gate mixture of experts soft actor
  critic.
\newblock In \emph{Proceedings of the 43rd International ACM SIGIR Conference
  on Research and Development in Information Retrieval}, pages 1553--1556,
  2020.

\bibitem[Liu et~al.(2020)Liu, Zhang, Basar, and Yin]{liu2020improved}
Yanli Liu, Kaiqing Zhang, Tamer Basar, and Wotao Yin.
\newblock An improved analysis of (variance-reduced) policy gradient and
  natural policy gradient methods.
\newblock \emph{Advances in Neural Information Processing Systems},
  33:\penalty0 7624--7636, 2020.

\bibitem[Mishkin et~al.(2022)Mishkin, Sahiner, and
  Pilanci]{pmlr-v162-mishkin22a}
Aaron Mishkin, Arda Sahiner, and Mert Pilanci.
\newblock Fast convex optimization for two-layer {R}e{LU} networks: Equivalent
  model classes and cone decompositions.
\newblock In \emph{Proceedings of the 39th International Conference on Machine
  Learning}, volume 162 of \emph{Proceedings of Machine Learning Research},
  pages 15770--15816. PMLR, 17--23 Jul 2022.
\newblock URL \url{https://proceedings.mlr.press/v162/mishkin22a.html}.

\bibitem[Morgan et~al.(2021)Morgan, Nandha, Chalvatzaki, D’Eramo, Dollar, and
  Peters]{morgan2021model}
Andrew~S Morgan, Daljeet Nandha, Georgia Chalvatzaki, Carlo D’Eramo, Aaron~M
  Dollar, and Jan Peters.
\newblock Model predictive actor-critic: Accelerating robot skill acquisition
  with deep reinforcement learning.
\newblock In \emph{2021 IEEE International Conference on Robotics and
  Automation (ICRA)}, pages 6672--6678. IEEE, 2021.

\bibitem[Munos(2003)]{munos2003error}
R{\'e}mi Munos.
\newblock Error bounds for approximate policy iteration.
\newblock In \emph{ICML}, volume~3, pages 560--567, 2003.

\bibitem[Papini et~al.(2017)Papini, Pirotta, and Restelli]{papini2017adaptive}
Matteo Papini, Matteo Pirotta, and Marcello Restelli.
\newblock Adaptive batch size for safe policy gradients.
\newblock \emph{Advances in neural information processing systems}, 30, 2017.

\bibitem[Papini et~al.(2018)Papini, Binaghi, Canonaco, Pirotta, and
  Restelli]{papini2018stochastic}
Matteo Papini, Damiano Binaghi, Giuseppe Canonaco, Matteo Pirotta, and Marcello
  Restelli.
\newblock Stochastic variance-reduced policy gradient.
\newblock In \emph{International conference on machine learning}, pages
  4026--4035. PMLR, 2018.

\bibitem[Peters and Schaal(2008)]{peters2008natural}
Jan Peters and Stefan Schaal.
\newblock Natural actor-critic.
\newblock \emph{Neurocomputing}, 71\penalty0 (7-9):\penalty0 1180--1190, 2008.

\bibitem[Peters et~al.(2005)Peters, Vijayakumar, and Schaal]{peters2005natural}
Jan Peters, Sethu Vijayakumar, and Stefan Schaal.
\newblock Natural actor-critic.
\newblock In \emph{Machine Learning: ECML 2005: 16th European Conference on
  Machine Learning, Porto, Portugal, October 3-7, 2005. Proceedings 16}, pages
  280--291. Springer, 2005.

\bibitem[Pilanci and Ergen(2020)]{pmlr-v119-pilanci20a}
Mert Pilanci and Tolga Ergen.
\newblock Neural networks are convex regularizers: Exact polynomial-time convex
  optimization formulations for two-layer networks.
\newblock In \emph{Proceedings of the 37th International Conference on Machine
  Learning}, volume 119 of \emph{Proceedings of Machine Learning Research},
  pages 7695--7705. PMLR, 13--18 Jul 2020.
\newblock URL \url{https://proceedings.mlr.press/v119/pilanci20a.html}.

\bibitem[Rattray et~al.(1998)Rattray, Saad, and Amari]{rattray1998natural}
Magnus Rattray, David Saad, and Shun-ichi Amari.
\newblock Natural gradient descent for on-line learning.
\newblock \emph{Physical review letters}, 81\penalty0 (24):\penalty0 5461,
  1998.

\bibitem[Sahiner et~al.(2020{\natexlab{a}})Sahiner, Ergen, Pauly, and
  Pilanci]{sahiner2020vector}
Arda Sahiner, Tolga Ergen, John Pauly, and Mert Pilanci.
\newblock Vector-output relu neural network problems are copositive programs:
  Convex analysis of two layer networks and polynomial-time algorithms.
\newblock \emph{arXiv preprint arXiv:2012.13329}, 2020{\natexlab{a}}.

\bibitem[Sahiner et~al.(2020{\natexlab{b}})Sahiner, Mardani, Ozturkler,
  Pilanci, and Pauly]{sahiner2020convex}
Arda Sahiner, Morteza Mardani, Batu Ozturkler, Mert Pilanci, and John Pauly.
\newblock Convex regularization behind neural reconstruction.
\newblock \emph{arXiv preprint arXiv:2012.05169}, 2020{\natexlab{b}}.

\bibitem[Sahiner et~al.(2022)Sahiner, Ergen, Ozturkler, Pauly, Mardani, and
  Pilanci]{pmlr-v162-sahiner22a}
Arda Sahiner, Tolga Ergen, Batu Ozturkler, John Pauly, Morteza Mardani, and
  Mert Pilanci.
\newblock Unraveling attention via convex duality: Analysis and interpretations
  of vision transformers.
\newblock In \emph{Proceedings of the 39th International Conference on Machine
  Learning}, volume 162 of \emph{Proceedings of Machine Learning Research},
  pages 19050--19088. PMLR, 17--23 Jul 2022.
\newblock URL \url{https://proceedings.mlr.press/v162/sahiner22a.html}.

\bibitem[Scaman and Virmaux(2018)]{10.5555/3327144.3327299}
Kevin Scaman and Aladin Virmaux.
\newblock Lipschitz regularity of deep neural networks: Analysis and efficient
  estimation.
\newblock In \emph{Proceedings of the 32nd International Conference on Neural
  Information Processing Systems}, NIPS'18, page 3839–3848, Red Hook, NY,
  USA, 2018. Curran Associates Inc.

\bibitem[Shani et~al.(2020)Shani, Efroni, and Mannor]{shani2020adaptive}
Lior Shani, Yonathan Efroni, and Shie Mannor.
\newblock Adaptive trust region policy optimization: Global convergence and
  faster rates for regularized mdps.
\newblock In \emph{Proceedings of the AAAI Conference on Artificial
  Intelligence}, volume~34, pages 5668--5675, 2020.

\bibitem[Sun et~al.(2018)Sun, Sun, and Yin]{sun2018markov}
Tao Sun, Yuejiao Sun, and Wotao Yin.
\newblock On markov chain gradient descent.
\newblock \emph{Advances in neural information processing systems}, 31, 2018.

\bibitem[Sutton(1988)]{sutton1988learning}
Richard~S Sutton.
\newblock Learning to predict by the methods of temporal differences.
\newblock \emph{Machine learning}, 3:\penalty0 9--44, 1988.

\bibitem[Sutton et~al.(1999{\natexlab{a}})Sutton, McAllester, Singh, and
  Mansour]{sutton1999policy}
Richard~S Sutton, David McAllester, Satinder Singh, and Yishay Mansour.
\newblock Policy gradient methods for reinforcement learning with function
  approximation.
\newblock \emph{Advances in neural information processing systems}, 12,
  1999{\natexlab{a}}.

\bibitem[Sutton et~al.(1999{\natexlab{b}})Sutton, McAllester, Singh, and
  Mansour]{DBLP:conf/nips/SuttonMSM99}
Richard~S. Sutton, David~A. McAllester, Satinder Singh, and Yishay Mansour.
\newblock Policy gradient methods for reinforcement learning with function
  approximation.
\newblock In \emph{Advances in Neural Information Processing Systems 12,
  {[NIPS} Conference, Denver, Colorado, USA, November 29 - December 4, 1999]},
  pages 1057--1063. The {MIT} Press, 1999{\natexlab{b}}.
\newblock URL
  \url{http://papers.nips.cc/paper/1713-policy-gradient-methods-for-reinforcement-learning-with-function-approximation}.

\bibitem[Vinyals et~al.(2017)Vinyals, Ewalds, Bartunov, Georgiev, Vezhnevets,
  Yeo, Makhzani, K{\"{u}}ttler, Agapiou, Schrittwieser, Quan, Gaffney,
  Petersen, Simonyan, Schaul, van Hasselt, Silver, Lillicrap, Calderone, Keet,
  Brunasso, Lawrence, Ekermo, Repp, and
  Tsing]{DBLP:journals/corr/abs-1708-04782}
Oriol Vinyals, Timo Ewalds, Sergey Bartunov, Petko Georgiev, Alexander~Sasha
  Vezhnevets, Michelle Yeo, Alireza Makhzani, Heinrich K{\"{u}}ttler, John~P.
  Agapiou, Julian Schrittwieser, John Quan, Stephen Gaffney, Stig Petersen,
  Karen Simonyan, Tom Schaul, Hado van Hasselt, David Silver, Timothy~P.
  Lillicrap, Kevin Calderone, Paul Keet, Anthony Brunasso, David Lawrence,
  Anders Ekermo, Jacob Repp, and Rodney Tsing.
\newblock Starcraft {II:} {A} new challenge for reinforcement learning.
\newblock \emph{CoRR}, abs/1708.04782, 2017.
\newblock URL \url{http://arxiv.org/abs/1708.04782}.

\bibitem[Wang et~al.(2021)Wang, Lacotte, and Pilanci]{wang2021hidden}
Yifei Wang, Jonathan Lacotte, and Mert Pilanci.
\newblock The hidden convex optimization landscape of regularized two-layer
  relu networks: an exact characterization of optimal solutions.
\newblock In \emph{International Conference on Learning Representations}, 2021.

\bibitem[Williams(1992)]{williams1992simple}
Ronald~J Williams.
\newblock Simple statistical gradient-following algorithms for connectionist
  reinforcement learning.
\newblock \emph{Reinforcement learning}, pages 5--32, 1992.

\bibitem[Williams and Baird(1990)]{williams1990mathematical}
Ronald~J Williams and LC~Baird.
\newblock A mathematical analysis of actor-critic architectures for learning
  optimal controls through incremental dynamic programming.
\newblock In \emph{Proceedings of the Sixth Yale Workshop on Adaptive and
  Learning Systems}, pages 96--101. Citeseer, 1990.

\bibitem[Wu et~al.(2020)Wu, Zhang, Xu, and Gu]{wu2020finite}
Yue~Frank Wu, Weitong Zhang, Pan Xu, and Quanquan Gu.
\newblock A finite-time analysis of two time-scale actor-critic methods.
\newblock \emph{Advances in Neural Information Processing Systems},
  33:\penalty0 17617--17628, 2020.

\bibitem[Xu et~al.(2020{\natexlab{a}})Xu, Wang, and
  Liang]{NEURIPS2020_2e1b24a6}
Tengyu Xu, Zhe Wang, and Yingbin Liang.
\newblock Improving sample complexity bounds for (natural) actor-critic
  algorithms.
\newblock In \emph{Advances in Neural Information Processing Systems},
  volume~33, pages 4358--4369. Curran Associates, Inc., 2020{\natexlab{a}}.

\bibitem[Xu et~al.(2020{\natexlab{b}})Xu, Wang, and Liang]{xu2020improving}
Tengyu Xu, Zhe Wang, and Yingbin Liang.
\newblock Improving sample complexity bounds for (natural) actor-critic
  algorithms.
\newblock \emph{Advances in Neural Information Processing Systems},
  33:\penalty0 4358--4369, 2020{\natexlab{b}}.

\bibitem[Xu et~al.(2020{\natexlab{c}})Xu, Wang, and Liang]{xu2020non}
Tengyu Xu, Zhe Wang, and Yingbin Liang.
\newblock Non-asymptotic convergence analysis of two time-scale (natural)
  actor-critic algorithms.
\newblock \emph{arXiv preprint arXiv:2005.03557}, 2020{\natexlab{c}}.

\bibitem[Yarotsky(2017)]{yarotsky2017error}
Dmitry Yarotsky.
\newblock Error bounds for approximations with deep relu networks.
\newblock \emph{Neural Networks}, 94:\penalty0 103--114, 2017.

\bibitem[Zhang et~al.(2020{\natexlab{a}})Zhang, Koppel, Zhu, and
  Basar]{zhang2020global}
Kaiqing Zhang, Alec Koppel, Hao Zhu, and Tamer Basar.
\newblock Global convergence of policy gradient methods to (almost) locally
  optimal policies.
\newblock \emph{SIAM Journal on Control and Optimization}, 58\penalty0
  (6):\penalty0 3586--3612, 2020{\natexlab{a}}.

\bibitem[Zhang et~al.(2020{\natexlab{b}})Zhang, Liu, Yao, and
  Whiteson]{zhang2020provably}
Shangtong Zhang, Bo~Liu, Hengshuai Yao, and Shimon Whiteson.
\newblock Provably convergent two-timescale off-policy actor-critic with
  function approximation.
\newblock In \emph{International Conference on Machine Learning}, pages
  11204--11213. PMLR, 2020{\natexlab{b}}.

\bibitem[Zheng et~al.(2022)Zheng, Kurt, and Wang]{zheng2022stochastic}
Jiaohao Zheng, Mehmet~Necip Kurt, and Xiaodong Wang.
\newblock Stochastic integrated actor--critic for deep reinforcement learning.
\newblock \emph{IEEE Transactions on Neural Networks and Learning Systems},
  2022.

\end{thebibliography}
\onecolumn

\appendix
\appendix
\section*{\centering {Appendix}}

\section{Comparison of Sample Complexity Analysis with Natural Policy Gradient} \label{Comp_npg}

For natural policy gradient (NPG) \citep{agarwal2020optimality}, to derive the sample complexity result, the average error in estimation till iteration $K$ is given by 
\begin{eqnarray}
  \argminA_{k \in \{1,\cdots,K\}} V^{*}(\nu) - V^{\pi_{K}}(\nu) &\leq& \Bigg( \frac{\log(|\mathcal{A}|)}{K{\eta}(1-\gamma)}+ \frac{\eta{\beta_{k}}W^{2}}{2(1-\gamma)} + \frac{1}{K}\sum_{i=k}^{K}\frac{err_{k}}{1-\gamma} \Bigg) \label{comp_npg_1},
 \end{eqnarray}
 where $err_{k}$ in the last term on the right-hand side of \eqref{comp_npg_1} is  
\begin{eqnarray}
 err_{k} =\mathbb{E}_{s,a}(A^{\pi_{\lambda_{K}}} - w^{k}{\nabla}log(\pi_{\lambda_{k}}(a|s))) \label{comp_npg_2}
\end{eqnarray}

where $s \sim d^{\pi^{*}}_{\nu}, a \sim \pi^{*}(.|s)$, $w_{k}$ is our estimate of the NPG gradient update term and $\lambda_{k}$ is the policy parameter.  

The term $err_{k}$ is then decomposed in the following manner 

\begin{eqnarray}
 \mathbb{E}_{s,a}(A^{\pi_{\lambda_{K}}} - w^{k}{\nabla}log(\pi_{\lambda_{k}}(a|s))) = \mathbb{E}_{s,a}(A^{\pi_{\lambda_{K}}} - w^{*}{\nabla}log(\pi_{\lambda_{k}}(a|s))) \nonumber\\
 + (w^{*} - w^{k)}{\nabla}log(\pi_{\lambda_{k}}(a|s))) \label{comp_npg_3}
\end{eqnarray}

where $w^{*} = \argminA_{w}\mathbb{E}_{s,a}(Q^{\pi_{\lambda_{K}}} - w{\nabla}log(\pi_{\lambda_{k}}(a|s)))$ where $s \sim d^{\pi^{\lambda_{k}}}_{\nu}, a \sim \pi^{\lambda_{k}}(.|s)$.

For ease of notation we define 
\begin{eqnarray}
\mathbb{E}_{s \sim d^{\pi}_{\nu}, a \sim \pi(.|s)}(Q^{\pi_{\lambda}} - w{\nabla}log(\pi_{\lambda}(a|s)))^{2} = L(w,\lambda,d^{\pi}_{\nu}).     
\end{eqnarray}

Equation \eqref{comp_npg_3} is then be upper bounded as 

\begin{eqnarray}
 \mathbb{E}_{s,a}(Q^{\pi_{\lambda_{K}}} - w^{k}{\nabla}log(\pi_{\lambda_{k}}(a|s)))  &\le&   \sqrt{L(w^{*},\lambda_{k},d^{\pi^{*}}_{\nu})} \nonumber\\
 && + \phi_{k}\sqrt{L(w^{k},\lambda_{k},d^{\pi_{\lambda_{k}}}_{\nu}) - L(w^{*},\lambda_{k},d^{\pi_{\lambda_{k}}}_{\nu})}    \label{comp_npg_3_1}
\end{eqnarray}

where $\phi_{k}$ is a constant which represents the change in expectation from $d^{\pi^{*}}_{\nu}$ to $d^{\pi_{\lambda_{k}}}_{\nu}$. It is defined using an assumption similar to Assumption \ref{assump_6}.

Assumption 6.1 and 6.2 in \citep{agarwal2020optimality} are as follows
\begin{eqnarray}
 L(w^{k},\lambda_{k},d^{\pi_{\lambda_{k}}}_{\nu}) - L(w^{*},\lambda_{k},d^{\pi_{\lambda_{k}}}_{\nu}) &\le& \epsilon_{stat}   \label{comp_npg_4} \\
  L(w^{*};\lambda_{k},d^{\pi^{*}_{\nu}}) &\le& \epsilon_{bias}  \label{comp_npg_5}
\end{eqnarray}
$ \forall k \in \{1,\cdots,K\}$, where $K$ is the total number of iterations of the NPG algorithm.

The assumption in Equation \eqref{comp_npg_4} is known as the excess risk assumption and places an upper bound on the error incurred due to the difference between the obtained estimate $w^{k}$ and the optimal solution $w^{*}$ which minimizes $L(w,\lambda_{k},d^{\pi_{\lambda_{k}}}_{\nu})$. It is a measure of uncertainty in estimating the natural gradient update.
 
The assumption in Equation \eqref{comp_npg_5} is known as the transfer error assumption and places an upper bound on the loss function $L(w,\lambda_{k},d^{\pi^{*}}_{\nu})$ evaluated at the minima of the loss function  $L(w;\lambda^{k},d^{\pi_{\lambda_{k}}}_{\nu})$. This is a measure of how similar the policy $\pi^{\lambda_{k}}$ is to the optimal policy $\pi^{*}$. 

In the  analysis of \cite{agarwal2020optimality}, using results of stochastic gradient descent on a convex loss function, $\epsilon_{stat}$ is assumed to be upper bounded as $\tilde{\mathcal{O}}\left(\frac{1}{\sqrt{n_{k}}}\right)$ where $n_{k}$ is the number of state action samples at iteration $k$. Further, $\epsilon_{bias}$ is directly assumed as a constant while it depends on the accurate estimation of $A^{\pi_{\lambda_{k}}}$. 

\textbf{Comparison.} We note that the analysis in Equation \eqref{comp_npg_4}-\eqref{comp_npg_5} does not consider  (i) the extra $\left(\frac{1}{1-\gamma}\right)$ state action samples required to obtain Monte Carlo estimate $A^{\pi_{\lambda_{k}}}$. This is because each Monte Carlo estimate of $A^{\pi_{\lambda_{k}}}$  requires on average $\left(\frac{1}{1-\gamma}\right)$ state action samples;  (ii) the error incurred due to gap between the Monte Carlo estimate $A^{\pi_{\lambda_{k}}}$ and the actual Q-function. In \cite{agarwal2020optimality}, Monte Carlo estimate is only shown to be an unbiased estimate of $A^{\pi_{\lambda_{k}}}$ and no error bound for the estimate is given. This error bound will require additional samples to be very close such that the obtained value function for the policy is $\epsilon$-close. This is the key gap due to which our algorithm gets additional $1/\epsilon$ in the sample complexity.

Our analysis  considers  the number of samples required to estimate $A^{\pi_{\lambda_{k}}}$ to a given accuracy in our sample complexity analysis. In order to account for the difference between the optimal policy $\pi^{*}$and the policy estimate $\pi_{\lambda_{k}}$, we use \ref{assump_6} which has been used and verified in prior works such as \cite{farahmand2010error}, unlike the assumption in Equation \eqref{comp_npg_5}.
%\am{I agree, I didn't pay attention earlier, the explanation in this section is not in the best comprehensible manner. Please revise. }

The authors in \citep{liu2020improved} also perform a similar analysis but only has an Assumption similar to Equation \eqref{comp_npg_5}. This assumption also suffers from the same drawback described above.

\section{Convex Reformulation with Two-Layer Neural Networks}\label{cones_apdx}

%\am{I believe this is one of the main contributions and this should somehow come in the main body of the paper, one way or the other.}

%\am{This actually tells the reason why one could expect improvement with the proposed critic parametrization}

For representing the action value function, we will use a 2 layer ReLU neural network. In this section, we first lay out the theory behind the convex formulation of the 2 layer ReLU neural network.  In the next section it will shown how it is utilised for the FQI algorithm.

In order to obtain parameter $\theta$ for a given set of data $X \in \mathbb{R}^{n \times d}$  and the corresponding response values $y \in \mathbb{R}^{n \times 1}$, we desire the parameter that minimizes the squared loss  (with a regularization parameter $\beta \in [0,1]$), given by 
\begin{eqnarray}
\mathcal{L}(\theta) &=& \argminA_{\theta} \Bigg\|\sum_{i=1}^{m}\sigma(Xu_{i})\alpha_{i}- y\Bigg\|_{2}^{2}.
\end{eqnarray}
Here, we have the term $\sigma(Xu_{i})$ which is a vector $\{\sigma'((x_{j})^{T}u_{i})\}_{j \in \{1,\cdots, n\}}$ where $x_{j}$ is the $j^{th}$ row of $X$. It is the ReLU function applied to each element of the vector $Xu_{i}$. We note that the optimization in Equation \eqref{ReLU_1} is non-convex in $\theta$ due to the presence of the ReLU activation function. In \cite{wang2021hidden}, it is shown that this optimization problem has  an equivalent convex form, provided that the number of neurons $m$ goes above a certain threshold value. This convex problem is obtained by replacing the ReLU functions in the optimization problem with equivalent diagonal operators.
The convex problem is given as
\begin{eqnarray}
     \mathcal{L}^{'}_{\beta}(p) &:=& \argminA_{p} \Bigg\|\sum_{D_{i} \in D_{X}}D_{i}(Xp_{i}) - y\Bigg\|^{2}_{2} 
\end{eqnarray}
where $p \in \mathbb{R}^{d \times |D_{X}|}$. 
$D_{X}$ is the set of diagonal matrices $D_{i}$ which depend on the data-set $X$. Except for cases of $X$ being low rank it is not computationally feasible to obtain the set $D_{X}$. We instead use $\tilde{D} \in D_{X}$ to solve the convex problem 
\begin{eqnarray}
     \mathcal{L}^{'}_{\beta}(p) &:=& ( \argminA_{p} \Bigg\|\sum_{D_{i} \in \tilde{D}}D_{i}(Xp_{i}) - y\Bigg\|^{2}_{2}, \label{conv_form_1}
\end{eqnarray}
where $p \in \mathbb{R}^{d \times |\tilde{D}|}$. 
In order to understand the convex reformulation of the squared loss optimization problem, consider the vector $\sigma(Xu_{i})$

\begin{equation}
   \sigma(Xu_{i})= \begin{bmatrix}
        \{\sigma^{'}((x_{1})^{T}u_{i})\} \\
        \{\sigma^{'}((x_{2})^{T}u_{i})\} \\
                     \vdots\\
        \{\sigma^{'}((x_{n})^{T}u_{i})\}. 
    \end{bmatrix}
\end{equation}
Now for a fixed  $X \in \mathbb{R}^{n \times d}$, different $u_{i} \in \mathbb{R}^{d \times 1}$ will have different components of $\sigma(Xu_{i})$ that are non zero. For example, if we take the set of all $u_{i}$ such that only the first element of $\sigma(Xu_{i})$ are non zero (i.e, only $(x_{1})^{T}u_{i} \ge 0$ and $(x_{j})^{T}u_{i} < 0$ 
 $\forall j \in [2,\cdots,n]$ ) and denote it by the set $\mathcal{K}_{1}$, then we have 
\begin{equation}
     \sigma(Xu_{i}) =  D_{1}(Xu_{i})   \  \  \  \   \forall u_{i}\in \mathcal{K}_{1},
\end{equation}
where $D_{1}$ is the  $n \times n$ diagonal matrix with only the first diagonal element equal to $1$ and the rest $0$. Similarly, there exist a set of $u's$ which result in $\sigma(Xu)$ having certain components to be non-zero and the rest zero. For each such combination of zero and non-zero components, we will have a corresponding set of $u_{i}'s$ and a corresponding $n \times n$ Diagonal matrix $D_{i}$.  We define the possible set of such diagonal matrices possible for a given matrix X as 
\begin{eqnarray}
    D_{X} = \{D=diag(\mathbf{1}(Xu \geq 0)):u \in \mathbb{R}^{d}\, ,D \in \mathbb{R}^{n \times n}\},
\end{eqnarray}
where $diag(\mathbf{1}(Xu \geq 0))$ represents a matrix given by
\begin{equation}
    D_{k,j} = \left\{\begin{array}{lr}
    \mathbf{1}(x_{j}^{T}u), & \text{for } k=j\\
    0 & \text{for } k \neq j\end{array}, \right.
\end{equation}
where $\mathbf{1}(x) = 1$ if $x>0$ and $\mathbf{1}(x) = 0$ if $x \le 0.$ Corresponding to each such matrix $D_{i}$, there exists a set of $u_{i}$ given by
\begin{eqnarray}
    \mathcal{K}_{i} =\{ u \in \mathbb{R}^{d}:\sigma(Xu_{i})=D_{i}Xu_{i}, D_{i} \in D_{X} \} \label{cone}
\end{eqnarray}
where $I$ is the $n \times n$ identity matrix. The number of these matrices ${D}_{i}$ is upper bounded by $2^{n}$. From \cite{wang2021hidden} the upper bound is $\mathcal{O}\left(r\left(\frac{n}{r} \right)^{r}\right)$ where $r=rank(X)$. Also, note that the sets $\mathcal{K}_{i}$ form a partition of the space $\mathbb{R}^{d \times 1}$. Using these definitions,  we define the equivalent convex problem to the one in Equation \eqref{ReLU_1} as
\begin{eqnarray}
     \mathcal{L}_{\beta}(v,w) := \argminA_{v,w} \Bigg(\Bigg\|\sum_{D_{i} \in D_{X}}D_{i}(X(v_{i} - w_{i})) - y\Bigg\|^{2}_{2}  \label{ReLU_2}\Bigg) ,
\end{eqnarray}
where $v=\{v_{i}\}_{i \in 1,\cdots,|D_{X}|}$, $w=\{w_{i}\}_{i \in 1,\cdots,|D_{X}|}$, $v_{i},w_{i} \in \mathcal{K}_{i}$, note that by definition, for any fixed $i \in \{1,\cdots,|D_{X}|\}$ at-least one of $v_{i}$ or $w_{i}$ are zero. If $v^{*},w^{*}$ are the optimal solutions to Equation \eqref{ReLU_2}, the number of neurons $m$ of the original problem in Equation \eqref{ReLU_1} should be greater than the number of elements of $v^{*},w^{*}$, which have at-least one of $v_{i}^{*}$ or $w_{i}^{*}$ non-zero. We denote this value as $m^{*}_{X,y}$, with the subscript $X$ denoting that this quantity depends upon the data matrix $X$ and response $y$. 

We convert $v^{*},w^{*}$ to optimal values of Equation \eqref{ReLU_1}, denoted by $\theta^{*}=(U^{*},\alpha^{*})$, using a function $\psi:\mathbb{R}^{d}\times\mathbb{R}^{d} \rightarrow \mathbb{R}^{d}\times\mathbb{R}$ defined as follows
\begin{eqnarray} 
    \psi(v_{i},w_{i}) &=& \left\{\begin{array}{lr}({v}_{i},1), & \text{if } {w}_{i}=0 \label{ReLU_2_1}\\   
    ({w}_{i},-1), & \text{if }  
     {v}_{i} = 0\\
    (0,0), & \text{if } {v}_{i} = {w}_{i} = 0  \end{array} \right.
    \label{ReLU_4}
\end{eqnarray}
where according to \cite{pmlr-v119-pilanci20a} we have $(u_{i}^{*},\alpha_{i}^{*})=\psi(v_{i}^{*},w_{i}^{*})$, for all $i \in \{1,\cdots,|{D}_{X}|\}$ where $u^{*}_{i},\alpha^{*}_{i}$ are the elements of $\theta^{*}$. Note that restriction of $\alpha_{i}$ to $\{1,-1,0\}$ is shown to be valid in  \cite{pmlr-v162-mishkin22a}. For  $i \in \{|{D}_{X}|+1,\cdots,m\}$ we set $(u_{i}^{*},\alpha_{i}^{*})=(0,0)$.

Since $D_{X}$ is hard to obtain computationally unless $X$ is of low rank, we can construct a subset $\tilde{D} \in D_{X}$ and perform the optimization in Equation \eqref{ReLU_2} by replacing $D_{X}$ with $\tilde{D}$ to get 
\begin{eqnarray}
 \mathcal{L}_{\beta}(v,w) := \argminA_{v,w} \Bigg(\Bigg\|\sum_{D_{i} \in \tilde{D}}D_{i}(X(v_{i} - w_{i})) - y\Bigg\|^{2}_{2} \label{ReLU_2_2}\Bigg) 
\end{eqnarray}
where $v=\{v_{i}\}_{i \in 1,\cdots,|\tilde{D}|}$, $w=\{w_{i}\}_{i \in 1,\cdots,|\tilde{D}|}$, $v_{i},w_{i} \in \mathcal{K}_{i}$, by definition, for any fixed $i \in \{1,\cdots,|\tilde{D}|\}$ at-least one of $v_{i}$ or $w_{i}$ are zero.

 The required condition for $\tilde{D}$ to be a sufficient replacement for $D_{X}$ is as follows. Suppose $(v,w)=(\bar{v}_{i},\bar{w}_{i})_{i \in (1,\cdots,|\tilde{D}|)}$ denote the optimal solutions of Equation \eqref{ReLU_2_2}. Then we require  
\begin{eqnarray}
    m \ge \sum_{D_{i} \in \tilde{D}} |\{ \bar{v}_{i}: \bar{v}_{i} \neq 0 \} \cup  \{ \bar{w}_{i}: \bar{w}_{i} \neq 0 \}|. \label{ReLU_2_3}
\end{eqnarray}
Or,  the number of neurons in the neural network are greater than the number of indices $i$ for which at-least one of $v_{i}^{*}$ or $w_{i}^{*}$ is non-zero. Further, 
\begin{eqnarray}
  diag(Xu_{i}^{*} \geq 0: i \in [m]) \in \tilde{D}.   \label{ReLU_2_3_1}
\end{eqnarray}
In other words,  the diagonal matrices induced by the optimal $u_{i}^{*}$'s of Equation \eqref{ReLU_1} must be included in our sample of diagonal matrices. 
This is proved in Theorem 2.1 of \cite{pmlr-v162-mishkin22a}.

A computationally efficient method for obtaining $\tilde{D}$ and obtaining the optimal values of the Equation \eqref{ReLU_1}, is laid out in \cite{pmlr-v162-mishkin22a}. In this method we first get our sample of diagonal matrices $\tilde{D}$ by first sampling a fixed number of vectors from a $d$ dimensional standard multivariate distribution, multiplying the vectors with the data matrix $X$ and then forming the diagonal matrices based of which co-ordinates are positive. Then  we solve an optimization similar to the one in Equation \eqref{ReLU_2}, without the constraints, that its parameters belong to sets of the form $\mathcal{K}_{i}$ as follows.
\begin{eqnarray}
\mathcal{L}^{'}_{\beta}(p) := \argminA_{p} \Bigg(\Bigg\|\sum_{D_{i} \in \tilde{D}}D_{i}(Xp_{i}) - y\Bigg\|^{2}_{2}\Bigg) ,\label{ReLU_3}
\end{eqnarray}
where $p \in \mathbb{R}^{d \times |\tilde{D}|}$ . In order to satisfy the constraints of the form given in Equation \eqref{ReLU_2}, this step is followed by a cone decomposition step. This is implemented through a function $\{\psi_{i}^{'}\}_{i \in \{1,\cdots,|\tilde{D}|\}}$.  Let $p^{*}=\{p^{*}_{i}\}_{i \in \{1,\cdots,|\tilde{D}|\}}$ be the optimal solution of Equation \eqref{ReLU_3}. For each $i$ we define a function $\psi_{i}^{'}:\mathbb{R}^{d} \rightarrow \mathbb{R}^{d}\times\mathbb{R}^{d}$ as 
\begin{eqnarray}
    \psi_{i}^{'}(p_{i}) &=& (v_{i},w_{i}) \label{ReLU_3_1}\\
     \textit{such that } p&=& {v}_{i} - {w}_{i}, \textit{and }   {v}_{i},{w}_{i} \in \mathcal{K}_{i} \nonumber
\end{eqnarray}
Then we obtain $\psi(p^{*}_{i})=(\bar{v}_{i},\bar{w}_{i})$. As before, at-least one of $v_{i}$, $w_{i}$ is $0$. Note that in practice we do not know if the conditions in Equation \eqref{ReLU_2_3} and \eqref{ReLU_2_3_1} are satisfied for a given sampled $\tilde{D}$. We express this as follows. If $\tilde{D}$  was the full set of Diagonal matrices  then we would have $(\bar{v}_{i},\bar{w}_{i})={v}^{*}_{i},{w}^{*}_{i}$ and $\psi(\bar{v}_{i},\bar{w}_{i})=(u_{i}^{*},\alpha_{i}^{*})$ for all $i \in (1,\cdots,|D_{X}|)$. However, since that is not the case and $\tilde{D} \in D_{X}$, this means that $\{\psi(\bar{v}_{i},\bar{w}_{i})\}_{i \in (1,\cdots,|\tilde{D}|)}$ is an optimal solution of a non-convex optimization different from the one in Equation \eqref{ReLU_1}. We denote this non-convex optimization as $\mathcal{L}_{|\tilde{D}|}(\theta)$ defined as 

\begin{equation}
    \mathcal{L}_{|\tilde{D}|}(\theta) = \argminA_{\theta} \Bigg\|\sum_{i=1}^{m^{'}}\sigma(Xu_{i})\alpha_{i}- y\Bigg\|_{2}^{2} \label{ReLU_6}, 
\end{equation}
where $m^{'} = |\tilde{D}|$  or the size of the sampled diagonal matrix set. In order to quantify the error incurred  due to taking a subset of $D_{X}$, we assume that the expectation of the absolute value of the difference between the neural networks corresponding to the optimal solutions of the non-convex optimizations given in Equations \eqref{ReLU_6} and \eqref{ReLU_1}  is upper bounded by a constant depending on the size of $\tilde{D}$. The formal assumption and its justification is given in Assumption \ref{assump_2}.

\section{Error Characterization}
Before we define the errors incurred during the actor and critic steps, we define some additional terms as follows

We define the Bellman operator for a policy $\pi$ as follows
\begin{equation}
(T^{\pi}Q)(s,a) =  r'(s,a) + \gamma \int Q^{\pi}(s',\pi(s'))P(ds'|s,a), \label{ps_3}
\end{equation}
where $r'(s,a) = \mathbb{E}(r(s,a)|(s,a))$ Similarly we define the Bellman Optimality Operator as 

Similarly we define the Bellman Optimality Operator as  
\begin{equation}
(TQ)(s,a) = \left( r' +\max_{a' \in \mathcal{A}}\gamma\int Q(s',a')P(ds'|s,a)\right),\label{ps_4}
\end{equation}

Further, operator $P^{\pi}$ is  defined as
\begin{equation}
P^{\pi}Q(s,a)=\mathbb{E}[Q(s',a')|s' \sim P(\cdot|s,a), a' \sim \pi(\cdot|s') ]\label{ps_7} ,    
\end{equation}
which is the one step Markov transition operator for policy $\pi$ for the Markov chain defined on $\mathcal{S}\times\mathcal{A}$ with the transition dynamics given by $S_{t+1} \sim P(\cdot|S_{t},A_{t})$ and $A_{t+1} \sim \pi(\cdot|S_{t+1})$. It defines a distribution on the state action space after one transition from the initial state. Similarly, $P^{\pi_{1}}P^{\pi_{2}}\cdots{P}^{\pi_{m}}$ is the $m$-step Markov transition operator following policy $\pi_t$ at steps $1\le t\le m$. It defines a distribution on the state action space after $m$ transitions from the initial state. We have the relation 
\begin{align}
(T^{\pi}Q)(s,a) =&   r' + \gamma \int Q^{\pi}(s',\pi(s'))P(ds'|s,a) \label{ps_7_2}
\\
         =&   r' + {\gamma}(P^{\pi}Q)(s,a).
\end{align}
We thus defines $P^{*}$ as 
\begin{equation}
P^{*}Q(s,a)=\max_{a^{'} \in \mathcal{A}}\mathbb{E}[Q(s',a')|s' \sim P(\cdot|s,a)]\label{ps_7_3} ,
\end{equation}
in other words, $P^{*}$ is the one step Markov transition operator with respect to the greedy policy of the function on which it is acting. which implies that
\begin{eqnarray}
(TQ)(s,a)  &=&   r' + {\gamma}(P^{*}Q)(s,a)  \label{ps_7_2_1}
\end{eqnarray}
For any measurable function $f:\mathcal{S}\times\mathcal{A}:\rightarrow\mathbb{R}$, we also define 
\begin{equation}
\mathbb{E}(f)_{\nu}=\int_{\mathcal{S}\times\mathcal{A}}fd\nu,  
\end{equation}
for any distribution $\nu\in\mathcal{P}(\mathcal{S}\times\mathcal{A})$.

We now characterize the errors which are incurred from the actor and critic steps. We  define as $\zeta^{\nu}_{\pi}(s,a)$ as the stationary state action distribution induced by the policy $\pi$ with the starting state action distribution drawn from a distribution $\nu \in \mathcal{P}(\mathcal{S}\times\mathcal{A})$. 
For the error incurred in the actor update we define the related loss function as
\begin{definition} \label{def_0}
   For iteration $k$ of the outer for loop of Algorithm \ref{algo_1} ,we define $w_{k}$ as the estimate of the minima of the loss function given by  $\mathbb{E}_{(s,a) \sim \zeta^{\nu}_{\pi}(s,a)}\left(A_{k,J}(s,a) - (w){\nabla}log(\pi_{\lambda_{k}})(a|s)\right)^{2}$ obtained at the end of the second inner for loop of Algorithm \ref{algo_1}. We further define the true minima as
   \begin{eqnarray}
   w^{*}_{k} &=&  \argminA_{w}\mathbb{E}_{(s,a) \sim \zeta^{\nu}_{\pi}(s,a)}\left(A_{k,J}(s,a) - (w){\nabla}log(\pi_{\lambda_{k}})(a|s)\right)^{2},
   \end{eqnarray}
\end{definition}
For finding the estimate $w_{k}$, we re-use the state action pairs sampled in the first inner for loop of Algorithm \ref{algo_1}. Note that we have to solve for the loss function where the expectation is with respect to the steady state distribution $\zeta^{\nu}_{\pi}(s,a)$, while our sample are from a markov chain which has the steady state distribution
For the error incurred in the critic update, we first define the various possible $Q$-functions which we can approximate in decreasing order of the accuracy.

For the error compnents incurred during critic estimation, we start by defining the best possible  approximation of the function $T^{\pi_{\lambda_{k}}}Q_{k,j-1}$ possible from the class of two layer ReLU neural networks, with respect to the expected square from the true ground truth $T^{\pi_{\lambda_{k}}}Q_{k,j-1}$.

\begin{definition} \label{def_1}
   For iteration $k$ of  the outer for loop  and iteration $j$ of the first inner for loop of Algorithm \ref{algo_1}, we define
   \begin{equation}
   Q^{1}_{k,j}=\argminA_{Q_{\theta},\theta \in \Theta}\mathbb{E}(Q_{\theta}(s,a) - T^{\pi_{\lambda_{k}}}Q_{k,j-1}(s,a))^{2},    
   \end{equation}
where $(s,a) \sim \zeta^{\nu}_{\pi}(s,a)$.
\end{definition}

Note that we do not have access to the transition probability kernel $P$, hence we cannot calculate $ T^{\pi_{\lambda_{k}}}$. To alleviate this, we use the observed next states to estimate the $Q$-value function. Using this, we define  $Q^{2}_{k,j}$ as, 

\begin{definition} \label{def_2}
     For iteration $k$ of  the outer for loop  and iteration $j$ of the first inner for loop of Algorithm \ref{algo_1}, we define
   \begin{eqnarray}
    Q^{2}_{k,j} = \argminA_{Q_{\theta},\theta \in \Theta}\mathbb{E}(Q_{\theta}(s,a)-(r'(s,a)+\gamma{\mathbb{E}}Q_{j-1}(s',a'))^{2},\label{temp}
   \end{eqnarray}
\end{definition}
where ${(s,a) \sim \zeta^{\nu}_{\pi}(s,a),s' \sim P(s'|s,a)}$ and $ {r'(\cdot|s,a)\sim {R}(\cdot|s,a)}$.

Compared to $Q^{1}_{k,j}$, in $Q^{2}_{k,j}$, we are minimizing the expected square loss from target function $\big(r'(s,a)+\gamma{\mathbb{E}_{a' \sim \pi_{k}}}Q_{j-1}(s',a')\big)$.

To obtain $Q^{2}_{k,j}$, we still need to compute the true expected value in Equation \ref{temp}. However, we still do not know the transition function $P$. To remove this limitation, we use sampling. Consider a set, $\mathcal{X}_{k,j}$ , of state-action pairs sampled  as where $(s,a) \sim \zeta^{\nu}_{\pi}(s,a)$. We now define $Q^{3}_{k,j}$ as,
\begin{definition} \label{def_3}
   For a given set of state action pairs $\mathcal{X}_{k,j}$  we define
   \begin{eqnarray}
    Q^{3}_{k,j}=\argminA_{Q_{\theta},\theta \in \Theta}\frac{1}{|\mathcal{X}_{k,j}|} \sum_{(s_{i},a_{i}) \in \mathcal{X}_{k,j}}\Big( Q_{\theta}(s_{i},a_{i})  
        - \big(r(s_{i},a_{i}) + \gamma{\mathbb{E}_{a' \sim \pi_{k}}}Q_{k,j-1}(s_{i+1},a') \big)\Big)^{2},
   \end{eqnarray}
where $r(s_{i},a_{i})$, and $s_{i+1}$ are the observed reward and the observed next state for state action pair $s_i, a_i$ respectively.
\end{definition}

$Q^{3}_{k,j}$ is the best possible approximation for $Q$-value function which minimizes the sample average of the square loss functions with the target values as $ \big(r'(s_{i},a_{i})+\gamma{\mathbb{E}_{a' \sim \pi_{k}}}Q_{k,j-1}(s_{i+1},a') \big)^{2}$ or the empirical loss function. % Since this is the loss function based on empirically observed state action pairs, it is known as the empirical loss function. 
After defining the possible solutions for the $Q$-values using different loss functions, we define the errors.

We first define approximation error which represents the difference between $T^{\pi_{\lambda_{k}}}Q_{j-1}$ and its best approximation possible from the class of 2 layer ReLU neural networks. We have
\begin{definition}[Approximation Error] \label{def_4}
    For a given iteration $k$ of Algorithm \ref{algo_1} and iteration $j$ of the first for loop of Algorithm \ref{algo_1}, we define, $\epsilon^{1}_{k,j} =T^{\pi_{\lambda_{k}}}Q_{k,j-1} - Q^{1}_{k,j}$, where $Q^{1}_{k,j}$ is the estimate of the $Q$ function at the iteration $j-1$ of the second for loop of Algorithm \ref{algo_1}.
\end{definition}

We also define Estimation Error which denotes the error between the best approximation of $T^{\pi_{\lambda_{k}}}Q_{k,j-1}$ possible from a 2 layer ReLU neural network and $Q^{2}_{k,j}$. We demonstrate that these two terms are the same and this error is zero.
\begin{definition}[Estimation Error] \label{def_5}
   For a given iteration $k$ of Algorithm \ref{algo_1} and iteration $j$ of the first for loop of Algorithm \ref{algo_1}, $\epsilon^{2}_{k,j} =Q^{1}_{k,j} - Q^{2}_{k,j}$.
\end{definition}

We now define Sampling error which denotes the difference between the minimizer of expected loss function $Q^{2}_{k,j}$ and the minimizer of the empirical loss function using samples, $Q^{3}_{k,j}$. We will use Rademacher complexity results to upper bound this error.
\begin{definition}[Sampling Error] \label{def_6}
   For a given iteration $k$ of Algorithm \ref{algo_1} and iteration $j$ of the first for loop of Algorithm \ref{algo_1}, $\epsilon^{3}_{k,j} =Q^{3}_{k,j} - Q^{2}_{k,j}$. 
\end{definition}

Lastly, we define optimization error which denotes the difference between the minimizer of the empirical square loss function, $Q_{k_3}$, and our estimate of this minimizer that is obtained from the projected gradient descent algorithm.
\begin{definition}[Optimization Error] \label{def_7}
   For a given iteration $k$ of Algorithm \ref{algo_1} and iteration $j$ of the first for loop of Algorithm \ref{algo_1}, $\epsilon^{4}_{k} =Q^{3}_{k,j}- Q_{k,j}$. Here $Q_{k,j}$ is our estimate of the $Q$ function at iteration $k$ of Algorithm \ref{algo_1} and iteration $j$ of the first inner loop of Algorithm \ref{algo_1}. 
\end{definition}

\section{Supplementary lemmas and Definitions}\label{sup_lem}

Here we provide some definitions and results that will be used to prove the lemmas stated in the  paper.

\begin{definition} \label{def_8}
   For a given set $ Z \in \mathbb{R}^{n}$, we define the Rademacher complexity of the set $Z$ as 
   \begin{equation}
   Rad(Z) = \mathbb{E} \left(\sup_{z \in Z} \frac{1}{n} \sum_{i=1}^{d}\Omega_{i}z_{i}\right)    
   \end{equation}
   where $\Omega_{i}$ is random variable such that $P(\Omega_{i}=1)=\frac{1}{2}$,  $P(\Omega_{i}=-1)=\frac{1}{2}$ and $z_{i}$ are the co-ordinates of $z$ which is an element of the set $Z$
\end{definition}

\begin{lemma} \label{sup_lem_0}
Consider a set of observed data denoted by $ z = \{z_{1},z_{2},\cdots\,z_{n}\} \in \mathbb{R}^{n}$, a parameter space  $\Theta$, a loss function $\{l:\mathbb{R} \times \Theta \rightarrow \mathbb{R}\}$ where  $0 \le l(\theta,z) \le 1$  $\forall (\theta,z) \in \Theta \times \mathbb{R}$. The empirical risk for a set of observed data as $R(\theta)=\frac{1}{n} \sum_{i=1}^{n}l(\theta,z_{i})$ and the population risk  as $r(\theta)= \mathbb{E}l(\theta,\tilde{z_{i}})$, where $\tilde{z_{i}}$ is a co-ordinate of $\tilde{z}$ sampled from some distribution over $Z$.

We define a set of functions denoted by $\mathcal{L}$ as 

\begin{equation}
    \mathcal{L}=\{z \in Z \rightarrow l(\theta,z) \in \mathbb{R}:\theta \in \Theta \}
\end{equation}

Given $z=\{z_{1},z_{2},z_{3}\cdots,z_{n}\}$ we further define a set $\mathcal{L} \circ z$ as 

\begin{equation}
    \mathcal{L} \circ z \ =\{ (l(\theta,z_{1}),l(\theta,z_{2}),\cdots,l(\theta,z_{n})) \in \mathbb{R}^{n} : \theta \in \Theta\}
\end{equation}

Then, we have 

\begin{equation}
    \mathbb{E}\sup_{\theta \in \Theta} |\{r(\theta)-R(\theta)\}| \le 2\mathbb{E} \left(Rad(\mathcal{L} \circ z)\right)
\end{equation}

If the data is of the form $z_{i}=(x_{i},y_{i}), x \in X, y \in Y$ and the loss function is of the form $l(a_{\theta}(x),y)$, is $L$ lipschitz and $a_{\theta}:\Theta{\times}X \rightarrow \mathbb{R}$, then we have 

\begin{equation}
    \mathbb{E}\sup_{\theta \in \Theta} |\{r(\theta)-R(\theta)\}| \le 2{L}\mathbb{E} \left(Rad(\mathcal{A} \circ \{x_{1},x_{2},x_{3},\cdots,x_{n}\})\right)
\end{equation}

where \begin{equation}
    \mathcal{A} \circ \{x_{1},x_{2},\cdots,x_{n}\}\ =\{ (a(\theta,x_{1}),a(\theta,x_{2}),\cdots,a(\theta,x_{n})) \in \mathbb{R}^{n} : \theta \in \Theta\}
\end{equation}

\end{lemma}

The detailed proof of the above statement is given in (Rebeschini,  2022)\footnote{ Algorithmic Foundations of Learning [Lecture Notes]. https://www.stats.ox.ac.uk/~rebeschi/teaching/AFoL/20/material/}. The upper bound for $ \mathbb{E}\sup_{\theta \in \Theta} (\{r(\theta)-R(\theta)\})$ is proved in the aformentioned reference. However, without loss of generality the same proof holds for the upper bound for $ \mathbb{E}\sup_{\theta \in \Theta} (\{R(\theta)-r(\theta)\})$. Hence the upper bound for $ \mathbb{E}\sup_{\theta \in \Theta}|\{r(\theta)-R(\theta)\}|$ can be established.

\begin{lemma} \label{sup_lem_1}
Consider two random random variable $x \in \mathcal{X} $ and  $y,y^{'} \in \mathcal{Y}$. Let $\mathbb{E}_{x,y}, \mathbb{E}_{x}$ and $\mathbb{E}_{y|x}$, $\mathbb{E}_{y^{'}|x}$  denote the expectation with respect to the joint distribution of $(x,y)$, the marginal distribution of $x$, the conditional distribution of $y$ given $x$ and the conditional distribution of $y^{'}$ given $x$ respectively . Let $f_{\theta}(x)$ denote a bounded measurable function of $x$ parameterised by some parameter $\theta$ and $g(x,y)$ be bounded measurable function of both $x$ and $y$.

Then we have

\begin{equation}
    \argminA_{f_{\theta}}\mathbb{E}_{x,y}\left(f_{\theta}(x)-g(x,y)\right)^{2}=\argminA_{f_{\theta}} \left(\mathbb{E}_{x,y}\left(f_{\theta}(x)-\mathbb{E}_{y^{'}|x}(g(x,y^{'})|x)\right)^{2}\right) \label{sup_lem_1_1}
\end{equation}    
\end{lemma}

\begin{proof}
Denote the left hand side of Equation \eqref{sup_lem_1_1} as $\mathbb{X}_{\theta}$, then add and subtract $\mathbb{E}_{y|x}(g(x,y)|x)$ to it to get 
\begin{align}
     \mathbb{X}_{\theta}=& \argminA_{f_{\theta}}\left(\mathbb{E}_{x,y}\left(f_{\theta}(x)-\mathbb{E}_{y^{'}|x}(g(x,y^{'})|x)+\mathbb{E}_{y^{'}|x}(g(x,y^{'})|x)-g(x,y)\right)^{2}\right) \label{sup_lem_1_2}
     \\
     =&  \argminA_{f_{\theta}}\Big(\mathbb{E}_{x,y}\left(f_{\theta}(x)-\mathbb{E}_{y^{'}|x}(g(x,y^{'})|x)\right)^{2} + \mathbb{E}_{x,y}\left(y-\mathbb{E}_{y^{'}|x}(g(x,y^{'})|x)\right)^{2} 
     \nonumber
     \\ 
     &\qquad-2\mathbb{E}_{x,y}\Big(f_{\theta}(x)-\mathbb{E}_{y^{'}|x}(g(x,y^{'})|x)\Big)\left(g(x,y)-\mathbb{E}_{y^{'}|x}(g(x,y^{'})|x)\right)\Big) .\label{sup_lem_1_3}
\end{align}
Consider the third term on the right hand side of Equation \eqref{sup_lem_1_3}
\begin{align}
    2\mathbb{E}_{x,y}&\left(f_{\theta}(x)-\mathbb{E}_{y^{'}|x}(g(x,y^{'})|x)\right)\left(g(x,y)-  \mathbb{E}_{y^{'}|x}(g(x,y^{'})|x)\right) 
    \nonumber
    \\
    =& 2\mathbb{E}_{x}\mathbb{E}_{y|x}\left(f_{\theta}(x)-\mathbb{E}_{y^{'}|x} (g(x,y^{'})|x)\right)
  \left(g(x,y)-\mathbb{E}_{y^{'}|x}(g(x,y^{'})|x)\right)\label{sup_lem_1_4}
    \\
    =& 2\mathbb{E}_{x}\left(f_{\theta}(x)-\mathbb{E}_{y^{'}|x}(g(x,y^{'})|x)\right)\mathbb{E}_{y|x}\left(g(x,y)-\mathbb{E}_{y^{'}|x}(g(x,y^{'})|x)\right) \label{sup_lem_1_5}
    \\
    =& 2\mathbb{E}_{x}\left(f_{\theta}(x)-\mathbb{E}_{y^{'}|x}(g(x,y^{'})|x)\right)\left(\mathbb{E}_{y|x}(g(x,y))-\mathbb{E}_{y|x}\left(\mathbb{E}_{y^{'}|x}(g(x,y^{'})|x)\right)\right)
    \label{sup_lem_1_6}
    \\
    =& 2\mathbb{E}_{x}\left(f_{\theta}(x)-\mathbb{E}(y|x)\right)\Big(\mathbb{E}_{y|x}(g(x,y))  -\mathbb{E}_{y^{'}|x}(g(x,y^{'})|x)\Big)  \label{sup_lem_1_7}\\
    =& 0
\end{align}

Equation \eqref{sup_lem_1_4} is obtained by writing $\mathbb{E}_{x,y}=\mathbb{E}_{x}\mathbb{E}_{y|x}$ from the law of total expectation. Equation \eqref{sup_lem_1_5} is obtained from  \eqref{sup_lem_1_4} as the term $f_{\theta}(x)-\mathbb{E}_{y^{'}|x}(g(x,y^{'})|x)$ is not a function of $y$. Equation \eqref{sup_lem_1_6} is obtained from \eqref{sup_lem_1_5} as $\mathbb{E}_{y|x}\left(\mathbb{E}_{y^{'}|x}(g(x,y^{'})|x)\right)=\mathbb{E}_{y^{'}|x}(g(x,y^{'})|x)$ because $\mathbb{E}_{y^{'}|x}(g(x,y^{'})|x)$ is not a function of $y$ hence is constant with respect to the expectation operator $\mathbb{E}_{y|x}$. 

Thus plugging in value of $  2\mathbb{E}_{x,y}\left(f_{\theta}(x)-\mathbb{E}_{y^{'}|x}(g(x,y^{'})|x)\right)\left(g(x,y)-  \mathbb{E}_{y^{'}|x}(g(x,y^{'})|x)\right)$ in Equation \eqref{sup_lem_1_3} we get 
\begin{align}    
\arg\min A_{f_{\theta}}\mathbb{E}_{x,y}\left(f_{\theta}(x)-g(x,y)\right)^{2} =  & \argminA_{f_{\theta}} (\mathbb{E}_{x,y}\left(f_{\theta}(x)-\mathbb{E}_{x,y^{'}}(g(x,y^{'})|x)\right)^{2} 
\nonumber\\
&+ \mathbb{E}_{x,y}\left(g(x,y)-\mathbb{E}_{y^{'}|x}(g(x,y^{'})|x)\right)^{2}). \label{sup_lem_1_8}
\end{align}
Note that the second term on the right hand side of Equation \eqref{sup_lem_1_8} des not depend on $f_{\theta}(x)$ therefore we can write Equation \eqref{sup_lem_1_8} as 

\begin{equation}
    \argminA_{f_{\theta}}\mathbb{E}_{x,y}\left(f_{\theta}(x)-g(x,y)\right)^{2} =  \argminA_{f_{\theta}} \left(\mathbb{E}_{x,y}\left(f_{\theta}(x)-\mathbb{E}_{y^{'}|x}(g(x,y^{'})|x)\right)^{2}\right) \label{sup_lem_1_9}
\end{equation}

Since the right hand side of Equation \eqref{sup_lem_1_9} is not a function of $y$ we can replace $\mathbb{E}_{x,y}$ with $\mathbb{E}_{x}$ to get 

\begin{equation}
    \argminA_{f_{\theta}}\mathbb{E}_{x,y}\left(f_{\theta}(x)-g(x,y)\right)^{2} =  \argminA_{f_{\theta}} \left(\mathbb{E}_{x}\left(f_{\theta}(x)-\mathbb{E}_{y^{'}|x}(g(x,y^{'})|x)\right)^{2}\right) \label{sup_lem_1_10}
\end{equation}
\end{proof}

\begin{lemma} \label{sup_lem_3}
Consider an optimization of the form given in Equation \eqref{ReLU_2_2} with the regularization term $\beta = 0$ denoted by $\mathcal{L}_{|\tilde{D}|}$ and it's convex equivalent denoted by $\mathcal{L}_{0}$. Then the value of these two loss functions evaluated at $(v,w)=(v_{i},w_{i})_{i \in \{1,\cdots,|\tilde{D}|\}}$ and $\theta=\psi(v_{i},w_{i})_{i \in \{1,\cdots,|\tilde{D}|\}}$ respectively are equal and thus we have

\begin{equation}
\mathcal{L}_{|\tilde{D}|}(\psi(v_{i},w_{i})_{i \in \{1,\cdots,|\tilde{D}|\}}) = \mathcal{L}_{0}((v_{i},w_{i})_{i \in \{1,\cdots,|\tilde{D}|\}})   
\end{equation}

\end{lemma}

\begin{proof}
    Consider the loss functions in Equations  \eqref{ReLU_2}, \eqref{ReLU_3} with $\beta=0$ are as follows

\begin{eqnarray}
    \mathcal{L}_{0}((v_{i},w_{i})_{i \in \{1,\cdots,|\tilde{D}|\}}) &=& ||\sum_{D_{i} \in \tilde{D}}D_{i}(X(v_{i} - w_{i})) - y||^{2}_{2} \label{sup_lem_3_1}\\ 
    \mathcal{L}_{|\tilde{D}|}(\psi(v_{i},w_{i})_{i \in \{1,\cdots,|\tilde{D}|\}}) &=& ||\sum_{i=1}^{|\tilde{D}|}\sigma(X\psi(v_{i},w_{i})_{1})\psi(v_{i},w_{i})_{2}- y||_{2}^{2},  \label{sup_lem_3_2}
\end{eqnarray}

where $\psi(v_{i},w_{i})_{1}$, $\psi(v_{i},w_{i})_{2}$ represent the first and second coordinates of $\psi(v_{i},w_{i})$ respectively.

For any fixed $i \in \{1,\cdots,|\tilde{D}|\}$ consider the two terms 

\begin{eqnarray}
D_{i}(X(v_{i}-w_{i})) \label{sup_lem_3_3}\\
\sigma(X\psi(v_{i},w_{i})_{1})\psi(v_{i},w_{i})_{2}   \label{sup_lem_3_4}
\end{eqnarray}

For a fixed $i$ either $v_{i}$ or $w_{i}$ is zero. In case both are zero, both of the terms in Equations \eqref{sup_lem_3_3} and  \eqref{sup_lem_3_4} are zero as $\psi(0,0)=(0,0)$. Assume that for a given $i$ $w_{i}=0$. Then we have $\psi(v_{i},w_{i})=(v_{i},1)$. Then equations \eqref{sup_lem_3_3}, \eqref{sup_lem_3_4}  are.

\begin{eqnarray}
D_{i}(X(v_{i})  \label{sup_lem_3_5}\\
\sigma(X(v_{i}))    \label{sup_lem_3_6}
\end{eqnarray}

But by definition of $v_{i}$ we have $D_{i}(X(v_{i})=\sigma(X(v_{i}))$, therefore Equations \eqref{sup_lem_3_5}, \eqref{sup_lem_3_6} are equal. Alternatively if for a given $i$ $v_{i}=0$, then $\psi(v_{i},w_{i})=(w_{i},-1)$, then the terms in \eqref{sup_lem_3_3}, \eqref{sup_lem_3_4}  become.

\begin{eqnarray}
-D_{i}(X(w_{i})  \label{sup_lem_3_7}\\
-\sigma(X(w_{i}))    \label{sup_lem_3_8}
\end{eqnarray}

By definition of $w_{i}$ we have $D_{i}(X(w_{i})=\sigma(X(w_{i}))$, then the terms 
in \eqref{sup_lem_3_7}, \eqref{sup_lem_3_7} are equal. 
Since this is true for all $i$, we have 

\begin{equation}
\mathcal{L}_{|\tilde{D}|}(\psi(v_{i},w_{i})_{i \in \{1,\cdots,|\tilde{D}|\}}) = \mathcal{L}_{0}((v_{i},w_{i})_{i \in \{1,\cdots,|\tilde{D}|\}}) 
 \label{sup_lem_3_9}    
\end{equation}

\end{proof}

\begin{lemma} \label{sup_lem_5}
The function  $Q_{\theta}(x)$ defined in equation \eqref{ReLU_1_2} is Lipschitz continuous in $\theta$, where $\theta$ is considered a vector in $\mathbb{R}^{(d+1)m}$ with the assumption that the set of all possible $\theta$ belong to the set  $\mathcal{B} = \{ \theta:  |\theta^{*}-\theta|_{1} < 1\}$, where $\theta^{*}$ is some fixed value.
\end{lemma}
\begin{proof}

First we show that for all $\theta_{1} = \{u_{i},\alpha_{i}\}, \theta_{2}= \{u^{'}_{i},\alpha^{'}_{i}\} \in \mathcal{B}$  we have $\alpha_{i}=\alpha^{'}_{i}$ for all $i \in (1,\cdots,m)$

Note that 

\begin{equation}
    |\theta_{1} - \theta_{2}|_{1} = \sum_{i=1}^{m}|u_{i}-u^{'}_{i}|_{1}  +  \sum_{i=1}^{m}|\alpha_{i}-\alpha^{'}_{i}|,
\end{equation}

where $|u_{i}-u^{'}_{i}|_{1} = \sum_{j=1}^{d}|u_{i_{j}}-u^{'}_{i_{j}}|$ with $u_{i_{j}}, u^{'}_{i_{j}}$ denote the $j^{th}$ component of $u_{i}, u^{'}_{i}$ respectively.

By construction $\alpha_{i}, \alpha^{'}_{i}$ can only be $1$, $-1$ or $0$. Therefore if $\alpha_{i}\neq\alpha^{'}_{i}$ then $|\alpha_{i}-\alpha^{'}_{i}|=2$ if both non zero or $|\alpha_{i}-\alpha^{'}_{i}|=1$ if one is zero. Therefore $|\theta_{1} - \theta_{2}|_{1} \geq 1$. Which leads to a contradiction.  

Therefore  $\alpha_{i}=\alpha^{'}_{i}$ for all $i$ and we also have 

\begin{equation}
    |\theta_{1} - \theta_{2}|_{1} = \sum_{i=1}^{m}|u_{i}-u^{'}_{i}|_{1}
\end{equation}

$Q_{\theta}(x)$ is defined as 

\begin{equation}
Q_{\theta}(x)=\sum_{i=1}^{m}\sigma^{'}(x^{T}u_{i})\alpha_{i} \label{sup_lem_5_1}
\end{equation}

From Proposition $1$ in \cite{10.5555/3327144.3327299} the function $Q_{\theta}(x)$ is Lipschitz continuous in $x$, therefore there exist $l > 0$ such that 

\begin{eqnarray}
|Q_{\theta}(x)- Q_{\theta}(y)|  &\le&    l|x-y|_{1} \label{sup_lem_5_2} \\
|\sum_{i=1}^{m}\sigma^{'}(x^{T}u_{i})\alpha_{i} - \sum_{i=1}^{m}\sigma^{'}(y^{T}u_{i})\alpha_{i}|  &\le& l|x-y|_{1}  \label{sup_lem_5_3}
\end{eqnarray}

If we consider a single neuron of $Q_{\theta}$, for example $i=1$, we have  $l_{1} > 0$ such that 
\begin{eqnarray}
|\sigma^{'}(x^{T}u_{1})\alpha_{i} - \sigma^{'}(y^{T}u_{1})\alpha_{i}|  &\le&  l_{1}|x-y|_{1}  \label{sup_lem_5_4}
\end{eqnarray}

Now consider Equation \eqref{sup_lem_5_4}, but instead of considering the left hand side a a function of $x,y$ consider it a function of $u$ where we consider the difference between $\sigma^{'}(x^{T}u)\alpha_{i}$ evaluated at $u_{1}$ and $u^{'}_{1}$ such that 
\begin{eqnarray}
|\sigma^{'}(x^{T}u_{1})\alpha_{i} - \sigma^{'}(x^{T}u^{'}_{1})\alpha_{i}|  &\le&  l^{x}_{1}|u_{1}-u^{'}_{1}|_{1}  \label{sup_lem_5_5}
\end{eqnarray}

for some $l^{x}_{1} > 0$.

Similarly, for all other $i$ if we change $u_{i}$ to $u^{'}_{i}$ to be unchanged we have 

\begin{eqnarray}
|\sigma^{'}(x^{T}u_{i})\alpha_{i} - \sigma^{'}(x^{T}u^{'}_{i})\alpha_{i}|  &\le&  l^{x}_{i}|u_{i}-u^{'}_{i}|_{1}  \label{sup_lem_5_6}
\end{eqnarray}

for all $x$ if both $\theta_{1}, \theta_{2} \in \mathcal{B}$.

Therefore we obtain

\begin{eqnarray}
|\sum_{i=1}^{m}\sigma^{'}(x^{T}u_{i})\alpha_{i} - \sum_{i=1}^{m}\sigma^{'}(x^{T}u^{'}_{i})\alpha_{i}|  &\le&  \sum_{i=1}^{m}|\sigma^{'}(x^{T}u_{i})\alpha_{i} -(x^{T}u^{'}_{i})\alpha_{i}|   \label{sup_lem_5_7}\\
                                                                                                       &\le&  \sum_{i=1}^{m}l^{x}_{i}|u_{i}-u^{'}_{i}|_{1}  \label{sup_lem_5_8}\\
                                                                                                       &\le&  (\sup_{i}l_{i}^{x})\sum_{i=1}^{m}|u_{i}-u^{'}_{i}|_{1}  \label{sup_lem_5_9}\\
                                                                                                       &\le&  (\sup_{i}l_{i}^{x})|\theta_{1}-\theta_{2}|  \label{sup_lem_5_10}
\end{eqnarray}

This result for a fixed $x$. If we take the supremum over $x$ on both sides we get

\begin{eqnarray}
\sup_{x}|\sum_{i=1}^{m}\sigma^{'}(x^{T}u_{i})\alpha_{i} - \sum_{i=1}^{m}\sigma^{'}(x^{T}u^{'}_{i})\alpha_{i}| 
                                                                                                       &\le&  (\sup_{i,x}l_{i}^{x})|\theta_{1}-\theta_{2}|  \label{sup_lem_5_11}
\end{eqnarray}

Denoting $(\sup_{i,x}l_{i}^{x})=l$, we get 
\begin{eqnarray}
|\sum_{i=1}^{m}\sigma^{'}(x^{T}u_{i})\alpha_{i} - \sum_{i=1}^{m}\sigma^{'}(x^{T}u^{'}_{i})\alpha_{i}| 
                                                                                                       &\le&  l|\theta_{1}-\theta_{2}|_{1}  \label{sup_lem_5_12}\\
                                                                                                       && \forall x \in  \mathbb{R}^{d}
\end{eqnarray}
\end{proof}

\section{Supporting Lemmas}
We will now state the key lemmas that will be used for finding the sample complexity of the proposed algorithm. 

\begin{lemma} \label{lem_1}
    For a given iteration $k$ of Algorithm \ref{algo_1} and iteration $j$ of the first for loop of Algorithm \ref{algo_1}, the approximation error denoted by $\epsilon^{1}_{k,j}$ in Definition \ref{def_4}, we have 
    \begin{equation}
        \mathbb{E}\left(|\epsilon^{1}_{k,j}|\right) \le \sqrt{\epsilon_{bias}},
    \end{equation}
\end{lemma}

Where the expectation is with respect to and $(s,a) \sim \zeta^{\nu}_{\pi}(s,a)$

\textit{Proof Sketch:} We use Assumption \ref{assump_3} and the definition of the variance of a random variable to obtain the required result. The detailed proof is given in Appendix \ref{proof_lem_1}. 

%{\bf Renumber Section numbers in Appendix to A, B, .. }

\begin{lemma} \label{lem_2}
     For a given iteration $k$ of Algorithm \ref{algo_1} and iteration $j$ of the first for loop of Algorithm \ref{algo_1},  $Q^{1}_{k,j}=Q^{2}_{k,j}$, or equivalently $\epsilon^{2}_{k,j}=0$
\end{lemma}

\textit{Proof Sketch:} We use  Lemma \ref{sup_lem_1} in Appendix \ref{sup_lem} and use the definitions of $Q^{1}_{k,j}$ and $Q^{2}_{k,j}$ to prove this result. The detailed proof is given in Appendix \ref{proof_lem_2}.

\begin{lemma} \label{lem_3}
    For a given iteration $k$ of Algorithm \ref{algo_1} and iteration $j$ of the first for loop of Algorithm \ref{algo_1}, if the number of state action pairs sampled are denoted by $n_{k,j}$, then the error $\epsilon^{3}_{k,j}$ defined in Definition \ref{def_6} is upper bounded as
    \begin{equation}
\mathbb{E}\left(|\epsilon^{3}_{k,j}|\right)\le  \tilde{\mathcal{O}}\left(\frac{log(log(n_{k,j}))}{\sqrt{n_{k,j}}}\right), 
    \end{equation}
\end{lemma}
Where the expectation is with respect to and $(s,a) \sim \zeta^{\nu}_{\pi}(s,a)$

\textit{Proof Sketch:} First we note that For a given iteration $k$ of Algorithm \ref{algo_1} and iteration $j$ of the first for loop of Algorithm \ref{algo_1},  $\mathbb{E}(R_{X_{k,j},Q_{k,j-1}}({\theta})) = L_{Q_{j,k-1}}({\theta})$ where $R_{X_{k,j},Q_{j,k-1}}({\theta})$ and $L_{Q_{j,k-1}}({\theta})$ are defined in Appendix \ref{proof_lem_3}. We use this to get a probabilistic bound on the expected value of $|(Q^{2}_{j,k}) - (Q^{3}_{j,k})|$ using Rademacher complexity theory when the samples are drawn from an ergodic Markov chain. The detailed proof is given in Appendix \ref{proof_lem_3}. Note the presence of the $log(log(n_{k,j}))$ term is due to the fact that the state action samples belong to a Markov Chain.

\begin{lemma} \label{lem_4}
      For a given iteration $k$ of Algorithm \ref{algo_1} and iteration $j$ of the first for loop of Algorithm \ref{algo_1}, let the number of steps of the projected gradient descent performed by Algorithm \ref{algo_2}, denoted by $T_{k,j}$, and the gradient descent step size $\alpha_{k,j}$ satisfy
    \begin{eqnarray}
      \alpha_{k,j} &=& \frac{||u^{*}_{k,j}||_{2}}{L_{k,j}\sqrt{T_{k,j}+1}},
    \end{eqnarray}
  for some constants $L_{k,j}$ and  $||\left(u_{k}^{*}\right)||_{2}$. Then the error $\epsilon_{k_{4}}$ defined in Definition \ref{def_7} is upper bounded as
    \begin{equation}
       \mathbb{E}(|\epsilon^{4}_{k,j}|) \le \tilde{\mathcal{O}}\left(\frac{1}{\sqrt{T_{k,j}}}\right) + {\epsilon_{|\tilde{D}|}},
    \end{equation}
\end{lemma}
Where the expectation is with respect to $(s,a) \sim \zeta^{\nu}_{\pi}(s,a)$.

\textit{Proof Sketch:} We use the number of iterations $T_{k,j}$ required to get an $\epsilon$ bound on the difference between the minimum objective value and the objective value corresponding to the estimated parameter at iteration $T_{k}$. We use the convexity of the objective and the Lipschitz property of the neural network to get a bound on the $Q$ functions corresponding to the estimated parameters. The detailed proof is given in Appendix \ref{proof_lem_4}.

\begin{lemma} \label{lem_6}
    For a given iteration $k$ of Algorithm \ref{algo_1} and iteration $j$ of the first for loop of Algorithm \ref{algo_1}, if the number of samples of the state action pairs sampled are denoted by $n_{k,j}$ and $\beta_{i}$ be the step size in the projected gradient descent at iteration $i$ of the second inner for loop of Algorithm \ref{algo_1} which satisfies
    \begin{eqnarray}
        \beta_{i} = \frac{\mu_{k}}{i+1},
    \end{eqnarray}
 where $\mu_{k}$ is the strong convexity parameter of $F_{k}$.  Then, it holds that, %
    \begin{equation}
        \left(F_{k}(w_{i})\right) \le  \tilde{\mathcal{O}}\left(\frac{log(n_{k,j})}{{n_{k,j}}}\right) + F^{*}_{k}.
    \end{equation}
\end{lemma}

\textit{Proof Sketch:} Note that we don not have access to state action samples belonging to the stationary state action distribution corresponding to the policy $\pi_{{\lambda}_{k}}$. We only have access to samples from Markov chain with the same stationary state action distribution.  To account for this, we use the results in \cite{9769873} and obtain the difference between the optimal loss function and the loss function obtained by performing stochastic gradient descent with samples from a Markov chain.  
\section{Proof of Theorem \ref{thm}} \label{thm proof} 
\begin{proof}

From Assumption \ref{assump_1}, we have 

\begin{eqnarray}
    \log\frac{\pi_{\lambda_{k+1}}(a|s)}{\pi_{\lambda_{k}}(a|s)} &\ge& {\nabla}_{\lambda_{k}}\log{\pi_{\lambda_{k}}}(a|s){\cdot} (\lambda^{k+1} - \lambda^{k}) - \frac{\beta}{2}||\lambda^{k+1} - \lambda^{k}||_{2}^{2} \label{proof_1}\\
    &=& {\eta}\log{\pi_{\lambda_{k}}}(a|s){\cdot}w^{k} - {\eta}\frac{\beta}{2}||w^{k}||_{2}^{2}\label{proof_2}
\end{eqnarray}

Thus we have,

\begin{eqnarray}
    \log\frac{\pi_{\lambda_{k+1}}(a|s)}{\pi_{\lambda_{k}}(a|s)} &\ge& {\nabla}_{\lambda_{k}}\log{\pi_{\lambda_{k}}}(a|s){\cdot} (\lambda^{k+1} - \lambda^{k}) - \frac{\beta}{2}||\lambda^{k+1} - \lambda^{k}||_{2}^{2} \label{proof_3}\\
    &=& {\eta}\log{\pi_{\lambda_{k}}}(a|s){\cdot}w^{k} - {\eta}^{2}\frac{\beta}{2}||w^{k}||_{2}^{2}\label{proof_4}
\end{eqnarray}

From the definition of KL divergence and from the performance difference lemma from \citep{kakade2002approximately} we have

\begin{align}
    \mathbb{E}_{s \sim d_{\nu}^{\pi^{*}}} \left(KL(\pi^{*}||{\pi}^{\lambda_{k}}) - \pi^{*}||{\pi}^{\lambda_{k+1}}) \right) =& \mathbb{E}_{s \sim d_{\nu}^{\pi^{*}}}\mathbb{E}_{a \sim \pi^{*}(.|s)} \left[log \frac{\pi^{\lambda_{k+1}}(a|s)}{\pi_{\lambda_{k}}(a|s)}\right] \label{proof_5}
    \\
    \leq & {\eta}\mathbb{E}_{s \sim d_{\nu}^{\pi^{*}}}\mathbb{E}_{a \sim \pi^{*}(.|s)} \left[{\nabla}_{\lambda_{k}}\log{\pi_{\lambda_{k}}}(a|s){\cdot}w^{k}\right] - {\eta}^{2}\frac{\beta}{2}||w^{k}||_{2}^{2}\label{proof_6}
    \\
   =& {\eta}\mathbb{E}_{s \sim d_{\nu}^{\pi^{*}}}\mathbb{E}_{a \sim \pi^{*}(.|s)} \left[Q_{k,J}(s,a)\right] - {\eta}^{2}\frac{\beta}{2}||w^{k}||_{2}^{2} \nonumber
    \\
    &- 
    {\eta}\mathbb{E}_{s \sim d_{\nu}^{\pi^{*}}}\mathbb{E}_{a \sim \pi^{*}(.|s)} \Big[{\nabla}_{\lambda_{k}}\log{\pi_{\lambda_{k}}}(a|s){\cdot}w^{k}  -A_{k,J}(s,a)\Big] \label{proof_7}
    \\
    =& (1-\gamma){\eta}\left(V^{\pi^{*}}(\nu) - V^{k}(\nu)\right) - {\eta}^{2}\frac{\beta}{2}||w^{k}||_{2}^{2} - 
    {\eta}{\cdot}err_{k} \label{proof_8}.
\end{align}

Equation \eqref{proof_6} is obtained from Equation \eqref{proof_5} using the result in Equation \eqref{proof_2}. \eqref{proof_7} is obtained from Equation \eqref{proof_6} using the performance difference lemma form \citep{kakade2002approximately} where $A_{k,J}$ is the advantage function to the corresponding $Q$ function $Q_{k,j}$.

Rearranging, we get 

\begin{eqnarray}
    \left(V^{\pi^{*}}(\nu) - V^{k}(\nu)\right)  
    &\le&  \frac{1}{1-\gamma}\left( \frac{1}{\eta}\mathbb{E}_{s \sim d_{\nu}^{\pi^{*}}} \left(KL(\pi^{*}||{\pi}^{\lambda_{k}}) - \pi^{*}||{\pi}^{\lambda_{k+1}}) \right) + {\eta}^{2}\frac{\beta}{2}{\cdot}{W}^{2} + 
    {\eta}{\cdot}err_{k} \right) \nonumber\\
    && \label{proof_9}
\end{eqnarray}
Summing from $1$ to $K$ and dividing by $K$ we get
\begin{align}
    \frac{1}{K}{\sum_{k=1}^{K}}\left(V^{\pi^{*}}(\nu) - V^{k}(\nu)\right) 
    \le&  \left(\frac{1}{1-\gamma}\right)\frac{1}{K}{\sum_{k=1}^{K}} \left(\mathbb{E}_{s \sim  d_{\nu}^{\pi^{*}}} \left(KL(\pi^{*}||{\pi}^{\lambda_{k}}) - \pi^{*}||{\pi}^{\lambda_{k+1}}) \right) + {\eta}{\cdot}err_{k} \right) \nonumber\\
    +& \left(\frac{1}{1-\gamma}\right){\eta}^{2}\frac{\beta}{2}{\cdot}{W}^{2}  \label{proof_10} \\
    \le&  \frac{1}{{\eta}(1-\gamma)}\frac{1}{K}\mathbb{E}_{s \sim \tilde{d}} \left(KL(\pi^{*}||{\pi}^{\lambda_{0}})\right) + \frac{{\eta}{\beta}{\cdot}{W}^{2}}{2(1-\gamma)}   + \frac{1}{K(1-\gamma)} \sum_{k=1}^{J-1}err_{k} \label{proof_11}
    \\
    \le&  \frac{\log(|\mathcal{A}|)}{K{\eta}(1-\gamma)} + \frac{{\eta}{\beta}{\cdot}{W}^{2}}{2(1-\gamma)}   + \frac{1}{K(1-\gamma)} \sum_{k=1}^{J-1}err_{k} \label{proof_12}
\end{align}
If we set $\eta = \frac{1}{\sqrt{K}}$ in Equation \eqref{proof_12} we get 
\begin{eqnarray}
    \frac{1}{K}{\sum_{k=1}^{K}}\left(V^{\pi^{*}}(\nu) - V^{k}(\nu)\right) 
    &\le&  \frac{1}{\sqrt{K}} \left(\frac{2\log(|\mathcal{A}|) +\beta{\cdot}{W}^{2}}{2(1-\gamma)}\right)   + \frac{1}{K(1-\gamma)} \sum_{k=1}^{J-1}err_{k} \label{proof_13}
\end{eqnarray}
Now consider the term $err_{k}$, we have from Equation \eqref{main_res_2}

\begin{eqnarray}
    err_{k} &=& \mathbb{E}_{s,a}(A^{\pi_{\lambda_{k}}} - w^{k}{\nabla}log(\pi^{{\theta}_{k}}(a|s))) \\
&=& \mathbb{E}_{s,a}(A^{\pi_{\lambda_{k}}} - A_{{k,J}}) + \mathbb{E}_{s,a}(A_{{k,J}} - w^{k}{\nabla}log(\pi^{{\theta}_{k}}(a|s))) \nonumber \\
&\le&  \underbrace{|\mathbb{E}_{s,a}(A^{\pi_{\lambda_{k}}} - A_{{k,J}})|}_{I} +  \underbrace{|\mathbb{E}_{s,a}(A_{{k,J}} - w^{k}{\nabla}log(\pi^{{\theta}_{k}}(a|s)))|}_{II}  \label{proof_13_1}
\end{eqnarray}
where $A_{k,j}$ is the estimate of $A^{\pi_{\lambda_{k}}}$ obtained at the $k^{th}$ iteration of Algorithm \ref{algo_1} and $s \sim d_{\nu}^{\pi^{*}}, a \sim \pi^{*}$. 

We first derive bounds on $I$. From the definition of advantage function we have 

\begin{eqnarray}
    |\mathbb{E}(A^{\pi_{\lambda_{k}}}(s,a) - A_{{k,J}}(s,a))| &=& |\mathbb{E}_{s \sim d_{\nu}^{\pi^{*}}, a \sim \pi^{*}}(Q^{\pi_{\lambda_{k}}}(s,a) - E_{a^{'} \sim \pi^{\lambda_{k}}}Q^{\pi_{\lambda_{k}}}(s,a^{'}) \nonumber\\
    && - Q_{{k,J}}(s,a) + E_{a^{'}\sim \pi^{\lambda_{k}}}Q_{k,J}(s,a^{'}))| \label{proof_13_1_1}\\
    &=& |\mathbb{E}_{s \sim d_{\nu}^{\pi^{*}}, a \sim \pi^{*}}(Q^{\pi_{\lambda_{k}}}(s,a) - E_{a^{'} \sim \pi^{\lambda_{k}}}Q^{\pi_{\lambda_{k}}}(s,a^{'}) \nonumber\\
    && - Q_{{k,J}}(s,a) + E_{a^{'}\sim \pi^{\lambda_{k}}}Q_{k,J}(s,a^{'}))|  \label{proof_13_1_2}\\
    &\le& |\mathbb{E}_{s \sim d_{\nu}^{\pi^{*}}, a \sim \pi^{*}}(Q^{\pi_{\lambda_{k}}}(s,a) -  Q_{{k,J}}(s,a)| \nonumber\\
    &+& |\mathbb{E}_{s \sim d_{\nu}^{\pi^{*}}, a^{'} \sim \pi^{\lambda_{k}}}(Q^{\pi_{\lambda_{k}}}(s,a) -  Q_{{k,J}}(s,a)| \label{proof_13_1_3}
\end{eqnarray}

We write the second term on the right hand side of Equation \eqref{proof_13_1_3} as $\int(|(Q^{\pi_{\lambda_{k}}}(s,a) -  Q_{{k,J}|}(s,a))d(\mu_{k})$ where $\mu_{k}$ is the measure associated with the state action distribution given by ${s \sim d_{\nu}^{\pi^{*}}, a \sim \pi^{\lambda_{k}}}$. Then we have

\begin{eqnarray}
    \int|Q^{\pi_{\lambda_{k}}}(s,a) -  Q_{{k,J}}(s,a)|d(\mu_{k}) \le  \Bigg|\Bigg|\frac{d{\mu_{k}}}{d{\mu^{*}}}\Bigg|\Bigg|_{\infty}\int|Q^{\pi_{\lambda_{k}}}(s,a) -  Q_{{k,J}}(s,a)|d(\mu^{*})  \label{proof_13_1_4}
\end{eqnarray}

where $\mu^{*}$ is the measure associated with the state action distribution given by $s \sim d_{\nu}^{\pi^{*}}, a^{'} \sim {\pi^{*}}$. Using Assumption \ref{assump_6} we have $\Bigg|\Bigg|\frac{d{\mu_{k}}}{d{\mu^{*}}}\Bigg|\Bigg|_{\infty} \le \phi_{\mu^{*},\mu_{k}}$. Thus Equation \eqref{proof_13_1_4} becomes

\begin{eqnarray}
    \int|Q^{\pi_{\lambda_{k}}}(s,a) -  Q_{{k,J}}(s,a)|d(\mu_{k}) \le  (\phi_{\mu_{k},\mu^{*}})\int(|Q^{\pi_{\lambda_{k}}}(s,a) -  Q_{{k,J}}(s,a)|d(\mu^{*})  \label{proof_13_1_5}
\end{eqnarray}

Since $|\int(Q^{\pi_{\lambda_{k}}}(s,a) -  Q_{{k,J}}(s,a)|d(\mu^{*}) = |\mathbb{E}_{s \sim d_{\nu}^{\pi^{*}}, a^{'} \sim \pi^{*}}(Q^{\pi_{\lambda_{k}}}(s,a) -  Q_{{k,J}}(s,a)|$ Equation \eqref{proof_13_1_3} now becomes.

\begin{eqnarray}
    |\mathbb{E}(A^{\pi_{\lambda_{k}}}(s,a) - A_{{k,J}}(s,a))| &\le& (1 +\phi_{\mu_{k},\mu^{*}})|\mathbb{E}_{s \sim d_{\nu}^{\pi^{*}}, a \sim \pi^{*}}(Q^{\pi_{\lambda_{k}}}(s,a) -  Q_{{k,J}}(s,a))| \nonumber\\
    \label{proof_13_1_6}
\end{eqnarray}

Therefore minimizing $A^{\pi_{\lambda_{k}}}(s,a) - A_{{k,J}}(s,a)$ is equivalent to minimizing $Q^{\pi_{\lambda_{k}}}(s,a) -  Q_{{k,J}}(s,a)$.

In order to prove the bound on $|\mathbb{E}_{s \sim d_{\nu}^{\pi^{*}}, a \sim \pi^{*}}(Q^{\pi_{\lambda_{k}}}(s,a) -  Q_{{k,J}}(s,a))|$ we first define some notation, let $Q_{1}, Q_{2}$ be two real valued functions on the state action space. The expression $Q_{1} \ge Q_{2}$ implies $Q_{1}(s,a) \ge Q_{2}(s,a)$ $\forall (s,a) \in \mathcal{S}\times\mathcal{A}$. 

   Let $Q_{k,j}$ denotes our estimate of the action value function at iteration $k$ of Algorithm \ref{algo_1} and iteration $j$ of the first for loop of Algorithm \ref{algo_1} and we denote $Q_{k,J} = Q_{k,j} $ where $J$ is the total number of iterations of the first for loop of Algorithm \ref{algo_1}.   $Q^{\pi_{\lambda_{k}}}$ denotes the action value function induced by the policy $\pi_{\lambda_{k}}$.
 
    Consider $\epsilon_{k,j+1}= T^{\pi_{\lambda_{k}}}Q_{k,j}-Q_{k,j+1}$.

    \begin{eqnarray}
        TQ_{k,j} \geq T^{\pi_{\lambda_{k}}} Q_{k,j} \label{thm_1_1}
    \end{eqnarray}

This follows from the definition of $T^{\pi_{\lambda_{k}}}$ and $T$ in Equation \eqref{ps_3} and \eqref{ps_4}, respectively.

Thus we get,
    \begin{align}
        Q^{\pi_{\lambda_{k}}}-Q_{k,j+1} =& T^{\pi_{\lambda_{k}}}Q^{\pi_{\lambda_{k}}} -  Q_{k,j+1}  \label{thm_1_1_1}\\
                      =& T^{\pi_{\lambda_{k}}}Q^{\pi_{\lambda_{k}}} -  T^{\pi_{\lambda_{k}}}Q_{k,j} + T^{\pi_{\lambda_{k}}}Q_{k,j} - TQ_{k,j} + TQ_{k,j} - Q_{k,j+1}  \label{thm_1_1_2}\\
                      =& r(s,a) + \gamma P^{\pi_{\lambda_{k}}}Q^{\pi_{\lambda_{k}}} - (r(s,a) +  \gamma P^{\pi_{\lambda_{k}}}Q_{k,j})   + (r(s,a) +  \gamma P^{\pi_{\lambda_{k}}}Q_{k,j}) \nonumber
                      \\
                      &-  (r(s,a) +  \gamma P^{*}Q_{k,j}) + \epsilon_{k,j+1}    \nonumber \\  
                      =& \gamma P^{\pi_{\lambda_{k}}}(Q^{\pi_{\lambda_{k}}}-Q_{k,j}) + \gamma P^{\pi_{\lambda_{k}}}Q_{k,j} - \gamma P^{*}Q_{k,j} + \epsilon_{k,j+1} \label{thm_1_1_4}\\  
                      \le& \gamma(P^{\pi_{\lambda_{k}}}(Q^{\pi_{\lambda_{k}}}-Q_{k,j})) + \epsilon_{k,j+1} \label{thm_1_1_5}
    \end{align}

Right hand side of Equation \eqref{thm_1_1_1} is obtained by writing $Q_{k,J} = T^{\pi^{*}}Q^{\pi_{\lambda_{k}}}$. This is because the function $Q^{\pi_{\lambda_{k}}}$ is a stationary point with respect to the operator $T^{\pi^{*}}$. Equation \eqref{thm_1_1_2} is obtained from \eqref{thm_1_1_1} by adding and subtracting $T^{\pi^{*}}Q_{k,j}$. Equation \eqref{thm_1_1_5} is obtained from \eqref{thm_1_1_4} as $P^{\pi^{*}}Q_{k,j} \le P^{*}Q_{k,j}$ and $P^{*}$ is the operator with respect to the greedy policy of $Q_{k,j}$.

By recursion on $k$, we get,

\begin{equation}
    Q^{\pi_{\lambda_{k}}}-Q_{k,J} \leq \sum_{j=0}^{J-1} \gamma^{J-j-1} (P^{\pi_{\lambda_{k}}})^{J-j-1}\epsilon_{k,j} +  \gamma^{J} (P^{\pi_{\lambda_{k}}})^{J}(Q^{\pi_{\lambda_{k}}}-Q_{0}) \label{thm_1_2}
\end{equation}

using $TQ_{k,j} \ge T^{\pi^{*}}Q_{k,j}$ (from definition of $T^{\pi^{*}}$) and $TQ_{k,j} \ge T^{\pi_{\lambda_{k}}}Q_{k,j}$ from definition of operator $T$. 

From this we obtain

\begin{align}
    \mathbb{E}_{s \sim d_{\nu}^{\pi^{*}}, a \sim \pi^{*}}{(Q^{\pi_{\lambda_{k}}}-Q_{k,J})}  
    \leq & {\phi_{k}}\sum_{k=0}^{J-1} \gamma^{J-j-1}\mathbb{E}_{s \sim d_{\nu}^{\pi^{*}}, a \sim \pi^{*}}((P^{\pi_{\lambda_{k}}})^{K-J-1}\epsilon_{k,j}) \nonumber
    \\
    &+  \gamma^{J}\mathbb{E}_{s \sim d_{\nu}^{\pi^{*}}, a \sim \pi^{*}}(P^{\pi_{\lambda_{k}}})^{J}(Q^{\pi_{\lambda_{k}}}-Q_{0}) \label{proof_14} 
\end{align}

Taking absolute value on both sides of Equation \eqref{proof_14} we get.

\begin{align}
    |\mathbb{E}_{s \sim d_{\nu}^{\pi^{*}}, a \sim \pi^{*}}{(Q^{\pi_{\lambda_{k}}}-Q_{k,J})}|  
    \leq & {\phi_{k}}\sum_{k=0}^{J-1} \gamma^{J-j-1}\mathbb{E}_{s \sim d_{\nu}^{\pi^{*}}, a \sim \pi^{*}}((P^{\pi_{\lambda_{k}}})^{J-j-1}|\epsilon_{k,j}|) \nonumber
    \\
    &+  \gamma^{J}\mathbb{E}_{s \sim d_{\nu}^{\pi^{*}}, a \sim \pi^{*}}(P^{\pi_{\lambda_{k}}})^{K}(|Q^{\pi_{\lambda_{k}}}-Q_{0}|) \label{proof_15} 
\end{align}

For a fixed $j$ consider the term $\mathbb{E}_{s \sim d_{\nu}^{\pi^{*}}, a \sim \pi^{*}}((P^{\pi_{\lambda_{k}}})^{J-j-1}|\epsilon_{k,j}|)$ using Assumption \ref{assump_6} we write

\begin{eqnarray}
    |\mathbb{E}_{s \sim d_{\nu}^{\pi^{*}}, a \sim \pi_{j}}((P^{\pi_{\lambda_{k}}})^{J-j-1}|\epsilon_{k,j}|)| &\le& \Bigg|\Bigg|\frac{d{(P^{\pi_{\lambda_{k}}})^{J-j-1}\mu_{j}}}{d{\mu^{'}_{k}}}\Bigg|\Bigg|_{\infty} \left|{\int}\epsilon_{k,j}d{\mu^{'}_{k}}\right| \\ 
   &\le&  (\phi_{\mu^{'}_{k},\mu_{j}})  \mathbb{E}_{(s,a) \sim \zeta^{\nu}_{\pi}(s,a)}(|\epsilon_{k,j}|)| \label{proof_15_1}\\   
\end{eqnarray}

Here $\mu_{j}$ is the measure associated with the state action distribution given by sampling from  $s \sim d_{\nu}^{\pi^{*}}, a^{'} \sim {\pi^{*}}$ and then applying the operator $P^{\pi_{\lambda_{k}}}$ $J-j-1$ times. $\mu_{j}$ is the measure associated with the steady state action distribution given by $(s,a) \sim \zeta^{\nu}_{\pi}(s,a)$. Thus Equation \eqref{proof_15} becomes

\begin{align}
    |\mathbb{E}_{s \sim d_{\nu}^{\pi^{*}}, a \sim \pi^{*}}{(Q^{\pi_{\lambda_{k}}}-Q_{k,J})}|  
    \leq & \sum_{k=0}^{J-1} \gamma^{J-j-1} (\phi_{\mu^{'}_{k},\mu_{j}})  \mathbb{E}_{(s,a) \sim \zeta^{\nu}_{\pi}(s,a)}(|\epsilon_{k,j}|)|  +   \gamma^{J}\left(\frac{R_{max}}{1-\gamma}\right) \label{proof_15_2} 
\end{align}

We get the second term on the right hand side by noting that $(Q^{\pi_{\lambda_{k}}}-Q_{0}) \le \frac{R_{max}}{1-\gamma}$. Now splitting $\epsilon_{k,j}$ as  was done in Equation \eqref{last} we obtain

\begin{align}
    \mathbb{E}_{s \sim d_{\nu}^{\pi^{*}}, a \sim \pi^{*}}{(Q^{\pi_{\lambda_{k}}}-Q_{k,j})} 
    \leq& \sum_{j=0}^{J-1} \gamma^{J-j-1} \big((\phi_{\mu^{'}_{k},\mu_{j}})\mathbb{E}{|\epsilon^{1}_{k,j}|} + (\phi_{\mu^{'}_{k},\mu_{j}})\mathbb{E}{|\epsilon^{2}_{k,j}|} 
    \nonumber
    \\
    &+(\phi_{\mu^{'}_{k},\mu_{j}})\mathbb{E}{|\epsilon^{3}_{k,j}|} + (\phi_{\mu^{'}_{k},\mu_{j}})\mathbb{E}{|\epsilon^{4}_{k,j}|}\big) +  \gamma^{J}\left(\frac{R_{max}}{1-\gamma}\right) \label{proof_16}
\end{align}

Now using Lemmas \ref{proof_lem_1}, \ref{proof_lem_2}, \ref{proof_lem_3}, \ref{proof_lem_4} we have 

\begin{eqnarray}
    \mathbb{E}_{s \sim d_{\nu}^{\pi^{*}}, a \sim \pi^{*}}{(Q^{\pi_{\lambda_{k}}}-Q_{k,j})}  &=&
   \sum_{j=0}^{J-1}  \tilde{\mathcal{O}}\left(\frac{\log(log(n_{k,j}))}{\sqrt{n_{k,j}}}\right)     
                                +  (\sqrt{\epsilon_{approx}} +{\epsilon_{|\tilde{D}|}}) + \tilde{\mathcal{O}} \left(\frac{1}{\sqrt{T_{k,j}}}\right) \nonumber\\
                                &+& \tilde{\mathcal{O}}(\gamma^{J})\label{proof_17}
\end{eqnarray}

From Equation \eqref{proof_13_1_6} we get

\begin{eqnarray}
    \mathbb{E}_{s \sim d_{\nu}^{\pi^{*}}, a \sim \pi^{*}}{(A^{\pi_{\lambda_{k}}}-A_{k,j})}  &=&
   \sum_{j=0}^{J-1}  \tilde{\mathcal{O}}\left(\frac{\log(log(n_{k,j}))}{\sqrt{n_{k,j}}}\right)     
                                +  (\sqrt{\epsilon_{approx}} +{\epsilon_{|\tilde{D}|}}) + \tilde{\mathcal{O}} \left(\frac{1}{\sqrt{T_{k,j}}}\right) \nonumber\\
                                &+& \tilde{\mathcal{O}}(\gamma^{J})\label{proof_17_1}
\end{eqnarray}

We now derive bounds on $II$. From Theorem 2 of \cite{9769873} we have that

\begin{eqnarray}
  ||w_{k} - w^{*}||_{2}  &\le& \mathcal{O} \left(\frac{\log(n_{k,j})}{{n_{k,j}}}\right)  \label{proof_19}
\end{eqnarray}

Now define the function $F_{k}(w) = \mathbb{E}_{s,a \sim \zeta^{\nu}_{\pi}}(A_{{k,J}} - w^{k}{\nabla}log(\pi^{{\theta}_{k}}(a|s)))$. From this definition we obtain

\begin{eqnarray}
  F_{k}(w_{k}) - F_{k}(w^{*})  \le  l_{F_{k}}||w_{k} - w^{*}||_{2}  &\le& \mathcal{O} \left(\frac{\log(n_{k,j})}{{n_{k,j}}}\right)  \label{proof_19_1}
\end{eqnarray}

where $l_{F_{k}}$ is lipschitz constant of $F_{k}(w)$. Thus we obtain

\begin{eqnarray}
  F_{k}(w_{k}) - F_{k}(w^{*})  &\le& \mathcal{O} \left(\frac{\log(n_{k,j})}{{n_{k,j}}}\right)  \label{proof_19_2}
\end{eqnarray}

From Assumption \ref{assump_3_1} we have

\begin{eqnarray}
  F_{k}(w_{k}) - \epsilon_{bias}  &\le& \mathcal{O} \left(\frac{\log(n_{k,j})}{{n_{k,j}}}\right)  \label{proof_19_3}
\end{eqnarray}

which gives us 
\begin{eqnarray}
  F_{k}(w_{k})   &\le& \mathcal{O} \left(\frac{\log(n_{k,j})}{{n_{k,j}}}\right) + \epsilon_{bias} \label{proof_19_4}
\end{eqnarray}

Plugging Equations \eqref{proof_19_4} and  \eqref{proof_17_1} in Equation \eqref{proof_13} we get

\begin{align}
\frac{1}{K}\sum_{k=1}^{K}(V^{*}(\nu)-V^{\pi_{\lambda_{K}}}(\nu)) &\le \min_{k \le K}(V^{*}(\nu)-V^{\pi_{\lambda_{K}}}(\nu)) \\
\leq&    {\mathcal{O}}\left(\frac{1}{\sqrt{K}(1-\gamma)}\right) \!+\!  \frac{1}{K(1-\gamma)}\sum_{k=1}^{K-1}\sum_{j=1}^{J} \mathcal{O}\left(\frac{\log\log(n_{k,j})}{\sqrt{n_{k,j}}}\right) \nonumber\\
& + \sum_{k=1}^{K-1}\sum_{j=1}^{J}{\mathcal{O}}\left(\frac{1}{\sqrt{T_{k,j}}}\right) + \frac{1}{K(1-\gamma)}\sum_{k=1}^{K-1}{\mathcal{O}}(\gamma^{J}) \nonumber\\
& + \sum_{k=1}^{K-1}\sum_{j=1}^{J}\left(\frac{\log(n_{k,j})}{{n_{k,j}}}\right) + \frac{1}{1-\gamma}\left(\epsilon_{bias} + (\sqrt{\epsilon_{approx}}) + {\epsilon_{|\tilde{D}|}}\right)\\
\leq&    {\mathcal{O}}\left(\frac{1}{\sqrt{K}(1-\gamma)}\right) \!+\!  \frac{1}{K(1-\gamma)}\sum_{k=1}^{K-1}\sum_{j=1}^{J} \mathcal{O}\left(\frac{\log\log(n_{k,j})}{\sqrt{n_{k,j}}}\right) \ \nonumber
\\
& + \sum_{k=1}^{K-1}\sum_{j=1}^{J}{\mathcal{O}}\left(\frac{1}{\sqrt{T_{k,j}}}\right) + \frac{1}{K(1-\gamma)}\sum_{k=1}^{K-1}{\mathcal{O}}(\gamma^{J}) \nonumber\\
& + \frac{1}{1-\gamma}\left(\epsilon_{bias} + (\sqrt{\epsilon_{approx}}) + {\epsilon_{|\tilde{D}|}}\right)\label{proof_20_1}
\end{align}

\end{proof}

\section{Proof of Supporting Lemmas} \label{Upper Bounding Bellman Error}

\subsection{Proof Of Lemma \ref{lem_1}} \label{proof_lem_1}

\begin{proof}

Using Assumption \ref{assump_3} and the definition of $Q_{kj_{1}}$ for some iteration $k$ of Algorithm \ref{algo_1} we have  
\begin{equation}
\mathbb{E}(T^{\pi_{\lambda_{k}}}Q_{k,j-1} - Q^{1}_{k,j})^{2}_{\nu} \le \epsilon_{approx}  
\end{equation}

where $\nu  \in \mathcal{P}(\mathcal{S}{\times}{\mathcal{A}})$.

Since $|a|^{2}=a^{2}$ we obtain

\begin{equation}
\mathbb{E}(|T^{\pi_{\lambda_{k}}}Q_{k,j-1} - Q^{1}_{k,j}|)^{2}_{\nu} \le \epsilon_{approx}  
\end{equation}

We have for a random variable $x$, $ Var(x)=\mathbb{E}(x^{2}) - (\mathbb{E}(x))^{2}$ hence $\mathbb{E}(x) = \sqrt{\mathbb{E}(x^{2}) -Var(x)}$, Therefore replacing $x$ with $|T^{\pi_{\lambda_{k}}}Q^{\pi_{\lambda_{k}}} - Q_{k1}|$ we get

using the definition of the variance of a random variable we get  
\begin{equation}
\mathbb{E}(|T^{\pi_{\lambda_{k}}}Q_{k,j-1} - Q^{1}_{k,j}|)_{\nu}=\sqrt{\mathbb{E}(|T^{\pi_{\lambda_{k}}}Q_{k,j-1} - Q^{1}_{k,j}|)^{2}_{\nu} - Var(|T^{\pi_{\lambda_{k}}}Q_{k,j-1} - Q^{1}_{k,j}|)_{\nu}}  
\end{equation}

Therefore we get

\begin{equation}
    \mathbb{E}(T^{\pi_{\lambda_{k}}}Q_{k,j-1} - Q^{1}_{k,j}|)_{\nu} \le \sqrt{\epsilon_{approx}}
\end{equation}

Since $\epsilon_{k_1}=T^{\pi_{\lambda_{k}}}Q^{\pi_{\lambda_{k}}} - Q_{k1}$ we have 
\begin{equation}
    \mathbb{E}(|\epsilon_{k,j_1}|)_{\nu} \le \sqrt{\epsilon_{approx}}
\end{equation}
\end{proof}

\subsection{Proof Of Lemma \ref{lem_2}} \label{proof_lem_2}
\begin{proof}
From Lemma \ref{sup_lem_1}, we have 

\begin{equation}
    \argminA_{f_{\theta}}\mathbb{E}_{x,y}\left(f_{\theta}(x)-g(x,y)\right)^{2}=\argminA_{f_{\theta}} \left(\mathbb{E}_{x,y}\left(f_{\theta}(x)-\mathbb{E}(g(y^{'},x)|x)\right)^{2}\right) \label{lem_3_1}
\end{equation}   

We label $x$ to be the state action pair $(s,a)$, function $f_{\theta}(x)$ to be $Q_{\theta}(s,a)$ and $g(x,y)$ to be the function $r^{'}(s,a) + {\gamma}\sum_{a^{'} \in \mathcal{A}}\pi_{\lambda_{k}}(a^{'}|s,a)Q_{k,j-1}(s^{'},a^{'}) =  r^{'}(s,a) + {\gamma}\mathbb{E}Q_{k,j-1}(s^{'},a^{'}) $, where  $y$ is the two dimensional random variable $(r^{'}(s,a),s^{'})$ and $s \sim d^{\pi_{{\theta}_{k}}}_{\nu}, a \sim \pi_{\lambda_{k}}(.|s)$,$s^{'} \sim P(.|(s,a))$ and $r^{'}(s,a) \sim {R}(.|s,a)$.

Then the loss function in \eqref{sup_lem_1_1} becomes
\begin{equation}
\mathbb{E}_{s \sim d^{\pi_{{\lambda}_{k}}}_{\nu}, a \sim \pi_{\lambda_{k}}(.|s),s^{'} \sim P(s^{'}|s,a), r(s,a) \sim \mathcal{R}(.|s,a)}(Q_{\theta}(s,a)-(r^{'}(s,a)+{\gamma}\mathbb{E}Q_{k,j-1}(s^{'},a^{'})))^{2} \label{lem_3_2}
\end{equation}

Therefore by Lemma \ref{sup_lem_1}, we have that the function $Q_{\theta}(s,a)$ which minimizes Equation \eqref{lem_3_2} it will be minimizing

\begin{equation}
\mathbb{E}_{s \sim d^{\pi_{{\lambda}_{k}}}_{\nu}, a \sim \pi_{\lambda_{k}}}(Q_{\theta}(s,a) -\mathbb{E}_{s^{'} \sim P(s^{'}|s,a),r \sim \mathcal{R}(.|s,a))}(r^{'}(s,a)+{\gamma}\mathbb{E}Q_{k,j-1}(s^{'},a^{'})|s,a))^{2} \label{lem_3_3}
\end{equation}

But we have from Equation  that

\begin{equation}
\mathbb{E}_{s^{'} \sim P(s^{'}|s,a),r \sim {R}(.|s,a))}(r^{'}(s,a)+{\gamma}\mathbb{E}Q_{k,j-1}(s^{'},a^{'})|s,a) = T^{\pi_{\lambda_{k}}}Q_{k,j-1}\label{lem_3_4}
\end{equation}

Combining Equation \eqref{lem_3_2} and \eqref{lem_3_4} we get

\begin{equation}
\argminA_{Q_{\theta}}\mathbb{E}(Q_{\theta}(s,a)-(r(s,a) + {\gamma}\mathbb{E}Q_{k,j-1}(s^{'},a^{'})))^{2} = \argminA_{Q_{\theta}}\mathbb{E}(Q_{\theta}(s,a)-T^{\pi_{\lambda_{k}}}Q_{k,j-1})^{2} \label{lem_3_4_1}
\end{equation}

The left hand side of Equation \eqref{lem_3_4_1} is $Q^{2}_{k,j}$ as defined in Definition \ref{def_2} and the right hand side is  $Q^{1}_{k,j}$ as defined in Definition \ref{def_1}, which gives us 

\begin{equation}
    Q^{2}_{k,j}=Q^{1}_{k,j}
\end{equation}

\end{proof}

\subsection{Proof Of Lemma \ref{lem_3}}\label{proof_lem_3}

\begin{proof}
 
We define $R_{X_{k,j},Q_{k,j-1}}({\theta})$ as

\begin{equation}
    R_{X_{k,j},Q_{k,j-1}}({\theta}) = \frac{1}{|X_{k,j}|} \sum_{(s_{i},a_{i}) \in X_{k,j}}\Bigg( Q_{\theta}(s_{i},a_{i}) - \Bigg(r(s_{i},a_{i})\nonumber\\ 
        + \gamma\mathbb{E}_{a^{'} \sim \pi_{\lambda_{k}} }Q_{k,j-1}(s_{i+1},a^{'}) \Bigg)\Bigg)^{2},
\end{equation}

Here, $X_{k,j}=\{s_{i},a_{i}\}_{i=\{1,\cdots,|X_{k,j}|\}}$, where $s,a$ are sampled from a Markov chain whose stationary distribution is, $s \sim d_{\nu}^{\pi_{\lambda_{k}}}, a \sim \pi_{\lambda_{k}}$. $Q_{\theta}$ is as defined in Equation \eqref{ReLU_1_2} and $Q_{k,j-1}$ is the estimate of the $Q$ function obtained at iteration $k$ of the outer for loop and iteration $j-1$ of the first inner for loop of Algorithm \ref{algo_1}.

We also define the term

\begin{equation}
   L_{Q_{k,j-1}}(Q_{\theta}) = \mathbb{E}(Q_{\theta}(s,a)-(r'(s,a)+\gamma\mathbb{E}_{a^{'} \sim \pi_{\lambda_{k}}}Q_{k,j-1}(s',a'))^{2}  
\end{equation}

where ${s \sim d^{\pi_{{\theta}_{k}}}_{\nu}, a \sim \pi_{\lambda_{k}}(.|s), r'(\cdot|s,a)\sim {R}(\cdot|s,a)}\nonumber\\$

We denote by $\theta^{2}_{k,j}, \theta^{3}_{k,j}$ the parameters of the neural networks $Q^{2}_{k,j}, Q^{3}_{k,j}$ respectively. $Q^{2}_{k,j}, Q^{3}_{k,j}$ are defined in Definition \ref{def_2} and  \ref{def_3} respectively.  

We then obtain,

\begin{eqnarray}
    R_{X_{k,j},Q_{k,j-1}}(\theta^{2}_{k,j}) - R_{X_{k,j},Q_{k,j-1}}(\theta^{3}_{k,j}) &\le&  R_{X_{k,j},Q_{k,j-1}}(\theta^{2}_{k,j}) - R_{X_{k,j},Q_{k,j-1}} 
    (\theta^{3}_{k,j}) \nonumber\\  
                                                        &&     + L_{Q_{k,j-1}}(\theta^{2}_{k,j}) - L_{Q_{k,j-1}}(\theta^{3}_{k,j}) \nonumber\\  
                                                        &&    \label{2_2_2}\\
                                                        &=&   {R_{X_{k,j},Q_{k,j-1}}(\theta^{2}_{k,j})- L_{Q_{k,j-1}}(\theta^{2}_{k,j})} \nonumber\\ 
                                                        &&  - {R_{X_{k,j},Q_{k,j-1}}(\theta^{3}_{k,j}) + L_{Q_{k,j-1}}(\theta^{2}_{k,j})} \nonumber\\
                                                        &&\label{2_2_3}\\
                                                        &\le&   \underbrace{|R_{X_{k,j},Q_{k,j-1}}(\theta^{2}_{k,j})- L_{Q_{k,j-1}}(\theta^{2}_{k,j})|}_{I}  \nonumber\\
                                                        &&  + \underbrace{|R_{X_{k,j},Q_{k,j-1}}(\theta^{3}_{k,j})- L_{Q_{k,j-1}}(\theta^{3}_{k,j})|}_{II} \nonumber\\
                                                        && \label{2_2_4}
\end{eqnarray}

We get the inequality in Equation \eqref{2_2_2} because $L_{Q_{k,j-1}}(\theta^{3}_{k,j}) - L_{Q_{k,j-1}}(\theta^{2}_{k,j}) > 0$ as  $Q^{2}_{k,j}$ is the minimizer of the loss function $ L_{Q_{k,j-1}}(Q_{\theta})$.

Consider Lemma \ref{sup_lem_0}. The loss function  $R_{X_{k,j},Q_{k,j-1}}(\theta^{3}_{k,j})$ can be written as the mean of loss functions of the form $l(a_{\theta}(s_{i},a_{i}),y_{i})$ where $l$ is the square function. $a_{\theta}(s_{i}, a_{i})=Q_{\theta}(s_{i},a_{i})$  and $y_{i}=\Big(r^{'}(s_{i},a_{i}) + \gamma{\mathbb{E}}Q_{k,j-1}(s^{'},a^{'})\Big)$. Thus we have

\begin{eqnarray}
   & \mathbb{E}\sup_{\theta \in \Theta}|R_{X_{k,j},Q_{k,j-1}}(\theta)- L_{Q_{k,j-1}}(\theta)| \le  \\
  & 2{\eta}^{'}\mathbb{E} \left(Rad(\mathcal{A} \circ \{(s_{1},a_{1}),(s_{2},a_{2}),(s_{3},a_{3}),\cdots,(s_{n},a_{n})\})\right)\nonumber
\end{eqnarray}

Note that the expectation is over all $(s_{i},a_{i})$. Where  $n_{k,j} = |X_{k,j}|$, $(\mathcal{A} \circ \{(s_{1},a_{1}),(s_{2},a_{2}),$  $(s_{3},a_{3}),\cdots,(s_{n},a_{n})\} = \{Q_{\theta}(s_{1},a_{1}), Q_{\theta}(s_{2},a_{2}), \cdots, Q_{\theta}(s_{n},a_{n})\}$  and $\eta^{'}$ is the Lipschitz constant for the square function  over the state action space $[0,1]^{d}$. The expectation is with respect to $s \sim d^{\pi_{{\lambda}_{k}}}_{\nu}, a \sim \pi_{\lambda_{k}}(.|s)$   ,${s_{i}^{'} \sim P(s^{'}|s,a)}$  $r_{i} \sim R(.|s_{i},a_{i})_{_{i \in (1,\cdots,n_{k,j})}, }$. 

From proposition 11 of \cite{article}  we have that 

\begin{equation}
    \left(Rad(\mathcal{A} \circ \{(s_{1},a_{1}),(s_{2},a_{2}),(s_{3},a_{3}),\cdots,(s_{n},a_{n})\})\right) \le  C_{k}\frac{\log(\log(n_{k,j}))}{\sqrt{n_{k,j}}}
\end{equation}

Note that proposition 11 of \cite{article} establishes an upper bound on the Rademacher complexity using theorem 4 of \cite{article} with the aim of applying it to the Metropolis Hastings algorithm For our purpose we only use the upper bound on Rademacher complexity established in  proposition 11 of \cite{article}.

We use this result as the state action pairs are drawn not from the stationary state of the policy $\pi_{\lambda_{k}}$ but from a Markov chain with the same steady state distribution.
Thus we have
\begin{eqnarray}
   \mathbb{E}|(R_{X_{k,j},Q_{k,j-1}}(\theta^{2}_{k,j})) - L_{Q_{k,j-1}}(\theta^{2}_{k,j})| \le  C_{k}\frac{\log(\log(n_{k,j}))}{\sqrt{n_{k,j}}} \label{2_2_5}
\end{eqnarray}
The same argument can be applied for $Q^{3}_{k,j}$ to get
\begin{eqnarray}
   \mathbb{E}|(R_{X_{k,j},Q_{k,j-1}}(\theta^{3}_{k,j})) - L_{Q_{k,j-1}}(\theta^{3}_{k,j})| \le  C_{k}\frac{\log(\log(n_{k,j}))}{\sqrt{n_{k,j}}} \label{2_2_5_1}
\end{eqnarray}
Then we have 
\begin{equation}
\mathbb{E}\left(R_{X_{k,j},Q_{k,j-1}}(\theta^{2}_{k,j}) - R_{X_{k,j},Q_{k,j-1}}(\theta^{3}_{k,j})\right) \le C_{k}\frac{\log(\log(n_{k,j}))}{\sqrt{n_{k,j}}}  \label{2_2_5_2}
\end{equation}
Plugging in the definition of $R_{X_{k,j},Q_{k,j-1}}(\theta^{2}_{k,j}), R_{X_{k,j},Q_{k,j-1}}(\theta^{3}_{k,j})$ in equation  \eqref{2_2_5_2} and denoting $ C_{k}\frac{\log(\log(n_{k,j}))}{\sqrt{n_{k,j}}}$ as  $\epsilon$ we get

\begin{align}
     \frac{1}{n_{k,j}} \sum_{i=1}^{n_{k,j}} \Big(\mathbb{E}(Q^{2}_{k,j}(s_{i},a_{i})-(r(s_{i},a_{i}) + \mathbb{E}_{a^{'} \sim \pi^{\lambda_{k}}}Q^{2}_{k,j}(s_{i+1},a^{'})))^{2} \nonumber\\
    - \mathbb{E}(Q^{3}_{k,j}(s_{i},a_{i})-(r(s_{i},a_{i}) + \mathbb{E}_{a^{'} \sim \pi^{\lambda_{k}}}Q^{3}_{k,j}(s_{i+1},a^{'})))^{2} \Big) \le \epsilon\label{2_2_5_3_1}
\end{align}

Now for a fixed $i$ consider the term $\alpha_{i}$ defined as.

\begin{eqnarray}
   \mathbb{E}_{s_{i+1} \sim P(.|s_{i},a_{i})}(Q^{2}_{k,j}(s_{i},a_{i})-(r(s_{i},a_{i}) + \mathbb{E}_{a^{'} \sim \pi^{\lambda_{k}}}Q^{2}_{k,j}(s_{i+1},a^{'})))^{2}&& \nonumber\\
    -  \mathbb{E}_{s_{i+1} \sim P(.|s_{i},a_{i})}(Q^{3}_{k,j}(s_{i},a_{i})-(r(s_{i},a_{i}) + \mathbb{E}_{a^{'} \sim \pi^{\lambda_{k}}}Q^{3}_{k,j}(s_{i+1},a^{'})))^{2}    \label{2_2_5_3_2}
\end{eqnarray}

where $s_{i},a_{i}$ are drawn from the state action distribution  at the $i^{th}$ step of the Markov chain induced by following the policy $\pi_{\lambda_{k}}$.

Now for a fixed $i$ consider the term $\beta_{i}$ defined as.

\begin{eqnarray}
    \mathbb{E}_{s_{i+1} \sim P(.|s_{i},a_{i})}(Q^{2}_{k,j}(s_{i},a_{i})-(r(s_{i},a_{i}) + \mathbb{E}_{a^{'} \sim \pi^{\lambda_{k}}}Q^{2}_{k,j}(s_{i+1},a^{'})))^{2} \nonumber\\
    -  \mathbb{E}_{s_{i+1} \sim P(.|s_{i},a_{i})}(Q^{3}_{k,j}(s_{i},a_{i})-(r(s_{i},a_{i}) + \mathbb{E}_{a^{'} \sim \pi^{\lambda_{k}}}Q^{3}_{k,j}(s_{i+1},a^{'})))^{2}    \label{2_2_5_3_3}
\end{eqnarray}

where $s_{i},a_{i}$ are drawn from the steady state action distribution  with $s_{i} \sim d_{\nu}^{\pi_{\lambda_{k}}}$ and $a_{i} \sim \pi_{\lambda_{k}}$. Note here that $\alpha_{i}$ and $\beta_{i}$ are the same function with only the state action pairs being drawn from different distributions.

Using these definitions we obtain

\begin{eqnarray}
  |\mathbb{E}(\alpha_{i}) -  \mathbb{E}(\beta_{i})|  &\le& \sup_{(s_{i},a_{i})}|2.\max(\alpha_{i},\beta_{i})|({\kappa}_{i})   \label{2_2_5_3_4}\\
   &\le& \left(4\frac{R}{1-\gamma}\right)^{2}m{\rho}^{i} \label{2_2_5_3_5}
\end{eqnarray}

We obtain Equation \eqref{2_2_5_3_4} by using the identity $|{\int}fd{\mu} - {\int}fd{\nu}|  \le  |\max_{\mathcal{S}\times\mathcal{A}}(f)|\sup_{\mathcal{S}\times\mathcal{A}}{\int}(d{\mu}-d{\nu})| \le  |\max_{\mathcal{S}\times\mathcal{A}}(f)|{d_{TV}}(\mu,\nu)|$, where $\mu$ and $\nu$ are two $\sigma$ finite state action probability measures and $f$ is a bounded measurable function.  We have used  $\kappa_{i}$ to represent the total variation distance between the state action measures of the steady state action distribution denoted by $s_{i} \sim d_{\nu}^{\pi_{\lambda_{k}}},a_{i} \sim \pi_{\lambda_{k}}$ and the state action distribution at the  $i^{th}$ step of the Markov chain induced by following the policy $\pi^{\lambda_{k}}$. The expectation is with respect to $(s_{i},a_{i})$. We obtain Equation \eqref{2_2_5_3_5} from  Equation  \eqref{2_2_5_3_4} by using  Assumption \ref{assump_4} and the fact that $\alpha_{i}$ and $\beta_{i}$ are upper bounded by  $\left(4\frac{R}{1-\gamma}\right)^{2}$

From equation \eqref{2_2_5_3_5} we get 

\begin{eqnarray}
   \mathbb{E}(\alpha_{i})  &\ge& \mathbb{E}(\beta_{i}) - 4\left(\frac{R}{1-\gamma}\right)^{2}m{\rho}^{i} \label{2_2_5_3_6}
\end{eqnarray}

We get Equation \eqref{2_2_5_3_6} from Equation \eqref{2_2_5_3_5} using the fact that $|a-b| \le c$ implies that  $ \left(-c \ge   (a-b) \le c \right)$ which in turn implies $ a \ge b-c $.

Using Equation \eqref{2_2_5_3_6} in equation \eqref{2_2_5_3_3} we get
\begin{eqnarray}
     &&\frac{1}{n_{k,j}} \sum_{i=1}^{n_{k,j}} \Big(\mathbb{E}(Q^{2}_{k,j}(s_{i},a_{i})-(r(s_{i},a_{i}) + \mathbb{E}_{a^{'} \sim \pi^{\lambda_{k}}}Q^{2}_{k,j}(s_{i+1},a^{'})))^{2} \nonumber\\
   && - \mathbb{E}(Q^{3}_{k,j}(s_{i},a_{i})-(r(s_{i},a_{i}) + \mathbb{E}_{a^{'} \sim \pi^{\lambda_{k}}}Q^{3}_{k,j}(s_{i+1},a^{'})))^{2} \Big) \nonumber\\
   &\le& \epsilon  + \frac{1}{n_{k,j}} \sum_{i=1}^{n_{k,j}}4\left(\frac{R}{1-\gamma}\right)^{2}m{\rho}^{i} \nonumber\\
    &\le& \epsilon  + \frac{1}{n_{k,j}}4\left(\frac{R}{1-\gamma}\right)^{2}m\frac{1}{1-\rho} \nonumber\\
    \label{2_2_5_3_7}    
\end{eqnarray}

In Equation \eqref{2_2_5_3_7}  $(s_{i},a_{i})$ are now drawn from $s_{i} \sim d_{\nu}^{\pi_{\lambda_{k}}}$ and $a_{i} \sim \pi_{\lambda_{k}}$ for al $i$.

We ignore the second term on the right hand side as it is $\tilde{\mathcal{O}}\left(\frac{1}{n_{k,j}}\right)$ as  compared to the first term which is $\tilde{\mathcal{O}} \left(\frac{\log\log(n_{k,j})}{\sqrt{n_{k,j}}}\right)$. Additionally the expectation in Equation \eqref{2_2_5_3_7} is wiht respect to ${s_{i} \sim d^{\pi_{{\theta}_{k}}}_{\nu}, a_{i} \sim \pi_{\lambda_{k}}(.|s), r'(\cdot|s,a)\sim {R}(\cdot|s,a)}, s_{i+1} \sim P(.|s_{i},a_{i})\\$
 
Since now we have $s_{i} \sim d_{\nu}^{\pi_{\lambda_{k}}},a_{i} \sim \pi_{\lambda_{k}}$ for all $i$,  Equation \eqref{2_2_5_3_7} is equivalent to,

\begin{eqnarray}
\mathbb{E}\underbrace{(Q^{2}_{k,j}(s,a)-Q^{3}_{k,j}(s,a))}_{A1}\underbrace{(Q^{2}_{k,j}(s,a)+Q^{3}_{k,j}(s,a) - 2(r(s,a)) + \gamma \max_{a \in \mathcal{A}}Q_{k,j-1}(s^{'},a))}_{A2} &\le& \epsilon \nonumber\\
    && \label{2_2_5_4}
\end{eqnarray}

Where the expectation is now over $s \sim d_{\nu}^{\pi_{\lambda_{k}}},a \sim \pi_{\lambda_{k}}$, $r(s,a)\sim R(.|s,a)$ and $s^{'}\sim P(.|s,a)$.  
We re-write Equation \eqref{2_2_5_4} as 
\begin{align} \int\underbrace{(Q^{2}_{k,j}(s,a)-Q^{3}_{k,j}(s,a))}_{A1}\times&
    \nonumber
    \\  &\hspace{-1cm}\times \underbrace{(Q^{2}_{k,j}(s,a)+Q^{3}_{k,j}(s,a) - 2(r(s,a)) + \gamma \max_{a \in \mathcal{A}}Q_{k,j-1}(s^{'},a))}_{A2}\times&
    \nonumber
    \\
    &\times d{\mu_{1}}(s,a)d{\mu_{2}}(r)d{\mu_{3}}(s^{'}) \le \epsilon .
    \label{2_2_5_5}
\end{align}
Where $\mu_{1}$ is the state action distribution $s \sim d_{\nu}^{\pi_{\lambda_{k}}},a \sim \pi_{\lambda_{k}}$, $\mu_{2}$, $\mu_{3}$ are the measures with respect to  $(s,a)$, $r^{'}$ and $s^{'}$ respectively

Now for the integral in Equation \eqref{2_2_5_5} we split the integral into four different integrals. Each integral is over the set of $(s,a), r^{'}, s^{'}$ corresponding to the 4 different combinations of signs of $A1, A2$. 
\begin{align}
    &\int_{\{(s,a), r^{'}, s^{'}\}:A1\ge0,A2\ge0}(A1)(A2)d{\mu_{1}}(s,a)d{\mu_{2}}(r)d{\nu_{3}}(s^{'})
    \nonumber
    \\
    &+ \int_{\{(s,a), r^{'}, s^{'}\}:A1<0,A2<0}(A1)(A2)d{\mu_{1}}(s,a)d{\mu_{2}}(r)d{\mu_{3}}(s^{'})   \nonumber\\
   &+  \int_{\{(s,a), r^{'}, s^{'}\}:A1\ge0,A2<0}(A1)(A2)d{\mu_{1}}(s,a)d{\mu_{2}}(r)d{\nu_{3}}(s^{'}) 
   \nonumber
   \\
   &+  \int_{\{(s,a), r^{'}, s^{'}\}:A1<0,A2\ge0}(A1)(A2)d{\mu_{1}}(s,a)d{\mu_{2}}(r)d{\mu_{3}}(s^{'})   \le \epsilon 
    \label{2_2_5_6}
\end{align}
Now note that the first 2 terms are non-negative and the last two terms are non-positive. We then write the first two terms as 
\begin{eqnarray}
&&    \int_{\{(s,a), r^{'}, s^{'}\}:A1\ge0,A2\ge0}(A1)(A2)d(s,a)d{\mu_{1}}(s,a)d{\mu_{2}}(r)d{\mu_{3}}(s^{'}) \nonumber\\
&=& C_{{k,j}_{1}}\int|Q^{2}_{k,j}-Q^{3}_{k,j}|d{\mu_{1}} \nonumber\\
    &=&  C_{{k,j}_{1}}\mathbb{E}(|Q^{2}_{k,j}-Q^{3}_{k,j}|)_{\mu_{1}} \nonumber\\
    &&\label{2_2_5_7}\\
   && \int_{\{(s,a), r^{'}, s^{'}\}:A1<0,A2<0}(A1)(A2)d(s,a)d{\mu_{1}}(s,a)d{\mu_{2}}(r)d{\mu_{3}}(s^{'}) \nonumber\\&=& C_{{k,j}_{2}}\int|Q^{2}_{k,j}-Q^{3}_{k,j}|d{\nu} \nonumber\\
    &=&  C_{{k,j}_{2}}\mathbb{E}(|Q^{2}_{k,j}-Q^{3}_{k,j}|)_{\mu_{1}}  \nonumber\\
    && \label{2_2_5_8}
\end{eqnarray}
We write the last two terms as 
\begin{eqnarray}
     \int_{\{(s,a), r^{'}, s^{'}\}:A1\ge0,A2<0}(A1)(A2)d{\mu_{1}}(s,a)d{\mu_{2}}(r)d{\mu_{3}}(s^{'}) =C_{{k,j}_{3}}{\epsilon} \label{2_2_5_9}\\
    \int_{\{(s,a), r^{'}, s^{'}\}:A1<0,A2\ge0}(A1)(A2)d{\mu_{1}}(s,a)d{\mu_{2}}(r)d{\mu_{3}}(s^{'}) = C_{{k,j}_{4}}{\epsilon} \label{2_2_5_10}
\end{eqnarray}

Here $C_{{k,j}_{1}}, C_{{k,j}_{2}}, C_{{k,j}_{4}}$ and $C_{{k,j}_{4}}$ are positive constants. Plugging Equations \eqref{2_2_5_7}, \eqref{2_2_5_8}, \eqref{2_2_5_9}, \eqref{2_2_5_10} into Equation \eqref{2_2_5_5}.  

\begin{eqnarray}
    (C_{{k,j}_{1}}+C_{{k,j}_{2}})\mathbb{E}(|Q^{2}_{k,j}-Q^{3}_{k,j}|)_{\mu_{1}} -(C_{{k,j}_{3}}+C_{{k,j}_{4}})\epsilon &\le& \epsilon \label{2_2_5_11}\\
\end{eqnarray}

which implies 

\begin{eqnarray}
    \mathbb{E}(|Q^{2}_{k,j}-Q^{3}_{k,j}|)_{\mu_{1}} &\le& \left(\frac{1+C_{{k,j}_{3}}+C_{{k,j}_{4}}}{C_{{k,j}_{1}}+C_{{k,j}_{2}}}\right)\epsilon \label{2_2_5_12}\\
\end{eqnarray}

Thus we have 

\begin{eqnarray}
    \mathbb{E}(|Q^{2}_{k,j}-Q^{3}_{k,j}|)_{\mu_{1}} &\le& \left(\frac{1+C_{{k,j}_{3}}+C_{{k,j}_{4}}}{C_{{k,j}_{1}}+C_{{k,j}_{2}}}\right)C_{k}\frac{\log(\log)(n_{k,j})}{\sqrt{n_{k,j}}} \label{2_2_5_13}\\
\end{eqnarray}

which implies

\begin{eqnarray}
    \mathbb{E}(|Q^{2}_{k,j}-Q^{3}_{k,j}|)_{\mu_{1}} &\le& \tilde{\mathcal{O}}\left(\frac{\log(\log)(n_{k,j})}{\sqrt{n_{k,j}}}\right) \label{2_2_5_14}\\
\end{eqnarray}

\end{proof}

\subsection{Proof Of Lemma 4} \label{proof_lem_4}

\begin{proof}
For a given iteration $k$  of Algorithm \ref{algo_1} and iteration $j$ of the first inner for loop, the optimization problem to be solved in Algorithm \ref{algo_2} is the following 

\begin{equation}
   \mathcal{L}(\theta) = \frac{1}{n_{k,j}} \sum_{i=1}^{n_{k,j}} \left( Q_{\theta}(s_{i},a_{i}) - \left(r(s_{i},a_{i}) + \gamma\max_{a^{'} \in \mathcal{A}}\gamma Q_{k,j-1}(s^{'},a^{'}) \right) \right)^{2} \label{lem_4_1}
\end{equation}

Here, $Q_{k,j-1}$ is the estimate of the $Q$ function from the iteration $j-1$ and the state action pairs $(s_{i},a_{i})_{i=\{1,\cdots,n\}}$ have been sampled from a distribution over the state action pairs denoted by $\nu$. Since $\min_{\theta}\mathcal{L}(\theta)$ is a non convex optimization problem we instead solve the equivalent convex problem given by

\begin{eqnarray}
      u_{k,j}^{*} &=& \argminA_{u}g_{k,j}(u) =\argminA_{u}|| \sum_{D_{i} \in \tilde{D}}D_{i}X_{k,j}u_{i}-y_{k} ||^{2}_{2}   \label{lem_4_2} \\
      &&\textit{subject to}   |u|_{1} \le \frac{R_{\max}}{1-\gamma} \label{lem_4_2_1}
\end{eqnarray}

Here, $X_{k,j} \in \mathbb{R}^{n_{k,j} \times d}$ is the matrix of sampled state action pairs at iteration $k$, $y_{k} \in \mathbb{R}^{n_{k,j} \times 1}$ is the vector of target values at iteration $k$. $\tilde{D}$ is the set of diagonal matrices obtained from line \ref{a2_l1} of Algorithm \ref{algo_2} and $u \in \mathbb{R}^{|\tilde{D}d| \times 1}$ (Note that we are treating $u$ as a vector here for notational convenience instead of a matrix as was done in Section \ref{Proposed Algorithm}).

%Since we have set the regularization term to zero for our optimization, from \cite{pmlr-v162-mishkin22a} we have that the optimal solution for the problem given in Equation \eqref{lem_4_2} is the same as the optimal solution to a problem $\mathcal{L}_{|\tilde{D}|}(\theta)$ (as described in Section \ref{Problem Setup}) up-to a linear transformation.

The constraint in Equation \eqref{lem_4_2_1} ensures that the all the co-ordinates of the vector $\sum_{D_{i} \in \tilde{D}}D_{i}X_{k,j}u_{i}$  are upper bounded by $\frac{R_{max}}{1-\gamma}$ (since all elements of $X_{k,j}$ are between $0$ and $1$). This ensures that the corresponding neural network represented by Equation \eqref{ReLU_1_2} is also upper bounded by   $\frac{R_{max}}{1-\gamma}$. We use the a projected gradient descent to solve the constrained convex optimization problem which can be written as. 

\begin{eqnarray}
      u_{k,j}^{*} &=& \argminA_{u : |u|_{1} \le \frac{R_{\max}}{1-\gamma}}g_{k,j}(u) =\argminA_{u : |u|_{1} \le \frac{R_{\max}}{1-\gamma}}|| \sum_{D_{i} \in \tilde{D}}D_{i}X_{k,j}u_{i}-y_{k} ||^{2}_{2} \label{lem_4_2_2} 
\end{eqnarray}

From Ang, Andersen(2017). “Continuous Optimization” [Notes]. https://angms.science/doc/CVX  we have that if the step size $\alpha =\frac{||u^{*}_{k,j}||_{2}}{L_{k,j}\sqrt{T_{k,j}+1}}$, after  $T_{k,j}$ iterations of the projected gradient descent algorithm we obtain

\begin{equation}
(g_{k,j}(u_{T_{k,j}})-g_{k,j}(u^{*})) \le  L_{k,j} \frac{||u_{k,j}^{*}||_{2}}{\sqrt{T_{k,j}+1}}  \label{lem_4_2_3}
\end{equation}

Where $L_{k,j}$ is the lipschitz constant of $g_{k,j}(u)$ and $u_{T_{k,j}}$ is the parameter estimate at step $T_{k,j}$. 

Therefore if the number of iteration of the projected gradient descent algorithm $T_{k,j}$  and the step-size $\alpha$ satisfy

\begin{eqnarray}
T_{k,j} &\ge& L_{k,j}^{2}||u_{k,j}^{*}||^{2}_{2}\epsilon^{-2} - 1, \label{lem_4_4}\\
\alpha &=& \frac{||u^{*}_{k}||_{2}}{L_{k,j}\sqrt{T_{k,j}+1}},
\end{eqnarray}

we have 

\begin{equation}
(g_{k,j}(u_{T_{k,j}})-g_{k,j}(u^{*})) \le  \epsilon  \label{lem_4_5}
\end{equation}

Let $(v^{*}_{i},w^{*}_{i})_{i \in (1,\cdots,|\tilde{D}|)}$, $(v^{T_{k,j}}_{i},w^{T_{k,j}}_{i})_{i \in (1,\cdots,|\tilde{D}|)}$ be defined as 
\begin{eqnarray}
 (v^{*}_{i},w^{*}_{i})_{i \in (1,\cdots,|\tilde{D}|)} = \psi_{i}^{'}(u^{*}_{i})_{i \in (1,\cdots,|\tilde{D}|)} \label{lem_4_7}\\
 (v^{T_{k,j}}_{i},w^{T_{k,j}}_{i})_{i \in (1,\cdots,|\tilde{D}|)} = \psi_{i}^{'}(u^{T_{k,j}}_{i})_{i \in (1,\cdots,|\tilde{D}|)} \label{lem_4_8}
\end{eqnarray}

where $\psi^{'}$ is defined in Equation \eqref{ReLU_3_1}.

Further, we define  $\theta_{|\tilde{D}|}^{*}$ and $\theta^{T_{k,j}}$ as 
\begin{eqnarray}
\theta_{|\tilde{D}|}^{*} = \psi(v^{*}_{i},w^{*}_{i})_{i \in (1,\cdots,|\tilde{D}|)}   \label{lem_4_9}\\
\theta^{T_{k,j}} = \psi(v^{T_{k,j}}_{i},w^{T_{k,j}}_{i})_{i \in (1,\cdots,|\tilde{D}|)}   \label{lem_4_10}
\end{eqnarray}

 where  $\psi$ is defined in Equation \eqref{ReLU_2_1},  $\theta_{|\tilde{D}|}^{*} = \argminA_{\theta}\mathcal{L}_{|\tilde{D}|}(\theta)$ for $\mathcal{L}_{|\tilde{D}|}(\theta)$ defined in Appendix \ref{cones_apdx}.

Since  $(g(u_{T_{k,j}})-g(u^{*})) \le \epsilon$, then by Lemma \ref{sup_lem_3}, we have 

\begin{equation}
     \mathcal{L}_{|\tilde{D}|}(\theta^{T_{k,j}}) -\mathcal{L}_{|\tilde{D}|}(\theta_{|\tilde{D}|}^{*}) \le \epsilon  \label{lem_4_11}
\end{equation}

Note that $\mathcal{L}_{|\tilde{D}|}(\theta^{T_{k,j}}) -\mathcal{L}_{|\tilde{D}|}(\theta_{|\tilde{D}|}^{*})$ is a constant value. Thus we can always find constant $C^{'}_{k,j}$ such that

\begin{equation}
     C_{k}^{'}|\theta^{T_{k,j}}-\theta_{|\tilde{D}|}^{*}|_{1}   \le  \mathcal{L}_{|\tilde{D}|}(\theta^{T_{k,j}}) -\mathcal{L}_{|\tilde{D}|}(\theta_{|\tilde{D}|}^{*})  \label{lem_4_14}
\end{equation}

\begin{equation}
     |\theta^{T_{k,j}}-\theta_{|\tilde{D}|}^{*}|_{1}   \le  \frac{\mathcal{L}(\theta^{T_{k,j}}) -\mathcal{L}(\theta^{*})}{C_{k}^{'}}  \label{lem_4_15}\\
\end{equation}

Therefore if we have 

\begin{eqnarray}
T_{k,j} &\ge& L_{k,j}^{2}||u_{k,j}^{*}||^{2}_{2}\epsilon^{-2} - 1, \label{lem_4_16}\\
\alpha_{k,j} &=& \frac{||u^{*}_{k}||_{2}}{L_{k,j}\sqrt{T_{k,j}+1}},
\end{eqnarray}

then we have 

\begin{equation}
     |\theta^{T_{k,j}}-\theta^{*}|_{1}   \le \frac{\epsilon}{C_{k}^{'}}  \label{lem_4_17}\\
\end{equation}

which according to Equation \eqref{lem_4_15} implies that

\begin{equation}
      C_{k}^{'}|\theta^{T_{k,j}}-\theta_{|\tilde{D}|}^{*}|_{1}   \le  \mathcal{L}_{|\tilde{D}|}(\theta^{T_{k,j}}) -\mathcal{L}_{|\tilde{D}|}(\theta_{|\tilde{D}|}^{*})   \le \epsilon \label{lem_4_17_1}\\
\end{equation}

Dividing Equation \eqref{lem_4_17_1} by $C_{k}^{'}$ we get 

\begin{equation}
      |\theta^{T_{k,j}}-\theta_{|\tilde{D}|}^{*}|_{1}   \le  \frac{\mathcal{L}_{|\tilde{D}|}(\theta^{T_{k,j}}) -\mathcal{L}_{|\tilde{D}|}(\theta_{|\tilde{D}|}^{*})}{C_{k}^{'}}   \le \frac{\epsilon}{C_{k}^{'}} \label{lem_4_17_2}\\
\end{equation}

Which implies 

\begin{equation}
      |\theta^{T_{k,j}}-\theta_{|\tilde{D}|}^{*}|_{1}  \le \frac{\epsilon}{C_{k}^{'}} \label{lem_4_17_3}\\
\end{equation}

Assuming $\epsilon$ is small enough such that $\frac{\epsilon}{C_{k}^{'}} < 1$ from lemma \ref{sup_lem_5}, this implies that there exists an $L_{k,j}>0$ such that    

\begin{eqnarray}
     |Q_{\theta^{T_{k,j}}}(s,a)-Q_{\theta_{|\tilde{D}|}^{*}}(s,a)| &\le&  L_{k,j}|\theta^{T_{k,j}}-\theta_{|\tilde{D}|}^{*}|_{1}   \label{lem_4_19}\\
                                         &\le&  \frac{L_{k,j}\epsilon}{C_{k}^{'}} \label{lem_4_20}
\end{eqnarray}
for all $(s,a) \in \mathcal{S}\times\mathcal{A}$. Equation \eqref{lem_4_20} implies that if

\begin{eqnarray}
T_{k,j}  &\ge& L_{k,j}^{2}||u_{k,j}^{*}||^{2}_{2}\epsilon^{-2} - 1,\\
\alpha_{k,j} &=& \frac{||u^{*}_{k}||_{2}}{L_{k,j}\sqrt{T_{k,j}+1}},
\end{eqnarray}

then we have

\begin{eqnarray}
     \mathbb{E}(|Q_{\theta^{T_{k,j}}}(s,a)-Q_{\theta_{|\tilde{D}|}^{*}}(s,a)|) &\le& \frac{L_{k,j}\epsilon}{C_{k}^{'}} \label{lem_4_21}
\end{eqnarray}

By definition in section \ref{thm proof} $Q_{k,j}$ is our estimate of the $Q$ function at the $k^{th}$ iteration of Algorithm $1$ and thus we have  $Q_{\theta^{T_{k,j}}}=Q_{k,j}$ which implies that

\begin{eqnarray}
     \mathbb{E}(|Q_{k,j}(s,a)-Q_{\theta_{\tilde{D}}^{*}}(s,a)|) &\le& \frac{L_{k,j}\epsilon}{C_{k}^{'}} \label{lem_4_22}
\end{eqnarray}

 If we replace $\epsilon$ by $\frac{C^{'}_{k,j}\epsilon}{L_{k,j}}$ in Equation \eqref{lem_4_21}, we get that if 

\begin{eqnarray}
T_{k,j} &\ge& \left({\frac{C^{'}_{k,j}\epsilon}{L_{k,j}}}\right)^{-2}L_{k,j}^{2}||u_{k,j}^{*}||^{2}_{2} - 1, \label{lem_4_23}\\
\alpha_{k,j} &=& \frac{||u^{*}_{k}||_{2}}{L_{k,j}\sqrt{T_{k,j}+1}},
\end{eqnarray}

we have

\begin{eqnarray}
     \mathbb{E}(|Q_{k,j}(s,a)-Q_{\theta_{\tilde{D}}^{*}}(s,a)|) &\le& \epsilon \label{lem_4_24}
\end{eqnarray}

From Assumption \ref{assump_2}, we have that 

\begin{eqnarray}
     \mathbb{E}(|Q_{\theta^{*}}(s,a)-Q_{\theta_{\tilde{D}}^{*}}(s,a)|)&\le& \epsilon_{|\tilde{D}|} \label{lem_4_25}
\end{eqnarray}

where $\theta^{*} = \argminA_{\theta \in \Theta}\mathcal{L}(\theta)$ and by definition of $Q^{3}_{k,j}$ in Definition \ref{def_6}, we have that  $Q^{3}_{k,j}=Q_{\theta^{*}}$. Therefore if we have 

\begin{eqnarray}
T_{k,j} &\ge& \left({\frac{C^{'}_{k,j}\epsilon}{L_{k,j}}}\right)^{-2}L_{k,j}^{2}||u_{k,j}^{*}||^{2}_{2} - 1, \label{lem_4_26}\\
\alpha_{k,j} &=& \frac{||u^{*}_{k}||_{2}}{L_{k,j}\sqrt{T_{k,j}+1}},
\end{eqnarray}

we have 

\begin{eqnarray}
     \mathbb{E}(|Q_{k,j}(s,a)-Q^{3}_{k,j}(s,a)|)_{\nu} &\le&  \mathbb{E}(|Q_{k,j}(s,a)-Q_{\theta_{\tilde{D}}^{*}}(s,a)|)   + \mathbb{E}(|Q^{3}_{k,j}(s,a)-Q_{\theta_{\tilde{D}}^{*}}(s,a)|)   \nonumber\\
                                                        \label{lem_4_27}\\
                                                 &\le&  \epsilon + \epsilon_{|\tilde{D}|} \label{lem_4_28}
\end{eqnarray}

This implies
\begin{eqnarray}
     \mathbb{E}(|Q_{k,j}(s,a)-Q^{3}_{k,j}(s,a)|) &\le&  \tilde{\mathcal{O}}\left(\frac{1}{\sqrt{T_{k,j}}}\right) + \epsilon_{|\tilde{D}|} \label{lem_4_29}
\end{eqnarray}

\end{proof}

\end{document}